\PassOptionsToPackage{dvipsnames}{xcolor}
\documentclass{article}
\usepackage{arxiv,times}
\renewcommand{\headrulewidth}{0pt}

\usepackage{hyperref}
\usepackage{url}

\usepackage{amsfonts}
\usepackage{booktabs}
\usepackage{subcaption}
\usepackage{makecell}
\usepackage{siunitx}
\usepackage{braket}
\usepackage{placeins}
\usepackage{fontawesome5}
\usepackage{amsthm}
\usepackage{mathtools}
\usepackage[dvipsnames]{xcolor}
\usepackage{wrapfig}

\usepackage{enumitem}
\usepackage[nameinlink, capitalize]{cleveref}
\usepackage{appendix}
\usepackage[bottom]{footmisc}

\usepackage{tikz}
\usetikzlibrary{arrows.meta}
\usetikzlibrary{decorations.pathreplacing}
\usetikzlibrary{decorations.pathmorphing}
\usetikzlibrary{positioning,fit,backgrounds,calc}

\crefformat{equation}{(#2#1#3)}
\crefformat{figure}{Figure~#2#1#3}
\crefname{example}{Example}{Examples}
\crefname{lemma}{Lemma}{Lemmas}
\crefname{cor}{Corollary}{Corollaries}
\crefname{theorem}{Theorem}{Theorems}
\crefname{assumption}{Assumption}{Assumptions}

\crefname{appsec}{appendix}{appendices}
\Crefname{appsec}{Appendix}{Appendices}

\theoremstyle{plain}
\newtheorem{theorem}{Theorem}[section]

\newtheorem{corollary}[theorem]{Corollary}
\theoremstyle{definition}
\newtheorem{definition}[theorem]{Definition}

\newtheorem{example}[theorem]{Example}
\theoremstyle{remark}

\newcommand{\qualimg}[1]{\includegraphics[width=\linewidth,height=2.6cm,keepaspectratio]{figs/llava_qual_failures/#1}}
\newcommand{\cvec}[1]{\bigl[\begin{smallmatrix}#1\end{smallmatrix}\bigr]}
\newcommand\boldred[1]{\textcolor{red}{\textbf{#1}}}
\newcommand\boldgreen[1]{\textcolor{ForestGreen}{\textbf{#1}}}
\DeclareMathOperator{\Tr}{Tr}

\title{Persistent Representational Misalignment\\Induces Vision-Language Model Failures}

\author{Yonatan Gideoni\NoHyper\thanks{Correspondence to \texttt{yg@robots.ox.ac.uk}. Code can be found at \href{https://github.com/YonatanGideoni/persistent_repr_misalignment}{\url{https://github.com/YonatanGideoni/persistent_repr_misalignment}}.} \\
University of Oxford\\
\And
Yoav Gelberg \\
University of Oxford\\
\And
Tim G. J. Rudner \\
University of Toronto \\
\And
Yarin Gal \\
University of Oxford \\
}

\iclrfinalcopy
\begin{document}

\maketitle

\begin{abstract}
Vision-language models (VLMs) fail at some tasks that are simple for humans---but why? We hypothesize that many such failures stem from VLMs’ difficulties translating between vision and language representations—a misalignment the Platonic Representation Hypothesis predicts should disappear with scale. This paper is composed of two parts: in the first, we show that existing vision and language representations are linearly unalignable. This misalignment is an inherent property of the data and hence expected to persist regardless of either scale or the alignment's complexity. In the second, we demonstrate how this misalignment both leads to a predictably induced generalization failure and correlates with existing failure modes, where more representationally aligned VLM backbones accordingly exhibit fewer failures. Overall, these results imply representations will not become aligned simply given sufficient scale, providing evidence against a strong version of the Platonic Representation Hypothesis.
\end{abstract}

\vspace{-10pt}
\section{Introduction}
\vspace{-5pt}
Although vision and language models are considered superhuman on many tasks \citep{he2015delving, achiam2023gpt}, their multimodal counterparts exhibit some very simple failures. Vision-language models (VLMs) struggle identifying visual analogies \citep{yiu2024kiva}, matching semantically different but syntactically similar captions to images \citep{thrush2022winoground}, identifying simple geometric properties like whether two circles overlap \citep{rahmanzadehgervi2024vision}, and counting atypical objects, e.g. the number of legs on a three-legged chicken \citep{vo2025vision}. Humans outperform VLMs on all of these tasks.

What causes these failures and why they persist is generally unclear, with several linked observations. \citet{vo2025vision} observe that VLMs can rely on their language component's world knowledge (``chickens typically have two legs'') more than on what their vision backbones see, but do not discuss why this occurs. Testing why VLMs struggle with fine visuolingual retrieval, \citet{diwan2022winoground} speculate that it is due to misalignment between the visual and text representations.

A lack of representational alignment is surprising given some common beliefs in the field. The Platonic Representation Hypothesis states that sufficiently well trained models should converge to equivalent representations \citep{huh2024platonic}, with many works presuming or relying on such convergence \citep{tjandrasuwita2025understanding, zhu2025dynamic, jha2025harnessing, gupta2025better, wang2025words, bahng2025cycle, he2025seeing}.

In this paper, we call the Platonic Representation Hypothesis into question and show implications thereof on existing models. To test whether vision and language representations converge, we develop a simple theoretical framework for when sets of representations are alignable (\Cref{sec:intermod_hard}). We find that existing representations are unalignable due to underlying differences in the data, meaning alignment will likely not naturally emerge with scale. This misalignment predictably leads to an induced generalization failure in vision-language models, where a VLM fails to recognize a concept its vision and language backbones know if it does not appear in its paired finetuning data (\Cref{sec:vlms_failure}). Moreover, representational alignment correlates with performance on the failure modes, where more aligned VLM backbones lead to VLMs which exhibit fewer failures (\Cref{sec:repr_align_failures}).

Altogether, these results support the conjecture that {\bf a diverse set of VLM failures likely stem from representational misalignment and will not be resolved solely by increasing model scale}. This conjecture provides a joint cause for seemingly disjoint failure modes while explaining why the failures persist even in very capable models.

\section{Relations Between Representations}
\vspace{-10pt}
\label{sec:align_equiv}
In this section we discuss different ways in which sets of representations can be similar. While there are many existing notions \citep{sucholutsky2023getting}, many are statements on whether the representations are \textit{equivalent} or \textit{alignable}. We now discuss what each of these mean and the relationships between different kinds of similarity, with additional related work in Appendix \ref{app:rel_work}.

There are two main ways in which different sets of representations can be equivalent. The first is \textit{functional equivalence}, where two representations or models producing some representations are used for some task. For example, take two sets of word embeddings used for retrieval. In this case, the representations are equivalent if they perform comparably well and make the same mistakes, thereby being indistinguishable from a black-box perspective. A weaker but related form of equivalence is \textit{representational equivalence} \citep{klabunde2025similarity}, where sets of representations are equivalent up to some transformation. Invariances to different transformations define different similarity kernels between representations, with most similarity kernels defining two representations as equivalent if they are the same up to some (constrained) linear transformation, perhaps with some additional normalization \citep[section~4]{kornblith2019similarity}. Works like the Platonic Representation Hypothesis primarily discuss representational equivalence \citep{huh2024platonic}. Representational equivalence is likely, but not guaranteed, to result in functional equivalence -- for the word embeddings, an orthogonal transformation over an entire space would not affect which words are retrieved as distances would be preserved, but more general transforms could lead to functional differences.

\looseness=-1 Representational equivalence is closely related to \textit{representational alignment}. For typical generative VLMs, pretrained vision and language models are combined using an adapter that projects vision representations into the language model, thereby aligning the vision representations with the language token embeddings \citep{alayrac2022flamingo, merullo2022linearly, liu2023visual}. For retrieval, language and vision representations are aligned by being projected into a shared representation space \citep{radford2021learning}. Representations that are representationally equivalent are easy to align, as a linear map is often sufficient to map between them. More broadly, representations can be \textit{functionally alignable} if given some training they can be used together to generalize over some task, where generalization is important as an adapter may just memorize training data. If the adapter is a simple linear map, as it often is for generative VLMs \citep{liu2023visual, zhu2023minigpt}, functional alignment implies a degree of representational alignment, as the learned map aligns the representations. This is often used in model-stitching \citep{lenc2015understanding,bansal2021revisiting,csiszarik2021similarity}, where parts of one model are stitched to another using a trained linear layer, as a way to probe whether two models learn the same representations.\footnote{This reasoning has some known caveats; see \citet{smith2025functional}.}

\begin{wrapfigure}{r}{0.36\textwidth}
    \vspace{-15pt}
    \scalebox{0.82}{
    \begin{tikzpicture}[
    >=Stealth,
    every node/.style={text=black, align=center},
    line/.style={draw=black, thick}
    ]
        
        \node (FE) at (0,2.5) {Functional\\Equivalence};
        \node (FA) at (0,0)   {Functional\\Alignment};
        
        \node (RE) at (3.75,2.5) {Representational\\Equivalence};
        \node (RA) at (3.75,0)   {Representational\\Alignment};
        
        \draw[line,->] (FA) -- (FE);
        
        \draw[line,->] (FA) -- (RA);
        
        \draw[line,->] (RE) -- node[above, midway] {\small (typically)} (FE);
        
        \draw[line,->] ([xshift=-6pt]RE.south) -- ([xshift=-6pt]RA.north);
        \draw[line,->] ([xshift=6pt]RA.north) -- node[above,midway, rotate=270, yshift=0pt] {\small (typically)} ([xshift=6pt]RE.south);
    \end{tikzpicture}
    }
    \vspace{-20pt}
\end{wrapfigure}

In practice, both equivalence and alignment are not binary properties but fall on various spectrums and in some cases one can imply another. Typical relationships between equivalence and alignment are summarized in this diagram, where arrows represent ``X implies Y'' and ``typically'' means that counterexamples exist but are somewhat adversarial. Thus, two representations being representationally equivalent typically implies they are functionally equivalent as well.

This framework helps understand how different statements on representations relate to one another. For example, the Platonic Representation Hypothesis \citep{huh2024platonic} asserts that representations in different modalities should be representationally equivalent and tests this by showing trends in representational alignment. Representational alignment typically implies representational equivalence, thereby supporting their claim. Throughout this paper we will show different forms of alignment or equivalence fail to hold, noting how one not holding rules out another, e.g. representations cannot be functionally aligned if they are functionally inequivalent.

\vspace{-9pt}
\section{Broken Conditions for Representational Alignment}
\label{sec:intermod_hard}
In this section we derive the conditions for representations to be linearly alignable. Operationally, many alignment metrics like CKA quantify whether sets of representations are linearly alignable up to normalization, restricted to some subset of linear transformations \citep{kornblith2019similarity}. In Appendix \ref{app:repr_sim_lin_align} we formalize some of this intuition, showing how certain bounds on representational alignability accordingly bound some common representational similarity metrics.

A useful way to test for representational alignability is by considering how concepts are related in different representations \citep{dhimoila2026unifying}. For example, over a decade ago \citet{mikolov2013exploiting} showed that word embeddings trained on different languages have similar relationships between the same sets of words, so a linear map is sufficient to translate between representations in different languages. In this case each word embedding defines how a concept, namely its word, is represented.

One very simple kind of relationship between concepts is analogies. More formally, let there be some analogy where $A$ is to $B$ as $C$ is to $D$, e.g. king is to queen as man is to woman. The vision representation of a concept like ``king'' can be found by averaging the representations of many images of kings. All presented theorems are based on simple known lemmas in linear algebra, with proofs delegated to \Cref{app:thm_proof}. If $v_X$ denotes some model's representation of concept $X$ then we can formally define an analogy as follows.

\begin{definition}[Concepts and analogies in representations]
The vector representing some concept $X$ is denoted as $v_X$. The analogy $A$ is to $B$ as $C$ is to $D$ (denoted $A:B::C:D$) is said to exist for a model's representations of these concepts if $v_A-v_B=v_C-v_D$. The tuple representing an analogy is denoted as $v^{A:B::C:D}\coloneqq(v_A,v_B,v_C,v_D)$. We assume no concept is represented by the zero vector.
\end{definition}

In practice, the analogy would only be approximate, so $v_A-v_B\approx v_C-v_D$. Throughout this section, take $V,W$ to be two finite-dimensional inner product spaces, e.g. $\mathbb{R}^n$ and $\mathbb{R}^m$. One can intuitively think of them as the representations of two different models, potentially from different modalities. For simplicity, all inner products are viewed as dot products, where $\braket{v,w}=v^Tw$, but the theorems equally hold regardless of the kind of inner product.

If there are enough corresponding analogies between two models then one might expect a linear map between their representations to exist, as the analogies enforce a linear structure. In practice, this is not the case.

\begin{theorem}[Analogies are insufficient for a linear map to exist]
    Let there be two sets of the same analogies in $V$ and $W$, $\{v^{A_i:B_i::C_i:D_i}\}$ and $\{w^{A_i:B_i::C_i:D_i}\}$. There exist sets of analogies such that there is no corresponding linear map $T:V\to W$ such that $\forall i,X:T(v_{X_i})=w_{X_i}$.
\end{theorem}

\begin{wrapfigure}{r}{0.35\textwidth}
\centering
\vspace{-22pt}
\begin{tikzpicture}[>=Stealth, scale=1.]

\definecolor{colA}{RGB}{228,26,28}
\definecolor{colB}{RGB}{55,126,184}
\definecolor{colC}{RGB}{77,175,74}
\definecolor{colD}{RGB}{152,78,163}

\begin{scope}
  \draw[->] (0,0) -- (4.5,0) node[below left] {$\mathbb{R}^1$};

  \draw[colD, thick, ->] (0,0) -- (4,0) node[above] {$v_D$};
  \draw[colC, thick, ->] (0,0) -- (3,0) node[above] {$v_C$};
  \draw[colB, thick, ->] (0,0) -- (2,0) node[above] {$v_B$};
  \draw[colA, thick, ->] (0,0) -- (1,0) node[above] {$v_A$};

  \draw[decorate, decoration={brace, amplitude=4pt}]
  (1,0.4) -- (2,0.4)
  node[midway, above=2pt] {$v_A - v_B$};

  \draw[decorate, decoration={brace, amplitude=4pt}]
  (3,0.4) -- (4,0.4)
  node[midway, above=2pt] {$v_C - v_D$};
\end{scope}

\begin{scope}[yshift=-3.1cm]
  \draw[->] (0,0) -- (3.9,0) node[right] {$x$};
  \draw[->] (0,0) -- (0,2.6) node[above] {$y$};
  \node at (3.5,0.35) {$\mathbb{R}^2$};

  \draw[colD, thick, ->] (0,0) -- (4,2) node[above] {$w_D$};
  \draw[colC, thick, ->] (0,0) -- (3,2) node[above] {$w_C$};
  \draw[colB, thick, ->] (0,0) -- (2,1) node[above] {$w_B$};
  \draw[colA, thick, ->] (0,0) -- (1,1) node[above] {$w_A$};

  \draw[black, dashed, ->] (1,1) -- (2,1);
  \draw[black, dashed, ->] (3,2) -- (4,2);
\end{scope}

\end{tikzpicture}
\vspace{-30pt}
\end{wrapfigure}

The following example is effectively a proof of this theorem.
\begin{example}[Different information can lead to inconsistent mappings]
    \label{exmple:diff_info}
    Observe the following analogies in $\mathbb{R}^1$ and $\mathbb{R}^2$, $v_A=[1],v_B=[2],v_C=[3],v_D=[4]$ and $w_A=\cvec{1\\1},w_B=\cvec{2\\1},w_C=\cvec{3\\2},w_D=\cvec{4\\2}$. Clearly, there is no linear transformation $T(v_X)=w_X$ as $4v_A=v_D$ but $4w_A\neq w_D$. The diagram to the right illustrates this visually.
\end{example}

\Cref{exmple:diff_info} can be equivalently interpreted as one representation having information that is missing in the other. Intuitively, the second index in the $w$ vectors could be some information that is irrelevant for the analogy but consistently exists, like a typical syntax in a sentence or different lighting in an image.

A stronger requirement, which does result in linear mappings, is that the different representations have the same similarity structure, akin to the case of \citet{mikolov2013exploiting} in machine translation \citep{smith2017offline}. For representations with unit norm this means that the cosine similarity between concepts $A$ and $B$ is the same in both modalities, so $v_A^Tv_B=w_A^Tw_B$. For example, this condition implies that the cosine similarity between an image of a cat and an image of a dog is close to the cosine similarity between the words ``cat'' and ``dog''.
\begin{definition}[Similarity structures]
    \label{def:simstruct}
    Given a set of concept representations $\{v_{X_i}\}_{i=1}^N$, the Gram matrix $G_{ij}\coloneqq v_{X_i}^Tv_{X_j}$ is said to define the representations' similarity structure. Two representations are said to have the same similarity structure if their Gram matrices are the same, so $\forall i,j:v_{X_i}^Tv_{X_j}=w_{X_i}^Tw_{X_j}$.
\end{definition}

Interestingly, having the same similarity structure implies that under some similarity metrics, like CKA, these representations are equivalent.

\begin{theorem}[Same similarities enable linear mappings]
    \label{thm:similarities_enable_lin}
    Let $V$ and $W$ have corresponding sets of concept representations, $\{v_{X_i}\}_{i=1}^N, \{w_{X_i}\}_{i=1}^N$. If they have the same similarity structure then there exists a linear map $T:V\to W$ such that $\forall i:T(v_{X_i})=w_{X_i}$.
\end{theorem}
This theorem is stronger than required for a linear map to exist, as the resulting $T$ is an isometry, a map that preserves distances. Moreover, assuming the spaces have the same dimension, such a map would have an inverse, so it loses no information as it allows mapping back and forth between them. In the single modality setting \citet{smith2017offline} construct such a map for translating between sets of word embeddings, showing that the different embeddings indeed have a similar similarity structure. Representations having different similarity structures, however, restricts which transformations can map between them.
\begin{corollary}[Different similarity structures rule out isometric mappings]
\label{thm:diff_simstruct_no_iso}
    Assume $V,W$ have corresponding sets of concept representations but different similarity structures. Thus, there cannot exist an isometric linear map $T:V\to W$ such that $\forall i:T(v_{X_i})=w_{X_i}$.
\end{corollary}

A weaker condition that results in generic linear maps is requiring a form of generalized analogies to exist in both modalities. For example, if $v_\text{zebra}=v_\text{horse}+v_\text{stripes}$ then $w_\text{zebra}=w_\text{horse}+w_\text{stripes}$ as well.
\begin{theorem}[Generalized analogies enable linear mappings]
    \label{thm:gen_analogs}
    Let $V$ and $W$ have corresponding sets of concept representations, $\{v_{X_i}\}_{i=1}^N$ and $\{w_{X_i}\}_{i=1}^N$. A linear map $T$ such that $\forall i:T(v_{X_i})=w_{X_i}$ exists if and only if for all sets of scalars $a_i\in\mathbb{R}$ if $\sum_i a_i v_{X_i}=0$ then $\sum_i a_i w_{X_i}=0$.
\end{theorem}
Note that any generalized analogy can be written as $\sum_i a_i v_{X_i}=0$, e.g. in the case of $v_\text{zebra}=v_\text{horse}+v_\text{stripes}$ by moving $v_\text{zebra}$ to the right hand side. A natural but important corollary of this theorem is that regular analogies not holding results in generalized analogies not holding as well.
\begin{corollary}[Generalized analogies result in and require regular analogies]
    If the conditions of \Cref{thm:gen_analogs} hold then any regular analogy must exist in both $V$ and $W$ as well, so if $v^{A:B::C:D}$ is an analogy then so is $w^{A:B::C:D}$. Moreover, any analogy existing only in $V$ but not in $W$, so $v_A-v_B=v_C-v_D$ but $w_A-w_B\neq w_C-w_D$, violates the conditions of \Cref{thm:gen_analogs}.
\end{corollary}

In practice, text and vision have neither the same similarity structure nor the same analogies, respectively violating the conditions of both \Cref{thm:similarities_enable_lin} and \Cref{thm:gen_analogs}. Note that, as \Cref{thm:gen_analogs} is both a necessary and sufficient condition, this implies there is no concept-preserving linear map between them. We first show that text and vision representations have a different similarity structure and then that they have different analogies.

\vspace{-3pt}
\subsection{Language and Vision Have Different Similarity Structures}
Similarity between datapoints is typically measured using representations from semantic similarity models. When these representations are normalized the dot product between them corresponds to a cosine similarity. For text and vision to have the same similarity structure (as per \Cref{def:simstruct}), thus fulfilling the conditions of \Cref{thm:similarities_enable_lin}, two images with a given cosine similarity should have captions with approximately the same cosine similarity as well.

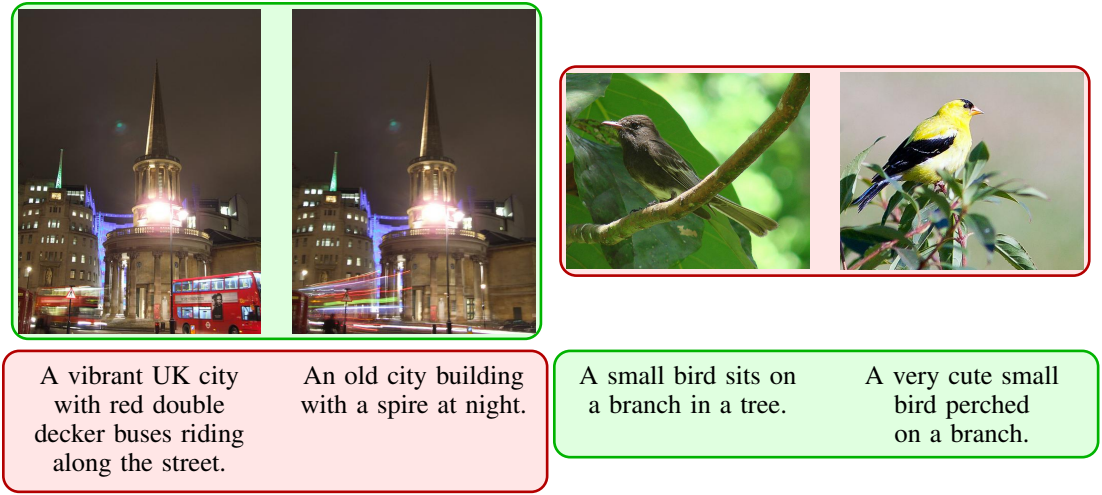
\begin{figure}[!h]
\vspace{-5pt}
\centering

\begin{tikzpicture}
  \node[inner sep=0pt] (aimg) at (0,0)
    {\includegraphics[width=0.23\textwidth]{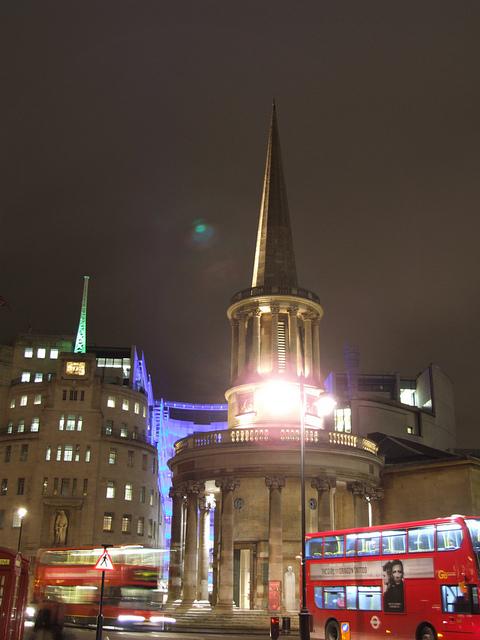}};
  \node[inner sep=0pt, right=4mm of aimg] (bimg)
    {\includegraphics[width=0.23\textwidth]{figs/london_bus2.jpg}};
  \node[inner sep=0pt, right=4mm of bimg] (cimg)
    {\includegraphics[width=0.23\textwidth]{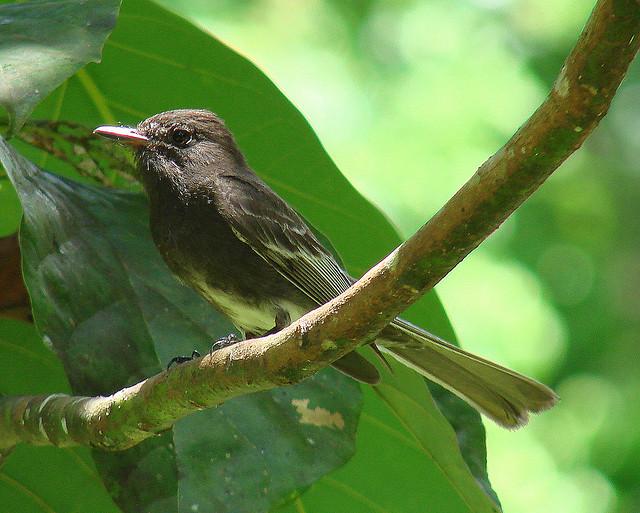}};
  \node[inner sep=0pt, right=4mm of cimg] (dimg)
    {\includegraphics[width=0.23\textwidth]{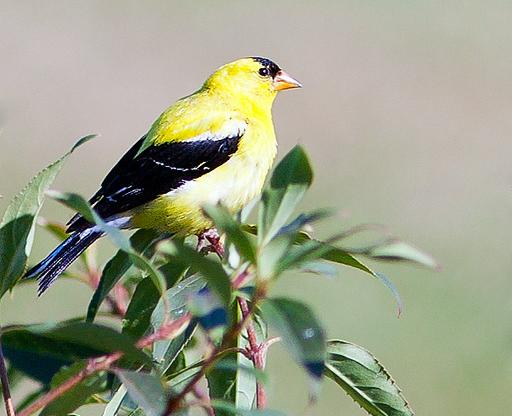}};

  \coordinate (capline) at ($(aimg.south) + (0,-3mm)$);

  \node[align=center, text width=0.23\textwidth, anchor=north] (acap)
    at ($(aimg.south |- capline)$)
    {A vibrant UK city with red double decker buses riding along the street.};

  \node[align=center, text width=0.225\textwidth, anchor=north] (bcap)
    at ($(bimg.south |- capline)$)
    {An old city building with a spire at night.};

  \node[align=center, text width=0.225\textwidth, anchor=north] (ccap)
    at ($(cimg.south |- capline)$)
    {A small bird sits on a branch in a tree.};

  \node[align=center, text width=0.23\textwidth, anchor=north] (dcap)
    at ($(dimg.south |- capline)$)
    {A very cute small bird perched on a branch.};

  \begin{scope}[on background layer]
    \node[draw=green!70!black, fill=green!12, rounded corners=6pt,
          line width=1pt, inner sep=2pt, fit=(aimg) (bimg)] {};
    \node[draw=red!70!black, fill=red!10, rounded corners=6pt,
          line width=1pt, inner sep=2pt, fit=(acap) (bcap)] {};

    \node[draw=red!70!black, fill=red!10, rounded corners=6pt,
          line width=1pt, inner sep=2pt, fit=(cimg) (dimg)] {};
    \node[draw=green!70!black, fill=green!12, rounded corners=6pt,
          line width=1pt, inner sep=2pt, fit=(ccap) (dcap)] {};
  \end{scope}
\end{tikzpicture}
\caption{\textbf{Two scenes can be similar in one modality but different in another.} \textbf{(left)} shows two images which are visually similar, where a vision embedding model gives them a cosine similarity of $0.96$. Their corresponding captions are more distinct, having a similarity of $0.44$. \textbf{(right)} shows the opposite case, where the captions are similar but the images are distinct. The captions have a cosine similarity of $0.97$ while the images' cosine similarity is $-0.06$. The text similarities are for the average embeddings over the five different captions given in COCO for each image. See \Cref{app:more_diff_semsim_examples} for more examples.}
\label{fig:modality_semdiff}
\vspace{-15pt}
\end{figure}

In practice, visual and textual descriptions of the same scene can be similar in one modality but very different in another due to the different modalities having different information contents. \Cref{fig:modality_semdiff} gives some examples of image-caption pairs where, due to each modality containing different information, this occurs. All image-caption pairs are from the COCO training set. Cosine similarities are computed using DINOv3 base for the images and Sentence Transformers' all-MiniLM-L12-v2 for the captions \citep{simeoni2025dinov3, reimers-2019-sentence-bert}. Caption representations are averages over all of an image's caption variants in the dataset.

Pairs like those in \Cref{fig:modality_semdiff} occur with a non-negligible frequency. Their prevalence can be quantified by finding pairs for which the cosine similarity between the images differs substantially from that between their captions, with the absolute difference in their cosine similarities being above some large threshold (e.g. 0.4). For about 2\% of COCO image/caption pairs, their cosine similarity difference exceeds this threshold.

To quantitatively measure whether the two representations have the same similarity structure, we empirically compare their cosine similarities. Formally, for models $f_1,f_2$ outputting normalized vectors $f_1(x),f_2(x)$, we define $\operatorname{CosSimRMSE}(f_1,f_2)\coloneqq \sqrt{\mathbb{E}[(f_1(x)^Tf_1(y)-f_2(x)^Tf_2(y))^2]}$, where the expectation is over the dataset. CosSimRMSE tests the extent to which two representations have the same similarity structure as per \Cref{def:simstruct}. Although here both models get as input some $x$, 
\begin{wraptable}{r}{0.5\textwidth}
    \centering
    \caption{Relative cosine similarity RMSEs for representations from different vision and language models. RMSEs across different modalities are typically at least $3\times$ higher than those within modalities. All values are relative to Text,Text', where $\operatorname{CosSimRMSE}(\text{Text},\text{Text'})=0.05$. RMSEs are calculated using image/caption pairs from the COCO dataset. Models are DINOv3-B and -S for Vis and Vis', all-MiniLM-L12-v2 and -L6-v2 for Text and Text', and the CLIP-B/16 vision/text encoders for CLIP Vis/Text.
    }
    \vspace{-10pt}
    \label{tab:rel_cosine_rmse}
    \begin{tabular}{c c c}
        \multicolumn{2}{c}{Representations} & \shortstack{Cosine sim.\\RMSE ratio} \\
        \midrule
        Text      & Text'      & 1.0  \\
        Vis       & Vis'       & 1.3  \\
        Vis       & Text       & 3.3  \\
        Vis       & CLIP Vis   & 10.9 \\
        Text      & CLIP Text  & 11.2 \\
        CLIP Vis  & CLIP Text  & 2.9 \\
    \end{tabular}
    \vspace{-20pt}
\end{wraptable}
in general one model could get an image while the other gets the corresponding set of captions. We measure $\operatorname{CosSimRMSERatio}(f_1,f_2|g_1,g_2)\coloneqq \frac{\operatorname{CosSimRMSE}(f_1,f_2)}{\operatorname{CosSimRMSE}(g_1,g_2)}$, with $g_1,g_2$ representing two models in the same modality that form some lower bound. As \Cref{tab:rel_cosine_rmse} shows, vision and text models have an approximately $3\times$ larger gap between one another than models in the same modality. CLIP models, despite being purposely trained to align their representations, exhibit only a slightly better degree of alignment likely due to this difference being inherent in the data.

Altogether, this section's results show that \Cref{thm:similarities_enable_lin}'s condition does not hold, as vision and language have very different similarity structures. We now show that \Cref{thm:gen_analogs} also cannot hold, showing that typical vision and language representations are linearly representationally unalignable.

\subsection{Language and Vision Have Different Analogies}
To see whether vision and language representations exhibit the same analogies we use the text analogy dataset of \citet{mikolov2013efficient}. These include analogies like king:queen::man:woman and Paris:France::London:UK. To compare lingual to visual analogies, we construct a new dataset called VisLingA (for \textbf{vis}ual \textbf{ling}ual \textbf{a}nalogies), containing images representing various concepts, see Appendix \ref{app:vislinga} for qualitative examples.\footnote{VisLingA can be found at \url{https://huggingface.co/datasets/YGideoni/VisLingA}.} VisLingA is constructed by downloading about 5-10 images for each concept in the text analogies dataset. Analogies containing words with too few images and analogies in the plural, present-participle, and adjective-to-adverb categories are discarded as they tend to contain different words with mostly the same pictures, e.g. ``car'' and ``cars''. Images were downloaded from Wikimedia Commons and filtered to images in the public domain or with CC0/CC-BY copyright licenses, keeping creator metadata. The final dataset has 15,455 analogies containing 746 distinct concepts.

An analogy $A:B::C:D$ can be said to exist in a modality if the representation of $B-A+C$ is close to $D$. Retrieval is done over the embeddings for all the concepts in the dataset, so representations exhibit a set of analogies if they have a low average rank or high accuracy. Vision embeddings are found by averaging all of a concept's images' representations.

\textbf{Results.} \Cref{fig:vislinga_retr} shows that while most analogies exist in text, far fewer are present in vision. All text embedding models have top-10 analogy retrieval accuracies above 75\% while all vision models get below 45\%. This is likely due to the inherent difficulty of capturing abstract concepts or certain nouns in images, like countries and cities. Thus, as vision and language models exhibit different analogies, \Cref{thm:gen_analogs} cannot hold.

\begin{figure*}[h]
    \vspace{-5pt}
    \centering
    \includegraphics[width=\textwidth]{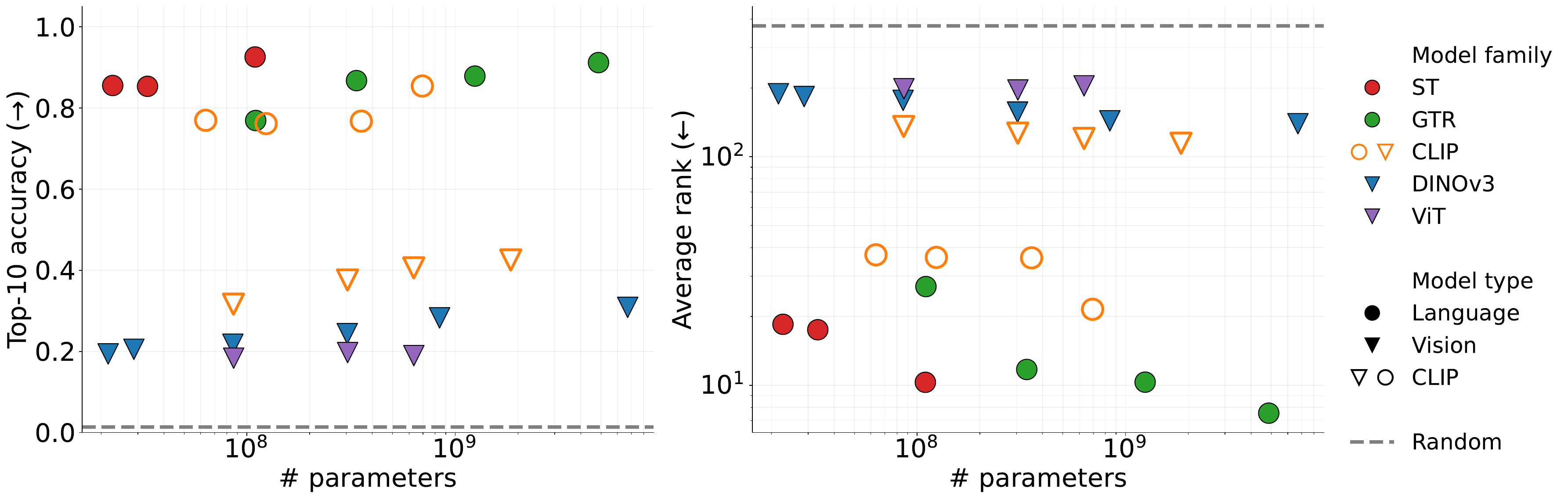}
    \vspace{-8pt}
    \caption{\textbf{Vision models do not exhibit the same analogies that language models do}. (left) top-10 analogy retrieval accuracies over VisLingA for several model families. The vision models (downwards triangles) do significantly worse than the text models (circles). (right) average retrieval rank. All vision models have a retrieval rank above 100, with a random baseline having a rank of 373.5. In both cases, the CLIP's text encoders are the worst models in the language family and the opposite for the vision CLIP models, likely due to the intermodal supervision pulling each CLIP encoder towards the other modality. ST=Sentence Transformer \citep{reimers-2019-sentence-bert}, GTR is from \citet{ni2022large}, CLIP is the LAION CLIP models \citep{cherti2023reproducible}, ViT are ViT models trained on ImageNet 21k \citep{Dosovitskiy2020AnII}.}
    \vspace{-6pt}
    \label{fig:vislinga_retr}
\end{figure*}

\section{Vision-Language Models Fail to Generalize Across Modalities}
\label{sec:vlms_failure}
A functional test for whether vision and language representations are alignable is to see if they can generalize out-of-distribution. Out-of-distribution generalization is important as an adapter between representations can act as a lookup table without relying on any shared underlying structure. For example, \citet{mikolov2013exploiting} found that, due to similar languages having some shared structure, a linear map between English and Spanish can translate words not in the map's training data. If English and Spanish had no shared structure then the linear map could only map from a word in English to its Spanish analogue, like a dictionary, without generalizing further. Thus, in some cases, representational misalignment directly causes a lack of generalization.

The multimodal analogue of machine translation would be to train a VLM as usual while omitting some concepts from its paired finetuning data, see \Cref{fig:fail_ex} (right). These concepts must be ones that the vision and language backbones recognize independently, as otherwise it would be surprising for the final VLM to suddenly recognize them. We discuss what it means for vision and language models to recognize or ``know'' of a concept in \Cref{app:know_conc}. After training, the VLM's generalization can be tested by querying whether it recognizes the held-out concepts.\footnote{This test is inspired by a thought experiment known as Molyneux's problem -- briefly, could a blind person, upon regaining their sight, be able to visually distinguish what they previously only felt? If tested a few days after gaining their sight, the answer is, surprisingly, yes -- see \citet{held2011newly}.}

\begin{figure}
    \centering
    \resizebox{\textwidth}{!}{
        \begin{tikzpicture}[every node/.style={align=center}]
            \draw[gray!40, thin, dashed] (7.5, -1.75) -- (7.5, 5.5);
            \draw[gray!40, thin, dashed] (-1, 2) -- (7.5, 2);
            \node[font=\large\bfseries] at (3, 2.5) {Vision backbone};
            \node[font=\large\bfseries] at (3, -1.25) {Language backbone};
            \node[anchor=south west, font=\large\bfseries] at (8, -0.25) {VLM};
            \node (imgV) at (0, 4) {\includegraphics[width=1.5cm]{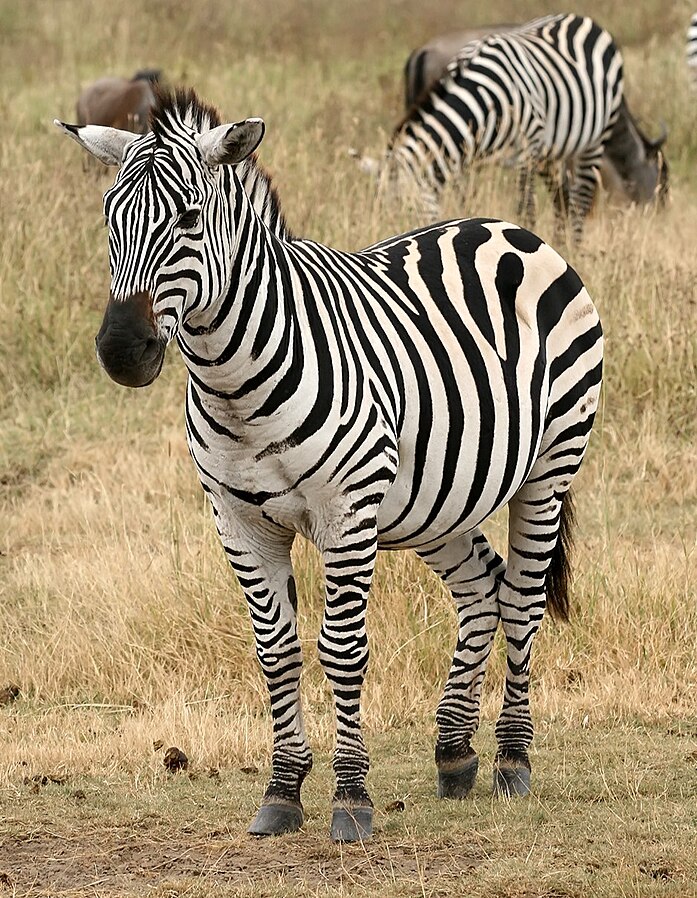}};
            \node[draw, fill=cyan!40, minimum size=1.5cm, rounded corners=3pt] (V1) at (3, 4)
                {\scalebox{1.35}{\Huge \faEye}\\\large Vision\\\large encoder};
            \node (outV) at (6, 4) {99.9\% zebra};
            \draw[->, very thick, line width=2pt] (imgV) -- (V1);
            \draw[->, very thick, line width=2pt] (V1) -- (outV);
            \node (txtL) at (0, 0.25) {A zebra is};
            \node[draw, fill=orange!50, minimum size=1.5cm, rounded corners=3pt] (L1) at (3, 0.25)
                {\scalebox{1.35}{\Huge \faLanguage}\\\large LLM};
            \node (outL) at (6, 0.25)
                {A zebra is a large,\\black and white\\striped animal\\that is native\\to Africa...};
            \draw[->, very thick, line width=2pt] (txtL) -- (L1);
            \draw[->, very thick, line width=2pt] (L1) -- (outL);
            \node (img2) at (8.5, 4) {\includegraphics[width=1.5cm]{figs/zebra.jpg}};
            \node (img3) at (8.5, 2) {\includegraphics[width=1.5cm]{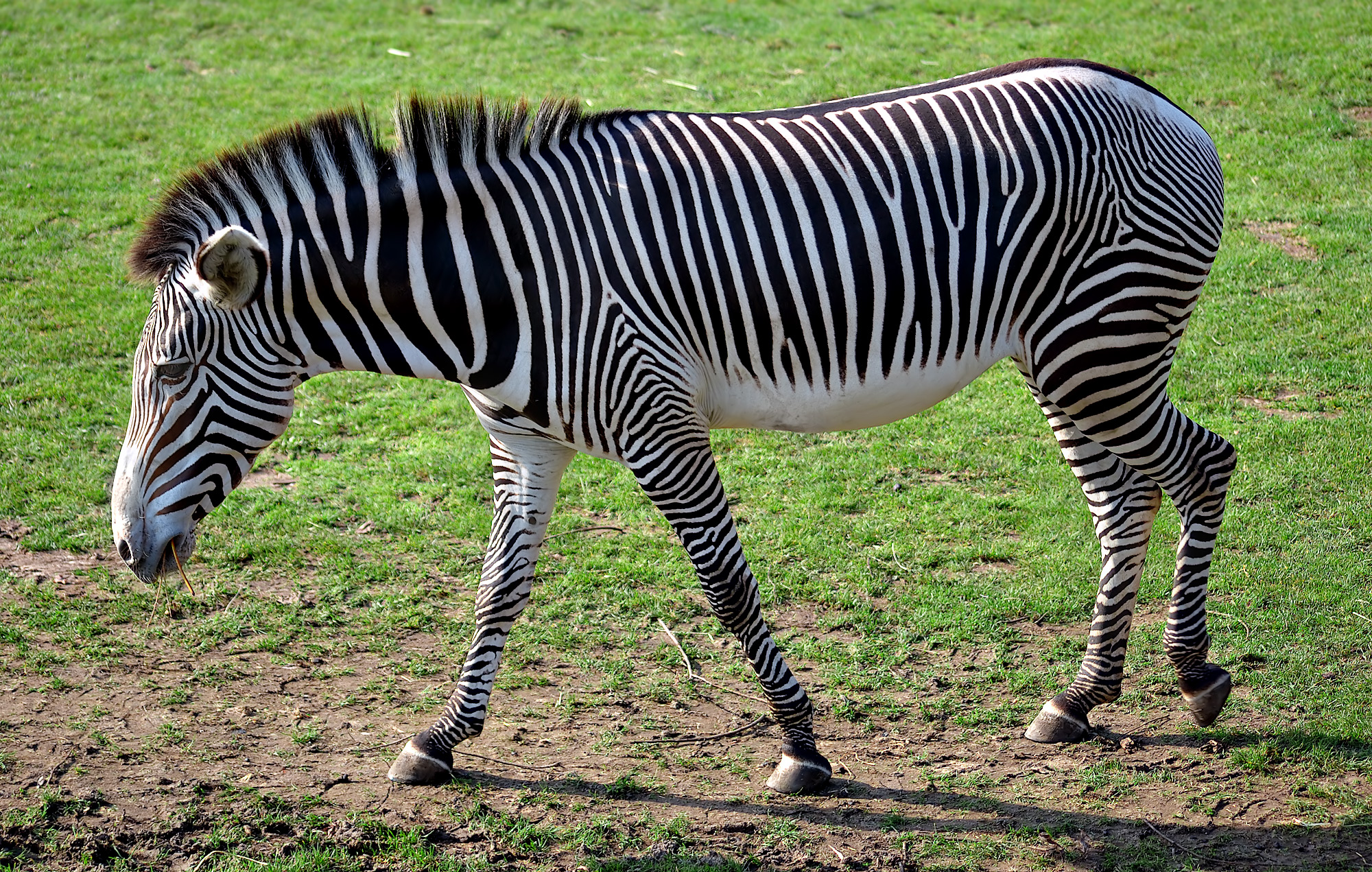}};
            \node[below=-0.15cm of img3, font=\tiny] (training_text) {(omitted during\\training)};
            \node[draw, fill=cyan!40, minimum size=1.5cm, rounded corners=3pt] (V2) at (11.5, 4)
                {\scalebox{1.35}{\Huge \faEye}\\\large Vision\\\large encoder};
            \node[draw, fill=purple!30, minimum width=1.5cm, minimum height=0.8cm, rounded corners=3pt] (A) at (11.5, 2)
                {\large Adapter};
            \node[draw, fill=orange!50, minimum size=1.5cm, rounded corners=3pt] (L2) at (11.5, 0.25)
                {\scalebox{1.35}{\Huge \faLanguage}\\\large LLM};
            \node (text4) at (11.5, -1.5)
                {The image features a group\\of three \textbf{giraffes}...};
            \node at ([xshift=3pt, yshift=-3pt]V2.south east) {\faSnowflake};
            \node at ([xshift=3pt, yshift=-3pt]L2.south east) {\faSnowflake};
            \node at ([xshift=3pt, yshift=-3pt]A.south east) {\faFire};
            \draw[->, very thick, line width=2pt] (img2) -- (V2);
            \draw[->, very thick, line width=2pt, decorate, decoration={snake, amplitude=0.25mm, segment length=2mm}] (V2) -- (A);
            \draw[->, very thick, line width=2pt, decorate, decoration={snake, amplitude=0.25mm, segment length=2mm}] (A) -- (L2);
            \draw[->, very thick, line width=2pt] (L2) -- (text4);
            \draw[->, very thick, line width=2pt] (img3) -- (A);
            \node[red!70] at (10.2, 2) {\scalebox{1.5}{\large \faTimes}};
        \end{tikzpicture}
    }
    \vspace{-3pt}
    \caption{Pretrained vision and language models can each know what a zebra is in isolation, while a unified vision-language model fails if its adapter was not trained on zebra image-text pairs. Wavy lines denote feature vectors. Zebra images from \citet{wiki:zebra}.}
    \label{fig:fail_ex}
    \vspace{-12pt}
\end{figure}

\begin{wraptable}{r}{0.47\linewidth}
  \centering
  \vspace{-15pt}
  \caption{Held-out concept detection accuracy across settings. In all cases, the trained VLMs fail to detect the held-out concepts when they are not included in their paired image-text finetuning data. Large-scale generation refers to LLaVA-1.5-7B. Values in brackets denote bootstrapped 95\% confidence intervals across per-concept accuracies. \textsuperscript{\textdagger}\citep{nlp_connect_2022}}
  \label{tab:concept_detection}
  \vspace{-3pt}
  \begin{tabular}{@{}l c@{}}
    \makecell[l]{\\Setup (\# Parameters)} & \makecell[c]{Held-out Concept\\Detection Accuracy} \\
    \midrule
    \multicolumn{2}{@{}l@{}}{\textbf{Retrieval (210M)}} \\
    \midrule
    Without concepts &
    \makecell[c]{3.8\% {\scriptsize [0.2\%, 9.2\%]}} \\
    
    Full dataset &
    \makecell[c]{81.0\% {\scriptsize [47.3\%, 99.4\%]}} \\
    
    CLIP (ViT-B/16) &
    \makecell[c]{79.2\% {\scriptsize [46.5\%, 97.6\%]}} \\
    
    \midrule
    \multicolumn{2}{@{}l@{}}{\textbf{Generation (210M)}} \\
    \midrule
    Without concepts &
    \makecell[c]{1.2\% {\scriptsize [0.0\%, 3.0\%]}} \\
    
    Full dataset &
    \makecell[c]{84.0\% {\scriptsize [57.9\%, 98.6\%]}} \\
    
    End-to-end finetuning\textsuperscript{\textdagger} &
    \makecell[c]{87.8\% {\scriptsize [66.3\%, 99.8\%]}} \\
    
    \midrule
    \multicolumn{2}{@{}l@{}}{\textbf{Large-scale generation (7B)}} \\
    \midrule
    Without concepts &
    \makecell[c]{21.5\% {\scriptsize [11.9\%, 32.8\%]}} \\
    
    Full dataset &
    \makecell[c]{87.3\% {\scriptsize [80.0\%, 94.1\%]}} \\
    
    CLIP backbone &
    \makecell[c]{95.3\% {\scriptsize [89.8\%, 99.3\%]}} \\
  \end{tabular}
  \vspace{-30pt}
\end{wraptable}
\textbf{Setup.} We train a series of small generative and retrieval VLMs and large generative VLMs while omitting a few concepts from their datasets.  Specifically, for the small-scale settings, we train adapters between frozen ViT-B/16 \citep{Dosovitskiy2020AnII} and GPT2-small \citep{radford2019language} models for image captioning or retrieval, with \Cref{app:smallscale_details} containing ablations and technical details.

For the large-scale setup, we train a LLaVA-1.5-7B model \citep{liu2023visual}, with two important differences relative to how LLaVA is typically trained. First, for its second stage of training, we use low-rank adapters (LoRA) \citep{hu2022lora} instead of tuning the entire LLM to reduce the possibility of catastrophic forgetting. Second, LLaVA typically uses a CLIP vision backbone, which already received language supervision during its pretraining and is already multimodal. Instead, we use a ViT-L/16@384 model that is trained for ImageNet classification. This ViT only saw images during its pretraining and has the same number of parameters and tokens as the default vision backbone.

\textbf{Concept filtering and detection.} In all settings, the held-out concepts are filtered from the finetuning datasets, with details in \Cref{app:concept_filt}. Specific concepts were selected due to being simple, synonymless, and are verifiably known by the models (e.g., all are ImageNet classes, so an image classifier clearly recognizes them). To measure how well a model detects a held-out concept after training, we construct a new test set consisting of concept images. Images containing the concepts are selected by querying ImageNet data using a CLIP model and filtering the resulting examples manually, yielding $\sim$100 images per concept. For generation, a concept is detected if it appears in the generated answer, whereas for retrieval we benchmark the model on image-based text retrieval over 1000 prompts, one from each ImageNet class, including those of the held-out concepts.

\begin{figure}
  \centering

  \begin{subfigure}[t]{0.32\textwidth}
    \centering
    \includegraphics[width=\linewidth]{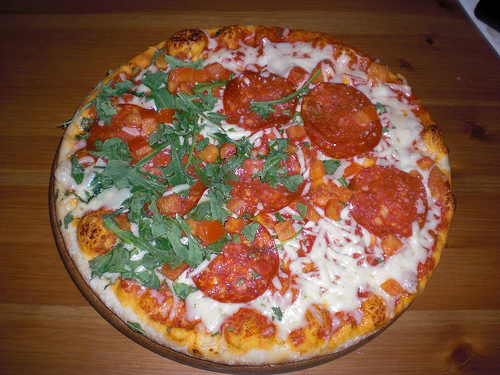}
    \caption{The image features \boldred{a large, round pie} with a crust, filled with a mixture of...}
  \end{subfigure}\hfill
  \begin{subfigure}[t]{0.32\textwidth}
    \centering
    \includegraphics[width=\linewidth]{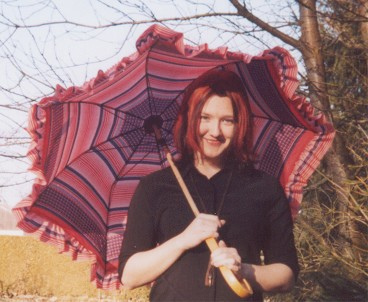}
    \caption{The image features a woman holding \boldred{a colorful kite} in a field.}
  \end{subfigure}\hfill
  \begin{subfigure}[t]{0.32\textwidth}
    \centering
    \includegraphics[width=\linewidth]{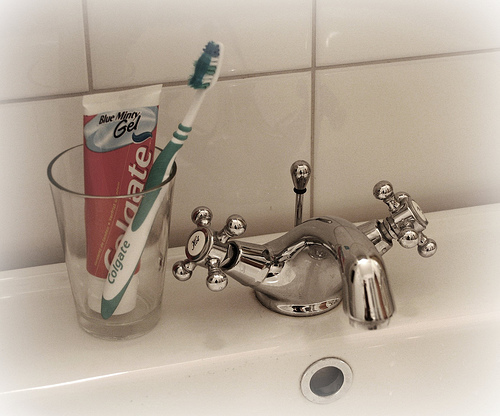}
    \caption{The image features a bathroom sink with \boldgreen{a toothbrush} and a cup of water placed on it.}
  \end{subfigure}

  \caption{Image descriptions given by a LLaVA-1.5-7B model trained without the held-out concepts on images containing them when prompted ``What is in this image?''. The corresponding concepts are (a) pizza, (b) umbrella, and (c) toothbrush. The model fails to detect them for (a) and (b), while it succeeds for (c). See \Cref{app:qual_llava_fail} for more examples.}
  \label{fig:llava_examples}
  \vspace{-17pt}
\end{figure}

\textbf{Results.} \Cref{tab:concept_detection} shows that VLMs trained without the held-out concepts struggle to detect them compared to models trained on the regular full datasets. To ensure the models are properly trained, each section's last row compares to a model trained elsewhere, showing they achieve similar performance. \Cref{fig:llava_examples} contains qualitative examples of descriptions generated by the LLaVA model trained without the held-out concepts.

For large-scale generation---although the LLaVA model trained without the concepts attains much worse performance than the one trained with them---it is still higher than the small-scale experiments' accuracies. We speculate why this may be in \Cref{app:large_scale_disc}.

\vspace{-8pt}
\section{Representationally Aligned VLMs Exhibit Fewer Failures}
\vspace{-4pt}
\label{sec:repr_align_failures}
The generalization failure described in the last section stems from representational misalignment. Representational misalignment is interesting in that, without explicit efforts to induce better alignment, misalignment should persist. As many VLM failures persist regardless of model scale, this raises the question whether additional failures may stem from representational misalignment as well \citep{diwan2022winoground,rahmanzadehgervi2024vision,vo2025vision}.

To test if representational alignment leads to fewer failures, we train 37 LLaVA models with different vision backbones and benchmark them on prediction tasks where VLM failures are common. All vision backbones use the same vision transformer architecture but have different pretraining objectives, such as image classification, multimodal contrastive losses, and so on.  Thus, the only difference between models is the backbone's resulting representations. For additional details see \Cref{app:align_failures}.

To test whether improved representational alignment results in fewer failures we use the following datasets: \textbf{Winoground} measures matching accuracies between images with syntactically similar but semantically different captions \citep{thrush2022winoground}. \textbf{VLMBias} tests whether VLMs can count objects in counterfactual images (e.g., the number of legs on a three-legged chicken) \citep{vo2025vision}. \textbf{VLMBlind} is a suite of simple geometric tasks like tracing paths in a simplified subway map \citep{rahmanzadehgervi2024vision}. \textbf{MMVP} tests if a model can differentiate between pairs of images which are ``CLIP-blind'', images which are similar per CLIP but dissimilar per other models \citep{tong2024eyes}.

\begin{figure}[!h]
    \vspace{-5pt}
    \centering
    \includegraphics[width=\linewidth]{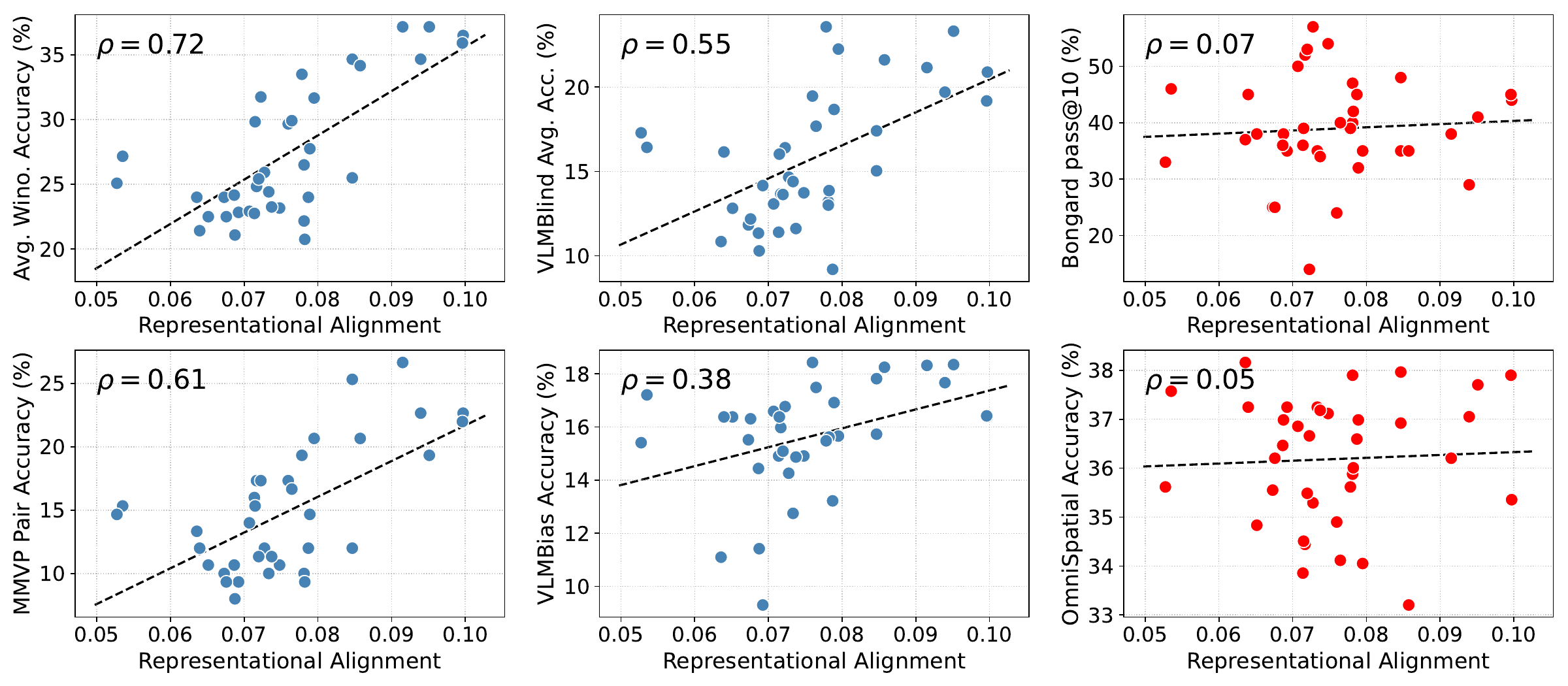}
    \vspace{-5pt}
    \caption{\textbf{Vision/language backbone representational alignment is correlated with better performance on benchmarks showcasing VLM failures} (left four plots). Each dot represents a LLaVA model trained with a different vision backbone. Representational alignment is measured using the mutual nearest neighbours similarity metric of \citet{huh2024platonic}. (right two plots) show that on datasets which are difficult due to the model not having other capabilities, e.g. reasoning, the correlation is much smaller.}
    \vspace{-7pt}
    \label{fig:repr_align_vlm_failures}
\end{figure}

Additionally, as a control, we test whether better representational alignment is correlated with better performance on datasets requiring the model to have properties orthogonal to representation alignment, such as reasoning. \textbf{Bongard} is a visual dataset consisting of geometric visual analogies that require reasoning about the relation between different sets of shapes \citep{wust2024bongard}. \textbf{OmniSpatial} is a visual question answering spatial reasoning dataset \citep{jia2026omnispatial}.

\Cref{fig:repr_align_vlm_failures} shows that models with relatively higher representational alignment between their vision and language backbones tend to perform better on benchmarks that showcase VLM failures. Notably, \Cref{fig:repr_align_vlm_failures} (right column) shows that the correlation is not simply due to some vision backbones having better representations, as for the reasoning datasets the correlation is close to zero.

\vspace{-9pt}
\section{Discussion}
\vspace{-7pt}
In this paper we provide evidence that existing vision and language representations are linearly unalignable. Although these results are only over existing models, they are expected to hold regardless of model size due to the underlying reasons \textit{why} the representations are misaligned.
In \Cref{sec:intermod_hard}, we find that the misalignment stems from the vision and language modalities having inherently different information; so what is similar in one may be dissimilar in another, and analogies may not carry over between modalities. As these differences are due to the data---not the models or training paradigm---they are expected to persist across all scales.

This implies that representations in different modalities should not converge---disproving a strong version of the Platonic Representation Hypothesis. Weaker versions of the hypothesis, where partial alignment is posited, can (and likely do) hold, consistent with prior work. Indeed, the weak trend in \Cref{fig:vislinga_retr} (left), where the vision models' retrieval accuracy increases with scale, supports this conclusion.

A natural question is whether alignment is possible using more complex, nonlinear adapters between modalities. As the misalignment is due to the data, better adapters could not lead to alignment, as it is impossible for the adapter to regain lost information. For example, captioning a picture filled with a single color as ``green'' does not convey the color's shade or hue.

More practically, representational misalignment has functional implications. We hypothesize that many noted VLM failure modes are due to representational misalignment, explaining why they persist in spite of continued improvements in model performance. The generalization failure demonstrated in \Cref{sec:vlms_failure} and the correlation between having fewer failures and representational alignment in \Cref{sec:repr_align_failures} both support this hypothesis.

There are several caveats to our conclusions, with a broader discussion deferred to Appendix \ref{app:limits}. Most notably, our results only intervene in one direction -- showing that representational misalignment can lead to failures -- and do not yet show that fixing the misalignment directly reduces the failures. While there are some works which broadly show that better alignment leads to improved performance \citep{ma2024learning,yamaguchi2025post,gan2026seeing}, finding a proper fix is an important avenue for future work. Moreover, although we show that representational alignment is linked to better results on datasets exhibiting VLM failures, this may only be a component and not the failures' sole cause. Some failures likely stem from other causes, such as a lack of reasoning capabilities.

\subsection*{Ethics statement}

We do not foresee any significant ethical consequences of this work. As far as we are aware, VisLingA does not contain concepts that would likely pose any safety risks. However, as it contains images from Wikimedia, it contains trace amounts of sensitive publicly available images, such as partial nudity and historical violence. Creator metadata for all images in VisLingA is kept, with only images in the public domain or with CC0/CC-BY copyright licenses being used.

\subsubsection*{Acknowledgments}
We would like to thank Charlie Tan for suggesting the use of zebras as a typical example, Dulhan Jayalath for reviewing a draft of this work, and Tymoteusz Wilk for help brainstorming the title. Yonatan is funded by the Rhodes Trust and the AIMS EPSRC CDT (grant no. EP/S024050/1).

\bibliography{iclr2027_conference}

@article{held2011newly,
  title={The newly sighted fail to match seen with felt},
  author={Held, Richard and Ostrovsky, Yuri and de Gelder, Beatrice and Gandhi, Tapan and Ganesh, Suma and Mathur, Umang and Sinha, Pawan},
  journal={Nature neuroscience},
  volume={14},
  number={5},
  pages={551--553},
  year={2011},
  publisher={Nature Publishing Group}
}

@misc{wiki:zebra,
  author = "Wikimedia Commons",
  title  = "Zebra --- Wikimedia Commons{,} the free media repository",
  year   = "2024",
  url    = "https://commons.wikimedia.org/wiki/Zebra",
}

@article{radford2019language,
  title={Language models are unsupervised multitask learners},
  author={Radford, Alec and Wu, Jeffrey and Child, Rewon and Luan, David and Amodei, Dario and Sutskever, Ilya and others},
  journal={OpenAI blog},
  volume={1},
  number={8},
  pages={9},
  year={2019}
}

@article{liu2023visual,
  title={Visual instruction tuning},
  author={Liu, Haotian and Li, Chunyuan and Wu, Qingyang and Lee, Yong Jae},
  journal={Advances in neural information processing systems},
  volume={36},
  pages={34892--34916},
  year={2023}
}

@article{Dosovitskiy2020AnII,
  title={An Image is Worth 16x16 Words: Transformers for Image Recognition at Scale},
  author={Alexey Dosovitskiy and Lucas Beyer and Alexander Kolesnikov and Dirk Weissenborn and Xiaohua Zhai and Thomas Unterthiner and Mostafa Dehghani and Matthias Minderer and Georg Heigold and Sylvain Gelly and Jakob Uszkoreit and Neil Houlsby},
  journal={ArXiv},
  year={2020},
  volume={abs/2010.11929},
  url={https://api.semanticscholar.org/CorpusID:225039882}
}

@inproceedings{radford2021learning,
  title={Learning transferable visual models from natural language supervision},
  author={Radford, Alec and Kim, Jong Wook and Hallacy, Chris and Ramesh, Aditya and Goh, Gabriel and Agarwal, Sandhini and Sastry, Girish and Askell, Amanda and Mishkin, Pamela and Clark, Jack and others},
  booktitle={International conference on machine learning},
  pages={8748--8763},
  year={2021},
  organization={PMLR}
}

@inproceedings{lin2014microsoft,
  title={Microsoft coco: Common objects in context},
  author={Lin, Tsung-Yi and Maire, Michael and Belongie, Serge and Hays, James and Perona, Pietro and Ramanan, Deva and Doll{\'a}r, Piotr and Zitnick, C Lawrence},
  booktitle={Computer Vision--ECCV 2014: 13th European Conference, Zurich, Switzerland, September 6-12, 2014, Proceedings, Part V 13},
  pages={740--755},
  year={2014},
  organization={Springer}
}

@inproceedings{he2022masked,
  title={Masked autoencoders are scalable vision learners},
  author={He, Kaiming and Chen, Xinlei and Xie, Saining and Li, Yanghao and Doll{\'a}r, Piotr and Girshick, Ross},
  booktitle={Proceedings of the IEEE/CVF conference on computer vision and pattern recognition},
  pages={16000--16009},
  year={2022}
}

@article{oquab2023dinov2,
  title={Dinov2: Learning robust visual features without supervision},
  author={Oquab, Maxime and Darcet, Timoth{\'e}e and Moutakanni, Th{\'e}o and Vo, Huy and Szafraniec, Marc and Khalidov, Vasil and Fernandez, Pierre and Haziza, Daniel and Massa, Francisco and El-Nouby, Alaaeldin and others},
  journal={arXiv preprint arXiv:2304.07193},
  year={2023}
}

@article{huh2024platonic,
  title={The platonic representation hypothesis},
  author={Huh, Minyoung and Cheung, Brian and Wang, Tongzhou and Isola, Phillip},
  journal={arXiv preprint arXiv:2405.07987},
  year={2024}
}

@article{mikolov2013efficient,
  title={Efficient estimation of word representations in vector space},
  author={Mikolov, Tomas and Chen, Kai and Corrado, Greg and Dean, Jeffrey},
  journal={arXiv preprint arXiv:1301.3781},
  year={2013}
}

@article{alayrac2022flamingo,
  title={Flamingo: a visual language model for few-shot learning},
  author={Alayrac, Jean-Baptiste and Donahue, Jeff and Luc, Pauline and Miech, Antoine and Barr, Iain and Hasson, Yana and Lenc, Karel and Mensch, Arthur and Millican, Katherine and Reynolds, Malcolm and others},
  journal={Advances in neural information processing systems},
  volume={35},
  pages={23716--23736},
  year={2022}
}

@article{mikolov2013exploiting,
  title={Exploiting similarities among languages for machine translation},
  author={Mikolov, Tomas and Le, Quoc V and Sutskever, Ilya},
  journal={arXiv preprint arXiv:1309.4168},
  year={2013}
}

@article{diwan2022winoground,
  title={Why is Winoground Hard? Investigating Failures in Visuolinguistic Compositionality},
  author={Diwan, Anuj and Berry, Layne and Choi, Eunsol and Harwath, David and Mahowald, Kyle},
  journal={arXiv preprint arXiv:2211.00768},
  year={2022}
}

@article{yiu2024kiva,
  title={Kiva: Kid-inspired visual analogies for testing large multimodal models},
  author={Yiu, Eunice and Qraitem, Maan and Majhi, Anisa Noor and Wong, Charlie and Bai, Yutong and Ginosar, Shiry and Gopnik, Alison and Saenko, Kate},
  journal={arXiv preprint arXiv:2407.17773},
  year={2024}
}

@article{vo2025vision,
  title={Vision Language Models are Biased},
  author={Vo, An and Nguyen, Khai-Nguyen and Taesiri, Mohammad Reza and Dang, Vy Tuong and Nguyen, Anh Totti and Kim, Daeyoung},
  journal={arXiv preprint arXiv:2505.23941},
  year={2025}
}

@inproceedings{thrush2022winoground,
  title={Winoground: Probing vision and language models for visio-linguistic compositionality},
  author={Thrush, Tristan and Jiang, Ryan and Bartolo, Max and Singh, Amanpreet and Williams, Adina and Kiela, Douwe and Ross, Candace},
  booktitle={Proceedings of the IEEE/CVF Conference on Computer Vision and Pattern Recognition},
  pages={5238--5248},
  year={2022}
}

@misc{nlp_connect_2022,
  author = { {NLP Connect} },
  title = { vit-gpt2-image-captioning (Revision 0e334c7) },
  year = 2022,
  url = { https://huggingface.co/nlpconnect/vit-gpt2-image-captioning },
  doi = { 10.57967/hf/0222 },
  publisher = { Hugging Face }
}

@article{merullo2022linearly,
  title={Linearly mapping from image to text space},
  author={Merullo, Jack and Castricato, Louis and Eickhoff, Carsten and Pavlick, Ellie},
  journal={arXiv preprint arXiv:2209.15162},
  year={2022}
}

@article{hu2022lora,
  title={Lora: Low-rank adaptation of large language models.},
  author={Hu, Edward J and Shen, Yelong and Wallis, Phillip and Allen-Zhu, Zeyuan and Li, Yuanzhi and Wang, Shean and Wang, Lu and Chen, Weizhu and others},
  journal={ICLR},
  volume={1},
  number={2},
  pages={3},
  year={2022}
}

@inproceedings{reimers-2019-sentence-bert,
    title = "Sentence-BERT: Sentence Embeddings using Siamese BERT-Networks",
    author = "Reimers, Nils and Gurevych, Iryna",
    booktitle = "Proceedings of the 2019 Conference on Empirical Methods in Natural Language Processing",
    month = "11",
    year = "2019",
    publisher = "Association for Computational Linguistics",
    url = "http://arxiv.org/abs/1908.10084",
}

@article{zhu2023minigpt,
  title={Minigpt-4: Enhancing vision-language understanding with advanced large language models},
  author={Zhu, Deyao and Chen, Jun and Shen, Xiaoqian and Li, Xiang and Elhoseiny, Mohamed},
  journal={arXiv preprint arXiv:2304.10592},
  year={2023}
}

@inproceedings{tong2024eyes,
  title={Eyes wide shut? exploring the visual shortcomings of multimodal llms},
  author={Tong, Shengbang and Liu, Zhuang and Zhai, Yuexiang and Ma, Yi and LeCun, Yann and Xie, Saining},
  booktitle={Proceedings of the IEEE/CVF Conference on Computer Vision and Pattern Recognition},
  pages={9568--9578},
  year={2024}
}

@techreport{gal2021robustness,
  author       = {Yarin Gal},
  title        = {The robustness of massive models pre-trained on web-scraped data is tightly tied to the way their data was curated},
  institution  = {Oxford Applied and Theoretical Machine Learning (OATML) Group, University of Oxford},
  type         = {Technical Report},
  year         = {2021},
  month        = jun,
  day          = {27},
  note         = {Blog post titled “Why I Wouldn't Trust OpenAI's CLIP to Drive My Car”},
  url          = {https://oatml.cs.ox.ac.uk/blog/2021/06/27/web-scraped-harmful.html}
}

@article{achiam2023gpt,
  title={Gpt-4 technical report},
  author={Achiam, Josh and Adler, Steven and Agarwal, Sandhini and Ahmad, Lama and Akkaya, Ilge and Aleman, Florencia Leoni and Almeida, Diogo and Altenschmidt, Janko and Altman, Sam and Anadkat, Shyamal and others},
  journal={arXiv preprint arXiv:2303.08774},
  year={2023}
}

@article{zhu2025dynamic,
  title={Dynamic Reflections: Probing Video Representations with Text Alignment},
  author={Zhu, Tyler and Han, Tengda and Guibas, Leonidas and P{\u{a}}tr{\u{a}}ucean, Viorica and Ovsjanikov, Maks},
  journal={arXiv preprint arXiv:2511.02767},
  year={2025}
}

@article{tjandrasuwita2025understanding,
  title={Understanding the emergence of multimodal representation alignment},
  author={Tjandrasuwita, Megan and Ekbote, Chanakya and Ziyin, Liu and Liang, Paul Pu},
  journal={arXiv preprint arXiv:2502.16282},
  year={2025}
}

@article{jha2025harnessing,
  title={Harnessing the universal geometry of embeddings},
  author={Jha, Rishi and Zhang, Collin and Shmatikov, Vitaly and Morris, John X},
  journal={arXiv preprint arXiv:2505.12540},
  year={2025}
}

@article{gupta2025better,
  title={Better together: Leveraging unpaired multimodal data for stronger unimodal models},
  author={Gupta, Sharut and Sundaram, Shobhita and Wang, Chenyu and Jegelka, Stefanie and Isola, Phillip},
  journal={arXiv preprint arXiv:2510.08492},
  year={2025}
}

@article{wang2025words,
  title={Words that make language models perceive},
  author={Wang, Sophie L and Isola, Phillip and Cheung, Brian},
  journal={arXiv preprint arXiv:2510.02425},
  year={2025}
}

@article{bahng2025cycle,
  title={Cycle Consistency as Reward: Learning Image-Text Alignment without Human Preferences},
  author={Bahng, Hyojin and Chan, Caroline and Durand, Fredo and Isola, Phillip},
  journal={arXiv preprint arXiv:2506.02095},
  year={2025}
}

@article{yuksekgonul2022and,
  title={When and why vision-language models behave like bags-of-words, and what to do about it?},
  author={Yuksekgonul, Mert and Bianchi, Federico and Kalluri, Pratyusha and Jurafsky, Dan and Zou, James},
  journal={arXiv preprint arXiv:2210.01936},
  year={2022}
}

@article{simeoni2025dinov3,
  title={Dinov3},
  author={Sim{\'e}oni, Oriane and Vo, Huy V and Seitzer, Maximilian and Baldassarre, Federico and Oquab, Maxime and Jose, Cijo and Khalidov, Vasil and Szafraniec, Marc and Yi, Seungeun and Ramamonjisoa, Micha{\"e}l and others},
  journal={arXiv preprint arXiv:2508.10104},
  year={2025}
}

@inproceedings{ni2022large,
  title={Large dual encoders are generalizable retrievers},
  author={Ni, Jianmo and Qu, Chen and Lu, Jing and Dai, Zhuyun and Abrego, Gustavo Hernandez and Ma, Ji and Zhao, Vincent and Luan, Yi and Hall, Keith and Chang, Ming-Wei and others},
  booktitle={Proceedings of the 2022 Conference on Empirical Methods in Natural Language Processing},
  pages={9844--9855},
  year={2022}
}

@inproceedings{cherti2023reproducible,
  title={Reproducible scaling laws for contrastive language-image learning},
  author={Cherti, Mehdi and Beaumont, Romain and Wightman, Ross and Wortsman, Mitchell and Ilharco, Gabriel and Gordon, Cade and Schuhmann, Christoph and Schmidt, Ludwig and Jitsev, Jenia},
  booktitle={Proceedings of the IEEE/CVF conference on computer vision and pattern recognition},
  pages={2818--2829},
  year={2023}
}

@article{smith2017offline,
  title={Offline bilingual word vectors, orthogonal transformations and the inverted softmax},
  author={Smith, Samuel L and Turban, David HP and Hamblin, Steven and Hammerla, Nils Y},
  journal={arXiv preprint arXiv:1702.03859},
  year={2017}
}

@inproceedings{kornblith2019similarity,
  title={Similarity of neural network representations revisited},
  author={Kornblith, Simon and Norouzi, Mohammad and Lee, Honglak and Hinton, Geoffrey},
  booktitle={International conference on machine learning},
  pages={3519--3529},
  year={2019},
  organization={PMlR}
}

@article{ramachandran2025well,
  title={How Well Does GPT-4o Understand Vision? Evaluating Multimodal Foundation Models on Standard Computer Vision Tasks},
  author={Ramachandran, Rahul and Garjani, Ali and Bachmann, Roman and Atanov, Andrei and Kar, O{\u{g}}uzhan Fatih and Zamir, Amir},
  journal={arXiv preprint arXiv:2507.01955},
  year={2025}
}

@inproceedings{rahmanzadehgervi2024vision,
  title={Vision language models are blind},
  author={Rahmanzadehgervi, Pooyan and Bolton, Logan and Taesiri, Mohammad Reza and Nguyen, Anh Totti},
  booktitle={Proceedings of the Asian Conference on Computer Vision},
  pages={18--34},
  year={2024}
}

@inproceedings{he2015delving,
  title={Delving deep into rectifiers: Surpassing human-level performance on imagenet classification},
  author={He, Kaiming and Zhang, Xiangyu and Ren, Shaoqing and Sun, Jian},
  booktitle={Proceedings of the IEEE international conference on computer vision},
  pages={1026--1034},
  year={2015}
}

@inproceedings{lenc2015understanding,
  title={Understanding image representations by measuring their equivariance and equivalence},
  author={Lenc, Karel and Vedaldi, Andrea},
  booktitle={Proceedings of the IEEE conference on computer vision and pattern recognition},
  pages={991--999},
  year={2015}
}

@article{bansal2021revisiting,
  title={Revisiting model stitching to compare neural representations},
  author={Bansal, Yamini and Nakkiran, Preetum and Barak, Boaz},
  journal={Advances in neural information processing systems},
  volume={34},
  pages={225--236},
  year={2021}
}

@article{csiszarik2021similarity,
  title={Similarity and matching of neural network representations},
  author={Csisz{\'a}rik, Adri{\'a}n and K{\H{o}}r{\"o}si-Szab{\'o}, P{\'e}ter and Matszangosz, Akos and Papp, Gergely and Varga, D{\'a}niel},
  journal={Advances in Neural Information Processing Systems},
  volume={34},
  pages={5656--5668},
  year={2021}
}

@inproceedings{smith2025functional,
  title={Functional Alignment Can Mislead: Examining Model Stitching},
  author={Smith, Damian and Mannering, Harvey and Marcu, Antonia},
  booktitle={Forty-second International Conference on Machine Learning},
  year={2025}
}

@article{klabunde2025similarity,
  title={Similarity of neural network models: A survey of functional and representational measures},
  author={Klabunde, Max and Schumacher, Tobias and Strohmaier, Markus and Lemmerich, Florian},
  journal={ACM Computing Surveys},
  volume={57},
  number={9},
  pages={1--52},
  year={2025},
  publisher={ACM New York, NY}
}

@article{goyal2025vla,
  title={Vla-0: Building state-of-the-art vlas with zero modification},
  author={Goyal, Ankit and Hadfield, Hugo and Yang, Xuning and Blukis, Valts and Ramos, Fabio},
  journal={arXiv preprint arXiv:2510.13054},
  year={2025}
}

@inproceedings{wang2025vggt,
  title={Vggt: Visual geometry grounded transformer},
  author={Wang, Jianyuan and Chen, Minghao and Karaev, Nikita and Vedaldi, Andrea and Rupprecht, Christian and Novotny, David},
  booktitle={Proceedings of the Computer Vision and Pattern Recognition Conference},
  pages={5294--5306},
  year={2025}
}

@article{liang2022mind,
  title={Mind the gap: Understanding the modality gap in multi-modal contrastive representation learning},
  author={Liang, Victor Weixin and Zhang, Yuhui and Kwon, Yongchan and Yeung, Serena and Zou, James Y},
  journal={Advances in Neural Information Processing Systems},
  volume={35},
  pages={17612--17625},
  year={2022}
}

@inproceedings{he2025seeing,
  title={Seeing through words, speaking through pixels: Deep representational alignment between vision and language models},
  author={He, Zoe Wanying and Trott, Sean and Khosla, Meenakshi},
  booktitle={Proceedings of the 2025 Conference on Empirical Methods in Natural Language Processing},
  pages={35645--35660},
  year={2025}
}

@article{udandarao2022understanding,
  title={Understanding and fixing the modality gap in vision-language models},
  author={Udandarao, Vishaal},
  journal={Master's thesis, University of Cambridge},
  volume={32},
  year={2022}
}

@article{shu2025large,
  title={Large vision-language model alignment and misalignment: A survey through the lens of explainability},
  author={Shu, Dong and Zhao, Haiyan and Hu, Jingyu and Liu, Weiru and Payani, Ali and Cheng, Lu and Du, Mengnan},
  journal={arXiv preprint arXiv:2501.01346},
  year={2025}
}

@inproceedings{park2024linear,
  title={The Linear Representation Hypothesis and the Geometry of Large Language Models},
  author={Park, Kiho and Choe, Yo Joong and Veitch, Victor},
  booktitle={International Conference on Machine Learning},
  pages={39643--39666},
  year={2024},
  organization={PMLR}
}

@article{piantadosi2024concepts,
  title={Why concepts are (probably) vectors},
  author={Piantadosi, Steven T and Muller, Dyana CY and Rule, Joshua S and Kaushik, Karthikeya and Gorenstein, Mark and Leib, Elena R and Sanford, Emily},
  journal={Trends in Cognitive Sciences},
  volume={28},
  number={9},
  pages={844--856},
  year={2024},
  publisher={Elsevier}
}

@article{wang2023concept,
  title={Concept algebra for (score-based) text-controlled generative models},
  author={Wang, Zihao and Gui, Lin and Negrea, Jeffrey and Veitch, Victor},
  journal={Advances in Neural Information Processing Systems},
  volume={36},
  pages={35331--35349},
  year={2023}
}

@inproceedings{trager2023linear,
  title={Linear spaces of meanings: compositional structures in vision-language models},
  author={Trager, Matthew and Perera, Pramuditha and Zancato, Luca and Achille, Alessandro and Bhatia, Parminder and Soatto, Stefano},
  booktitle={Proceedings of the IEEE/CVF International Conference on Computer Vision},
  pages={15395--15404},
  year={2023}
}

@inproceedings{merullo2024language,
  title={Language models implement simple word2vec-style vector arithmetic},
  author={Merullo, Jack and Eickhoff, Carsten and Pavlick, Ellie},
  booktitle={Proceedings of the 2024 Conference of the North American Chapter of the Association for Computational Linguistics: Human Language Technologies (Volume 1: Long Papers)},
  pages={5030--5047},
  year={2024}
}

@article{turner2023steering,
  title={Steering language models with activation engineering},
  author={Turner, Alexander Matt and Thiergart, Lisa and Leech, Gavin and Udell, David and Vazquez, Juan J and Mini, Ulisse and MacDiarmid, Monte},
  journal={arXiv preprint arXiv:2308.10248},
  year={2023}
}

@inproceedings{drozd2016word,
  title={Word embeddings, analogies, and machine learning: Beyond king-man+ woman= queen},
  author={Drozd, Aleksandr and Gladkova, Anna and Matsuoka, Satoshi},
  booktitle={Proceedings of coling 2016, the 26th international conference on computational linguistics: Technical papers},
  pages={3519--3530},
  year={2016}
}

@inproceedings{rogers2017too,
  title={The (too many) problems of analogical reasoning with word vectors},
  author={Rogers, Anna and Drozd, Aleksandr and Li, Bofang},
  booktitle={Proceedings of the 6th joint conference on lexical and computational semantics (* SEM 2017)},
  pages={135--148},
  year={2017}
}

@article{pawlonka2025bongard,
  title={Bongard-rwr+: Real-world representations of fine-grained concepts in bongard problems},
  author={Pawlonka, Szymon and Ma{\l}ki{\'n}ski, Miko{\l}aj and Ma{\'n}dziuk, Jacek},
  journal={arXiv preprint arXiv:2508.12026},
  year={2025}
}

@article{khangaonkar2026multimodal,
  title={Multimodal Language Models Cannot Spot Spatial Inconsistencies},
  author={Khangaonkar, Om and Rad, Hadi J and Pirsiavash, Hamed},
  journal={arXiv preprint arXiv:2604.00799},
  year={2026}
}

@article{campbell2024understanding,
  title={Understanding the limits of vision language models through the lens of the binding problem},
  author={Campbell, Declan and Rane, Sunayana and Giallanza, Tyler and De Sabbata, Nicol{\`o} and Ghods, Kia and Joshi, Amogh and Ku, Alexander and Frankland, Steven M and Griffiths, Thomas L and Cohen, Jonathan D and others},
  journal={Advances in Neural Information Processing Systems},
  volume={37},
  pages={113436--113460},
  year={2024}
}

@article{datacomp,
  title={Datacomp: In search of the next generation of multimodal datasets},
  author={Gadre, Samir Yitzhak and Ilharco, Gabriel and Fang, Alex and Hayase, Jonathan and Smyrnis, Georgios and Nguyen, Thao and Marten, Ryan and Wortsman, Mitchell and Ghosh, Dhruba and Zhang, Jieyu and others},
  journal={Advances in Neural Information Processing Systems},
  volume={36},
  pages={27092--27112},
  year={2023}
}

@inproceedings{caron2021emerging,
  title={Emerging properties in self-supervised vision transformers},
  author={Caron, Mathilde and Touvron, Hugo and Misra, Ishan and J{\'e}gou, Herv{\'e} and Mairal, Julien and Bojanowski, Piotr and Joulin, Armand},
  booktitle={Proceedings of the IEEE/CVF international conference on computer vision},
  pages={9650--9660},
  year={2021}
}

@inproceedings{zhai2023sigmoid,
  title={Sigmoid loss for language image pre-training},
  author={Zhai, Xiaohua and Mustafa, Basil and Kolesnikov, Alexander and Beyer, Lucas},
  booktitle={Proceedings of the IEEE/CVF international conference on computer vision},
  pages={11975--11986},
  year={2023}
}

@article{tschannen2025siglip,
  title={Siglip 2: Multilingual vision-language encoders with improved semantic understanding, localization, and dense features},
  author={Tschannen, Michael and Gritsenko, Alexey and Wang, Xiao and Naeem, Muhammad Ferjad and Alabdulmohsin, Ibrahim and Parthasarathy, Nikhil and Evans, Talfan and Beyer, Lucas and Xia, Ye and Mustafa, Basil and others},
  journal={arXiv preprint arXiv:2502.14786},
  year={2025}
}

@article{bao2021beit,
  title={Beit: Bert pre-training of image transformers},
  author={Bao, Hangbo and Dong, Li and Piao, Songhao and Wei, Furu},
  journal={arXiv preprint arXiv:2106.08254},
  year={2021}
}

@inproceedings{touvron2021training,
  title={Training data-efficient image transformers \& distillation through attention},
  author={Touvron, Hugo and Cord, Matthieu and Douze, Matthijs and Massa, Francisco and Sablayrolles, Alexandre and J{\'e}gou, Herv{\'e}},
  booktitle={International conference on machine learning},
  pages={10347--10357},
  year={2021},
  organization={PMLR}
}

@article{wust2024bongard,
  title={Bongard in Wonderland: Visual Puzzles that Still Make AI Go Mad?},
  author={W{\"u}st, Antonia and Woydt, Tim and Helff, Lukas and Ibs, Inga and Stammer, Wolfgang and Dhami, Devendra S and Rothkopf, Constantin A and Kersting, Kristian},
  journal={arXiv preprint arXiv:2410.19546},
  year={2024}
}

@inproceedings{jia2026omnispatial,
  title={Omnispatial: Towards comprehensive spatial reasoning benchmark for vision language models},
  author={Jia, Mengdi and Qi, Zekun and Zhang, Shaochen and Zhang, Wenyao and Yu, Xinqiang and He, Jiawei and Wang, He and Yi, Li},
  booktitle={International Conference on Learning Representations},
  volume={2026},
  pages={35634--35670},
  year={2026}
}

@article{dhimoila2026unifying,
  title={A Unifying Framework for Concept-Based Representational Similarity},
  author={Dhimo{\"\i}la, Gr{\'e}goire and Boutin, Victor and Picard, Agustin Martin and Fel, Thomas and Serre, Thomas},
  journal={arXiv preprint arXiv:2606.09653},
  year={2026}
}

@inproceedings{schrodi2025two,
  title={Two effects, one trigger: On the modality gap, object bias, and information imbalance in contrastive vision-language models},
  author={Schrodi, Simon and Hoffmann, David T and Argus, Max and Fischer, Volker and Brox, Thomas},
  booktitle={International Conference on Learning Representations},
  volume={2025},
  pages={27836--27864},
  year={2025}
}

@inproceedings{gan2026seeing,
  title={Seeing Helps Reasoning in Language Models},
  author={Gan, Yulu and Zhao, Kaiya Ivy and Poggio, Tomaso and Isola, Phillip},
  booktitle={Proceedings of the IEEE/CVF Conference on Computer Vision and Pattern Recognition},
  pages={7080--7090},
  year={2026}
}

@inproceedings{yamaguchi2025post,
  title={Post-pre-training for modality alignment in vision-language foundation models},
  author={Yamaguchi, Shin'ya and Feng, Dewei and Kanai, Sekitoshi and Adachi, Kazuki and Chijiwa, Daiki},
  booktitle={Proceedings of the IEEE/CVF Conference on Computer Vision and Pattern Recognition},
  pages={4256--4266},
  year={2025}
}

@inproceedings{
ma2024learning,
title={Learning Modality Knowledge Alignment for Cross-Modality Transfer},
author={Wenxuan Ma and Shuang Li and Lincan Cai and Jingxuan Kang},
booktitle={Forty-first International Conference on Machine Learning},
year={2024},
url={https://openreview.net/forum?id=lmiurzioja}
}

@inproceedings{
groger2026revisiting,
title={Revisiting the Platonic Representation Hypothesis: An Aristotelian View},
author={Fabian Gr{\"o}ger and Shuo Wen and Maria Brbic},
booktitle={Forty-third International Conference on Machine Learning},
year={2026},
url={https://openreview.net/forum?id=uz0gAAYydl}
}

@article{li2024vision,
  title={Do vision and language models share concepts? a vector space alignment study},
  author={Li, Jiaang and Kementchedjhieva, Yova and Fierro, Constanza and S{\o}gaard, Anders},
  journal={Transactions of the Association for Computational Linguistics},
  volume={12},
  pages={1232--1249},
  year={2024},
  publisher={MIT Press 255 Main Street, 9th Floor, Cambridge, Massachusetts 02142, USA~…}
}

@article{sucholutsky2023getting,
  title={Getting aligned on representational alignment},
  author={Sucholutsky, Ilia and Muttenthaler, Lukas and Weller, Adrian and Peng, Andi and Bobu, Andreea and Kim, Been and Love, Bradley C and Cueva, Christopher J and Grant, Erin and Groen, Iris and others},
  journal={arXiv preprint arXiv:2310.13018},
  year={2023}
}

@article{tang2025fusionbench,
  title   = {FusionBench: A Comprehensive Benchmark of Deep Model Fusion},
  author  = {Tang, Anke and Shen, Li and Luo, Yong and Yang, Enneng and Hu, Han and Zhang, Lefei and Du, Bo and Tao, Dacheng},
  journal = {Journal of Machine Learning Research},
  year    = {2025},
  note    = {arXiv:2406.03280}
}

@misc{iakubovskyi2024fairfaceage,
  author       = {Iakubovskyi, Dmytro},
  title        = {{fairface\_age\_image\_detection}},
  year         = {2024},
  publisher    = {Hugging Face},
  howpublished = {\url{https://huggingface.co/dima806/fairface_age_image_detection}}
}

@misc{iakubovskyi2023manwomanface,
  author       = {Iakubovskyi, Dmytro},
  title        = {{man\_woman\_face\_image\_detection}},
  year         = {2023},
  publisher    = {Hugging Face},
  howpublished = {\url{https://huggingface.co/dima806/man_woman_face_image_detection}}
}

@misc{iakubovskyi2023facialemotions,
  author       = {Iakubovskyi, Dmytro},
  title        = {{facial\_emotions\_image\_detection}},
  year         = {2023},
  publisher    = {Hugging Face},
  howpublished = {\url{https://huggingface.co/dima806/facial_emotions_image_detection}}
}

@misc{iakubovskyi2023deepfakevsreal,
  author       = {Iakubovskyi, Dmytro},
  title        = {{deepfake\_vs\_real\_image\_detection}},
  year         = {2023},
  publisher    = {Hugging Face},
  howpublished = {\url{https://huggingface.co/dima806/deepfake_vs_real_image_detection}}
}

@misc{iakubovskyi2023aivsreal,
  author       = {Iakubovskyi, Dmytro},
  title        = {{ai\_vs\_real\_image\_detection}},
  year         = {2023},
  publisher    = {Hugging Face},
  howpublished = {\url{https://huggingface.co/dima806/ai_vs_real_image_detection}}
}

@misc{iakubovskyi2024fairfacegender,
  author       = {Iakubovskyi, Dmytro},
  title        = {{fairface\_gender\_image\_detection}},
  year         = {2024},
  publisher    = {Hugging Face},
  howpublished = {\url{https://huggingface.co/dima806/fairface_gender_image_detection}}
}

@misc{iakubovskyi2023closedeyes,
  author       = {Iakubovskyi, Dmytro},
  title        = {{closed\_eyes\_image\_detection}},
  year         = {2023},
  publisher    = {Hugging Face},
  howpublished = {\url{https://huggingface.co/dima806/closed_eyes_image_detection}}
}

@misc{enright2023rorshark,
  author       = {Enright, Chester},
  title        = {{rorshark-vit-base}},
  year         = {2023},
  publisher    = {Hugging Face},
  howpublished = {\url{https://huggingface.co/amunchet/rorshark-vit-base}}
}

@misc{sakthi2025deepfakedetectorv2,
  author       = {prithivMLmods},
  title        = {{Deep-Fake-Detector-v2-Model}},
  year         = {2025},
  publisher    = {Hugging Face},
  howpublished = {\url{https://huggingface.co/prithivMLmods/Deep-Fake-Detector-v2-Model}},
  note         = {Initial release 21 Mar 2024; updated 02 Feb 2025}
}

@misc{sakthi2025watermarksiglip2,
  author       = {Sakthi, Prithiv},
  title        = {{Watermark-Detection-SigLIP2}},
  year         = {2025},
  publisher    = {Hugging Face},
  howpublished = {\url{https://huggingface.co/prithivMLmods/Watermark-Detection-SigLIP2}}
}

@misc{sakthi2025agesiglip2,
  author       = {Sakthi, Prithiv},
  title        = {{Age-Classification-SigLIP2}},
  year         = {2025},
  publisher    = {Hugging Face},
  howpublished = {\url{https://huggingface.co/prithivMLmods/Age-Classification-SigLIP2}}
}

@misc{sakthi2025genderclassifiermini,
  author       = {Sakthi, Prithiv},
  title        = {{Gender-Classifier-Mini}},
  year         = {2025},
  publisher    = {Hugging Face},
  howpublished = {\url{https://huggingface.co/prithivMLmods/Gender-Classifier-Mini}}
}

@misc{nurul2023watermarkdetector,
  author       = {Nurul, Amanda},
  title        = {{watermark\_detector}},
  year         = {2023},
  publisher    = {Hugging Face},
  howpublished = {\url{https://huggingface.co/amrul-hzz/watermark_detector}}
}

@misc{codewithdark2025vitchestxray,
  author       = {{Dark Coder}},
  title        = {{vit-chest-xray}},
  year         = {2025},
  publisher    = {Hugging Face},
  howpublished = {\url{https://huggingface.co/codewithdark/vit-chest-xray}}
}

@misc{sharma2024ulbase,
  author       = {Sharma, Jai},
  title        = {{UL\_base\_classification}},
  year         = {2024},
  publisher    = {Hugging Face},
  howpublished = {\url{https://huggingface.co/sharmajai901/UL_base_classification}}
}

@misc{sharma2024ulinterior,
  author       = {Sharma, Jai},
  title        = {{UL\_interior\_classification}},
  year         = {2024},
  publisher    = {Hugging Face},
  howpublished = {\url{https://huggingface.co/sharmajai901/UL_interior_classification}}
}

@misc{sharma2024ulbedroom,
  author       = {Sharma, Jai},
  title        = {{UL\_bedroom\_classification}},
  year         = {2024},
  publisher    = {Hugging Face},
  howpublished = {\url{https://huggingface.co/sharmajai901/UL_bedroom_classification}}
}

@misc{sahoo2025faceemotion,
  author       = {Abhilash},
  title        = {{ViT Face Emotion Detection}},
  year         = {2025},
  publisher    = {Hugging Face},
  howpublished = {\url{https://huggingface.co/abhilash88/face-emotion-detection}}
}

@misc{ng2023blurvsclean,
  author       = {Ng, Harry},
  title        = {{vit-base-patch16-224-blur\_vs\_clean}},
  year         = {2023},
  publisher    = {Hugging Face},
  howpublished = {\url{https://huggingface.co/harrytechiz/vit-base-patch16-224-blur_vs_clean}}
}

@misc{vishwakarma2023humanaction,
  author       = {Vishwakarma, Rahul},
  title        = {{Human-Action-Recognition-VIT-Base-patch16-224}},
  year         = {2023},
  publisher    = {Hugging Face},
  howpublished = {\url{https://huggingface.co/rvv-karma/Human-Action-Recognition-VIT-Base-patch16-224}}
}

@inproceedings{van2026same,
  title={Same content, different answers: Cross-modal inconsistency in mllms},
  author={van Sprang, Angela and Samson, Laurens and Lucic, Ana and Acar, Erman and Ghebreab, Sennay and Asano, Yuki M},
  booktitle={Proceedings of the IEEE/CVF Conference on Computer Vision and Pattern Recognition},
  pages={8781--8790},
  year={2026}
}

@incollection{hotelling1992relations,
  title={Relations between two sets of variates},
  author={Hotelling, Harold},
  booktitle={Breakthroughs in statistics: methodology and distribution},
  pages={162--190},
  year={1992},
  publisher={Springer}
}

@article{morcos2018insights,
  title={Insights on representational similarity in neural networks with canonical correlation},
  author={Morcos, Ari and Raghu, Maithra and Bengio, Samy},
  journal={Advances in neural information processing systems},
  volume={31},
  year={2018}
}

@article{ramsay1984matrix,
  title={Matrix correlation},
  author={Ramsay, James O and ten Berge, Jos and Styan, George PH},
  journal={Psychometrika},
  volume={49},
  number={3},
  pages={403--423},
  year={1984},
  publisher={Cambridge University Press \& Assessment}
}
\bibliographystyle{iclr2027_conference}

\appendix

\crefalias{section}{appsec}
\crefalias{subsection}{appsec}
\crefalias{subsubsection}{appsec}

\section{Related Work}
\label{app:rel_work}
\textbf{Adding modalities can harm performance.} The generalization failure and results on representational misalignment here demonstrate that when representations are unalignable then stitching them together leads to worse performance, expressed as a lack of generalization. While this work focuses on tasks with language outputs, \citet{ramachandran2025well} similarly find that state-of-the-art VLMs perform worse than vision-only models on standard vision tasks, although without interpreting why. Our work fills this gap.

A converse conclusion would be that removing modalities could improve generalization. This was observed, albeit not understood, in \citet{goyal2025vla}. \citet{goyal2025vla} use a VLM for a robot's vision-language-action model (VLA) without explicitly mapping into action-space, instead conveying actions as text. Their approach outperforms prior state-of-the-art methods that explicitly add the action modality.

There are cases when adding modalities would be expected to improve performance, but those likely require only in-distribution generalization. When the additional modalities enable training from large-scale data \citep{radford2021learning} or better supervision, e.g. using multi-task learning \citep{wang2025vggt}, it is likely that in-distribution performance will improve.

\textbf{Representational misalignment.} Many works observe some form of multimodal representational misalignment, but often attribute it to a lack of scale (as would be expected per the Platonic Representation Hypothesis \citep{huh2024platonic}) or do not interpret it; see \citet{shu2025large} for a survey. Notably, \citet{liang2022mind} note there being a ``modality gap'' between vision and language representations in CLIP-like contrastive learning, attributing it not to the inherent information difference in the modalities but to model initialization and optimization. This viewpoint is contested \citep{udandarao2022understanding}, with \citet{schrodi2025two} arguing, similarly to this work, that the modality gap is due to inherent information differences in the data.

Several prior works note caveats to the Platonic Representation Hypothesis. \citet{groger2026revisiting} note that representations can seem to converge simply due to confounders, e.g. wider models can have higher representational similarity. The experiments in \Cref{sec:repr_align_failures} are designed so all models have the same architecture, circumventing this pitfall. \citet{li2024vision} speculate that some representations might be inherently unalignable, such as vision and language, with our work providing evidence this is indeed the case.

\textbf{Analogies in representations.} Analogies have been seen to exist in representations since the seminal work of \citet{mikolov2013efficient}, with many contemporary works assuming or relying on analogies in representations \citep{trager2023linear,turner2023steering,wang2023concept,piantadosi2024concepts,merullo2024language}. Analogies have been argued to arise due to concepts being represented as linear directions in representation spaces (``the Linear Representation Hypothesis'', \citep{park2024linear}). However, this viewpoint has some known limitations, especially when analogies are used for retrieval \citep{drozd2016word,rogers2017too}.

\textbf{VLM failures.} VLMs exhibit many simple failures, some of which are not analyzed here. \citet{yuksekgonul2022and} demonstrate instances where VLMs can act as bags-of-words and propose a contrastive learning method to mitigate this behaviour. \citet{tong2024eyes} demonstrate VLMs giving systematically wrong answers over certain curated visual question answering datasets, likely due to their CLIP vision backbones. They explore the design space of VLMs and propose changes that improve performance. \citet{pawlonka2025bongard} show that VLMs struggle perceiving fine-grained concepts in everyday real-world images. \citet{khangaonkar2026multimodal} find that VLMs struggle detecting spatial inconsistencies between sets of images, e.g. a rotated chair. \citet{van2026same} show that some reasoning inconsistencies likely arise due to representational misalignment, which could imply that the analysis in section \ref{sec:repr_align_failures} applies also to some reasoning tasks. \citet{campbell2024understanding} posit that various VLM failures are due to the binding problem, where models fail to associate together different features of an object.

\section{Bounds On Some Representational Alignment Metrics}
\label{app:repr_sim_lin_align}
In this appendix we show how if two sets of representations are linearly alignable only to some given extent then some representational alignment metrics are bounded in some sense by the degree of alignment. It is easiest showing these results when using different measures of linear alignability for different similarity metrics, due to how different similarity metrics measure different properties of the representations. However, any alignability metrics we use being zero implies the two representations are linearly alignable.

\textbf{Notation.} We first introduce some notation, mostly following that of \citet{kornblith2019similarity}. Intuitively, given corresponding sets of representations $Y\in\mathbb{R}^{n\times p_1},X\in\mathbb{R}^{n\times p_2}$,\footnote{Note we assume $X$ has at least as many features as $Y$ while \citet{kornblith2019similarity} assume the opposite; this is purely notational but makes some proofs more intuitive later.} $X$ is linearly alignable to $Y$ if there exists a weight matrix $W\in\mathbb{R}^{p_2\times p_1}$ such that $XW=Y$. $X$ is a matrix consisting of $n$ representations with $p_2$ features, with $Y$ being another such matrix but with $p_1$ features. Following \citet{kornblith2019similarity}, we assume that $p_1\leq p_2$ and that $X,Y$ have been centred, or in other words, have zero mean over their columns.

\textbf{CCA.} Canonical correlation analysis (CCA) finds linear transformations that maximize the correlations between two sets of representations \citep{hotelling1992relations, morcos2018insights}. One similarity index given by CCA is Yanai's GCD measure \citep{ramsay1984matrix,kornblith2019similarity}, which between $X$ and $Y$ is $R^2_\text{CCA}\coloneqq\frac{||Q_Y^TQ_X||_F^2}{p_1}$, where $Q_Y,Q_X$ are matrices with orthonormal columns spanning the columns of $Y$ and $X$ respectively.

\textbf{CKA.} Centred kernel alignment (CKA) is an extension of some prior representational similarity indices, which is notably invariant to orthogonal transformations but not to linear maps in general \citep{kornblith2019similarity}. \citet{kornblith2019similarity} showed that in many cases CKA better identifies intuitively corresponding representations than other similarity indices. Formally, CKA is defined as $\text{CKA}\coloneqq \frac{||Y^TX||_F^2}{||X^TX||_F||Y^TY||_F}$.

\textbf{Limited linear alignability bounds CCA.} Arguably the simplest alignability metric is $l\coloneqq \min_W||Y-XW||_F^2$, where $l=0$ if there exists a weight matrix such that $Y=XW$ and is nonzero otherwise. Here we show how $l$ is bounded if $X,Y$ are not linearly alignable, and accordingly leads to a bounded CCA. Note that we do not assume $X,Y$, or $W$ have any special form, beyond that $X$ has at least as many features as $Y$.

Note that $Y$ decomposes into a matrix contained in $X$'s column space and a matrix orthogonal to it. $l=0$ if $Y$ is wholly contained in $X$'s column space, as then it is a linear combination of $X$'s columns, in which case $X$ and $Y$ are linearly alignable. More generally, denote by $P_X$ a projection onto $X$'s column space, such that $Y=P_XY+(I-P_X)Y$. Thus, note that $l=\min_W||Y-XW||_F^2=\min_W||(I-P_X)Y+P_XY-XW||_F^2$. As $P_XY-XW$ is contained in $X$'s columns while $(I-P_X)Y$ is orthogonal to it, from Pythagoras' theorem we have that $l=||(I-P_X)Y||_F^2+\min_W||P_XY-XW||_F^2=||(I-P_X)Y||_F^2$, where the second term is zero as $P_XY$ is in the column space of $X$, which is spanned by $XW$.

Thus, $l=||(I-P_X)Y||_F^2$. To relate $Y$ to how it appears in CCA's similarity index, note that each of its columns is a vector in its column space. Intuitively, if $q_1,...,q_{p_1}$ is an orthonormal basis for $Y$'s column space, then each vector in the column space can be decomposed as $y=\sum_{i=1}^{p_1}(q_i^Ty)q_i$. In matrix form, this gives the identity $Y=Q_YQ^T_YY$, with $Q_Y$ being an orthonormal basis column for $Y$'s column space. More generally, for any $Z$ we have that $P_Z=Q_ZQ_Z^T$, where $P_Z$ is the projection matrix onto $Z$'s column space and $Q_Z$ is a matrix of that space's orthonormal basis.

Now, note that as $||AB||_F^2\leq||A||_F^2||B||_F^2$, we have that $l=||(I-P_X)Y||_F^2=||(I-P_X)Q_YQ_Y^TY||_F^2$, as $Q_Y$ acts as the identity on $Y$. Using the inequality, we have that $l\leq ||(I-P_X)Q_Y||_F^2||Q_Y^TY||_F^2=||(I-P_X)Q_Y||_F^2||Y||_F^2$, with the last step being due to $Q_YQ_Y^T$ acting as the identity on $Y$, where $||Q_Y^TY||_F^2=\Tr(Y^TQ_YQ_Y^TY)=\Tr(Y^TY)=||Y||_F^2$.

We now connect this quantity to CCA. As the Frobenius norm is invariant to orthonormal matrices, note that $||Q_X^TQ_Y||_F^2=||Q_XQ_X^TQ_Y||_F^2=||P_XQ_Y||_F^2$. We now note that this is the complementary part of a projection term that bounds $l$, namely the $||(I-P_X)Q_Y||_F^2$ term. Thus, using Pythagoras' theorem, $||(I-P_X)Q_Y||_F^2+||P_XQ_Y||_F^2=||Q_Y||_F^2\leq p_1$, with the last step being due to $Q_Y$ being an orthonormal column matrix, with the inequality being tight when $Y$ is full rank. Plugging these results together, we have that:
\begin{equation*}
    R^2_\text{CCA}=\frac{||Q_X^TQ_Y||_F^2}{p_1}=\frac{||Q_Y||_F^2-||(I-P_X)Q_Y||_F^2}{p_1}\leq \frac{p_1-||(I-P_X)Q_Y||_F^2}{p_1}\leq1-\frac{l}{||Y||_F^2p_1}.
\end{equation*}

Thus, if the linear alignability is bounded, where $l>\varepsilon>0$, then $R^2_\text{CCA}<1-\frac{\varepsilon}{||Y||_F^2p_1}$ is bounded as well. Note that this bound is quite loose and for illustration only, as there are various places in the derivation where it could be made tighter.

\textbf{Different similarity structures bound CKA.} 
In section \ref{sec:intermod_hard} we show that if two representations have the same similarity structure then there is a linear mapping between them. Here we expand on what happens when their similarity structures differ; namely, how larger differences between the similarity structures generally lead to lower CKA.

Denote by $K_X\coloneqq XX^T, K_Y\coloneqq YY^T$ the Gram matrices for $X$ and $Y$ respectively, which as per definition \ref{def:simstruct} denote each representation's similarity structure. Denote by $\hat{K}$ a normalized Gram matrix, where $\hat{K}\coloneqq\frac{K}{||K||_F}$. Thus, a natural scale-invariant notion for the difference between two similarity structures is $d\coloneqq ||\hat{K}_Y-\hat{K_X}||_F^2$. Linking this to CKA, note that $\Tr(\hat{K}_Y^T\hat{K}_X)=\frac{\Tr(K_Y^TK_X)}{||K_Y||_F||K_X||_F}=\text{CKA}(X,Y)$. Thus, by expanding $d$ we find that $d=||\hat{K}_Y||_F^2+||\hat{K}_X||_F^2-2\Tr(\hat{K}_Y^T\hat{K}_X)=2(1-\text{CKA}(X,Y))$.

Thus, if the difference between the normalized similarity structures is bounded, so $d>\varepsilon$, the CKA between them is bounded as well, where $\text{CKA}(X,Y)<1-\frac{\varepsilon}{2}$.

\section{Proofs of Theorems}
\label{app:thm_proof}
We restate the theorems here for convenience. We believe essentially all of our theorems have standard linear algebra versions that are generally known but we did not find a source for some, so we provide short proofs for them instead. For ease of notation we drop the concepts in the notation, so $v_{X_i}=v_i$. We first prove \Cref{thm:gen_analogs} as \Cref{thm:similarities_enable_lin}'s proof relies on it, it being the weaker of the two.

\textbf{\Cref{thm:gen_analogs}} (Generalized analogies enable linear mappings.) \textit{Let $V$ and $W$ have corresponding sets of concept representations, $\{v_i\}_{i=1}^N$ and $\{w_i\}_{i=1}^N$. A linear map $T$ such that $\forall i:T(v_i)=w_i$ exists if and only if for all sets of scalars $a_i\in\mathbb{R}$ if $\sum_i a_i v_i=0$ then $\sum_i a_i w_i=0$.}

\textit{Proof.} It is simple to see that having such a $T$ is sufficient, as $\sum_{i=1}^N a_iv_i=0\Rightarrow T(\sum_{i=1}^N a_iv_i)=T(0)=0$, and due to linearity $T(\sum_{i=1}^N a_iv_i)=\sum_{i=1}^N a_iT(v_i)=\sum_{i=1}^N a_iw_i$, which completes this part of the proof.

As for necessity, assume that $\forall a_i\in\mathbb{R}:\sum_{i=1}^N a_iv_i=0\Rightarrow\sum_{i=1}^N a_iw_i=0$. We prove this using induction on $N$, the number of vectors. For the base case this is trivial, as only for $a_1=0$ the statement holds and the map $\forall a\in\mathbb{R}:T(av_1)=aw_1$ is clearly a valid linear transformation. For induction, assume the statement holds for $N-1$ and we wish to show it for $N$. If $v_N$ is linearly dependent on $\{v_i\}_{i=1}^{N-1}$ then $\exists a_i\in\mathbb{R}:v_N=\sum_{i=1}^{N-1}a_iv_i$, so $-v_N+\sum_{i=1}^{N-1}a_iv_i=0\Rightarrow -w_N+\sum_{i=1}^{N-1}a_iw_i=0\Rightarrow w_N=\sum_{i=1}^{N-1}a_iw_i$. Thus, as $T(\sum_{i=1}^{N-1}a_iv_i)=T(v_N)=w_N$ the same linear transformation works here as well. 

If $v_N$ is linearly independent of $\{v_i\}_{i=1}^{N-1}$ then recall that for two vector spaces $V,W$ there exists a linear transformation between a basis of $V$ and any set of vectors in $W$. Assume without loss of generality that $\{v_i\}_{i=1}^{N-1}$ are linearly independent, as otherwise a vector could be removed from the set without altering its span and the induction hypothesis could readily be used with the reduced set and $v_N$. Thus, the set $\{v_i\}_{i=1}^N$ is linearly independent and the aforementioned lemma can be used for $V'\coloneqq \operatorname{span}(\{v_i\}_{i=1}^N)$ and $W$. Therefore, such a $T$ exists. If $V'\subset V$ then $T$ can be extended to be over $V$ by mapping basis vectors in $V/V'$ to 0, thereby completing the proof. \qed

\textbf{\Cref{thm:similarities_enable_lin}} (Same similarities enable linear mappings.) \textit{Let $V$ and $W$ have corresponding sets of concept representations, $\{v_i\}_{i=1}^N$ and $\{w_i\}_{i=1}^N$. If $\forall i,j:v_i^Tv_j=w_i^Tw_j$ then there exists a linear map $T:V\to W$ such that $\forall i:T(v_i)=w_i$.}

\textit{Proof.} We wish to show that $\forall a_i\in\mathbb{R}: \sum_i a_i v_i=0\Rightarrow\sum_i a_i w_i=0$, thus allowing us to use \Cref{thm:gen_analogs}. Let there be some set of $a_i$ such that $\sum_i a_i v_i=0$. Multiplying by $\sum_i a_i v_i$ and using that $v_i^Tv_j=w_i^Tw_j$, we have that $\Vert\sum_i a_i v_i\Vert^2=\sum_{i,j} a_ia_j v_i^Tv_j=\sum_{i,j} a_ia_j w_i^Tw_j=0$. Note that $\sum_{i,j} a_ia_j w_i^Tw_j=\Vert\sum_i a_i w_i\Vert^2=0$, and as $\sum_i a_i w_i$ is of length zero it must be the zero vector. Thus, $\sum_i a_i w_i=0$ and the proof is complete. \qed

It is easy to see that the resulting transformation $T(v_i)=w_i$ is an isometry over $\operatorname{span}(\{v_i\})$ as $\forall a_i\in\mathbb{R}:\Vert \sum_{i=1}^N a_i v_i\Vert^2=\sum_{ij=1}^Na_ia_jv_i^Tv_j$ so due to the sets of concepts having the same similarity structure this equals $\sum_{ij=1}^Na_i a_j w_i^Tw_j=\Vert \sum_{i=1}^N a_i w_i\Vert^2=\Vert \sum_{i=1}^n a_i T(v_i)\Vert^2=\Vert T(\sum_{i=1}^N a_i v_i)\Vert^2$, so distances are preserved. The corollary, where different similarity structures rule out there being an isometry between the representations, is straightforward to prove:

\textbf{\Cref{thm:diff_simstruct_no_iso}}(Different similarity structures rule out isometric mappings)
Assume $V,W$ have corresponding sets of concept representations but different similarity structures. Thus, there cannot exist an isometric linear map $T:V\to W$ such that $\forall i:T(v_{X_i})=w_{X_i}$.
\begin{proof}
    Assume for the sake of contradiction that there exists such an isometric linear map. As $T$ is an isometry and thus respects inner products, $\forall i,j:v_{X_i}^Tv_{X_j}=T(v_{X_i})^TT(v_{X_j})=w_{X_i}^Tw_{X_j}$. However, as the two representations have a different similarity structure, $\exists i,j:v_{X_i}^Tv_{X_j}\not=w_{X_i}^Tw_{X_j}$, thereby yielding a contradiction and completing the proof.
\end{proof}

\section{More Examples of Data With Different Similarity Structures}
\label{app:more_diff_semsim_examples}
See \Cref{fig:semsim_app}. Note that a small subset of examples seem to be due to errors in the semantic similarity models, e.g. the two images in the bottom left being classified as similar.

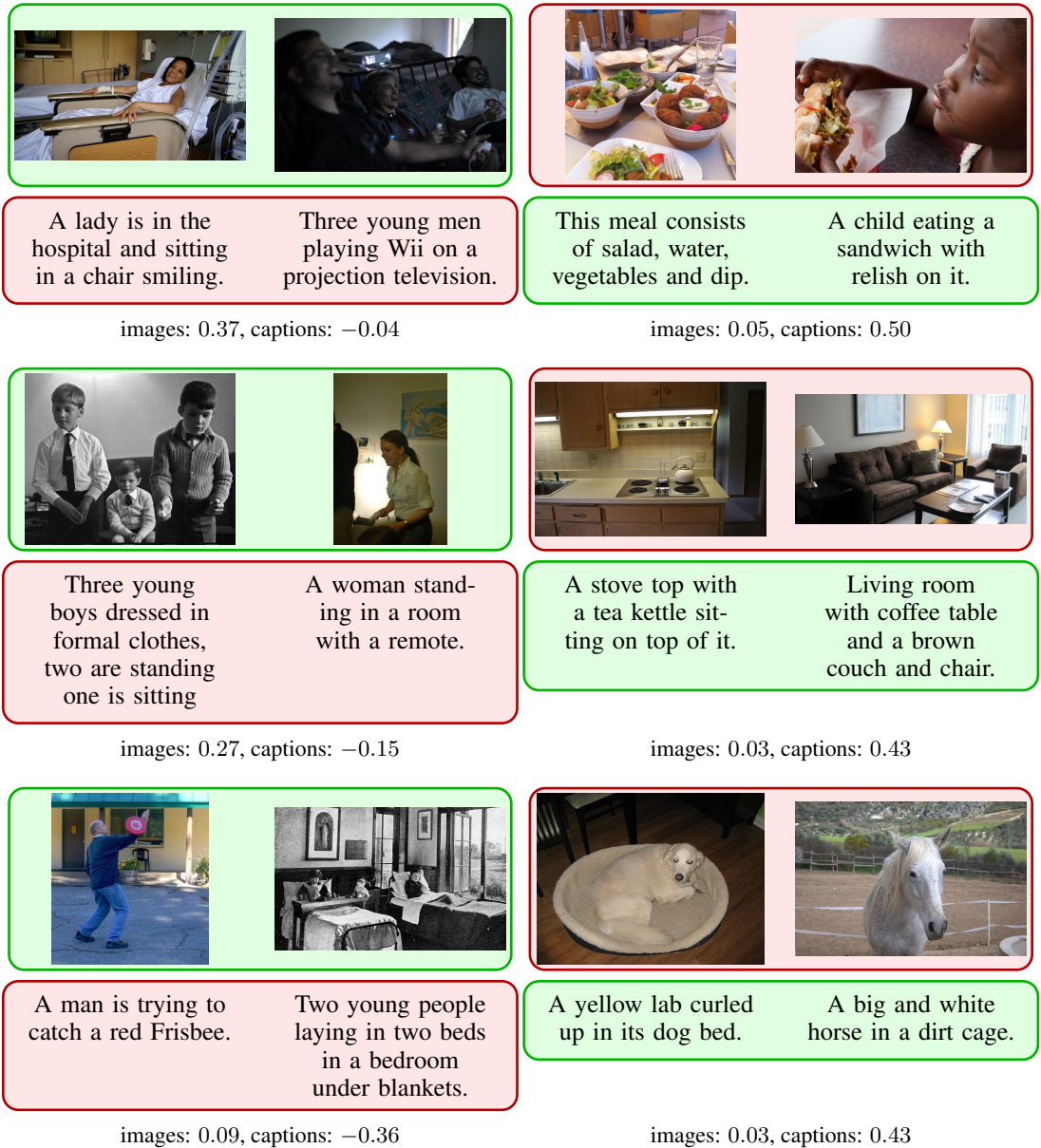
\begin{figure}[t]
\centering

\begin{tikzpicture}
  \node[inner sep=0pt, minimum width=0.23\textwidth, minimum height=0.17\textwidth] (aimg) at (0,0)
    {\includegraphics[width=0.23\textwidth,height=0.17\textwidth,keepaspectratio]{figs/semsim_appendix/img_more_similar_0_5884.jpg}};
  \node[inner sep=0pt, minimum width=0.23\textwidth, minimum height=0.17\textwidth] (bimg) [right=4mm of aimg]
    {\includegraphics[width=0.23\textwidth,height=0.17\textwidth,keepaspectratio]{figs/semsim_appendix/img_more_similar_0_815.jpg}};
  \node[inner sep=0pt, minimum width=0.23\textwidth, minimum height=0.17\textwidth] (cimg) [right=4mm of bimg]
    {\includegraphics[width=0.23\textwidth,height=0.17\textwidth,keepaspectratio]{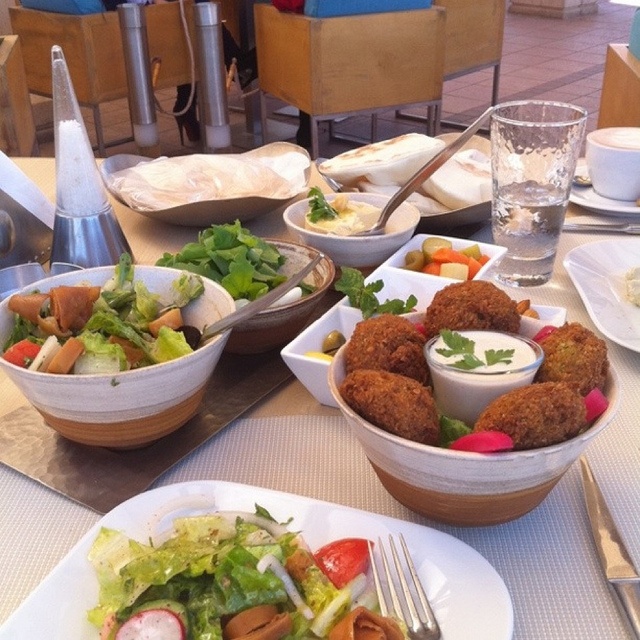}};
  \node[inner sep=0pt, minimum width=0.23\textwidth, minimum height=0.17\textwidth] (dimg) [right=4mm of cimg]
    {\includegraphics[width=0.23\textwidth,height=0.17\textwidth,keepaspectratio]{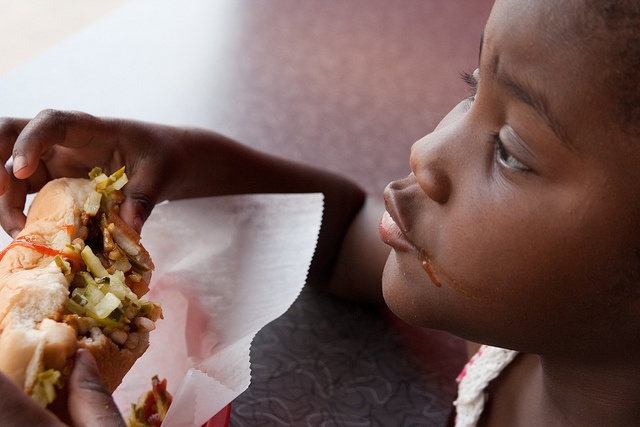}};
  \coordinate (capline) at ($(aimg.south) + (0,-3mm)$);
  \node[align=center, text width=0.225\textwidth, anchor=north] (acap) at ($(aimg.south |- capline)$)
    {A lady is in the hospital and sitting in a chair smiling.};
  \node[align=center, text width=0.225\textwidth, anchor=north] (bcap) at ($(bimg.south |- capline)$)
    {Three young men playing Wii on a projection television.};
  \node[align=center, text width=0.225\textwidth, anchor=north] (ccap) at ($(cimg.south |- capline)$)
    {This meal consists of salad, water, vegetables and dip.};
  \node[align=center, text width=0.225\textwidth, anchor=north] (dcap) at ($(dimg.south |- capline)$)
    {A child eating a sandwich with relish on it.};
  \begin{scope}[on background layer]
    \node[draw=green!70!black, fill=green!12, rounded corners=6pt, line width=1pt, inner sep=2pt, fit=(aimg) (bimg)] (aimgbimgbox) {};
    \node[draw=red!70!black, fill=red!10, rounded corners=6pt, line width=1pt, inner sep=2pt, fit=(acap) (bcap)] (acapbcapbox) {};
    \node[draw=red!70!black, fill=red!10, rounded corners=6pt, line width=1pt, inner sep=2pt, fit=(cimg) (dimg)] (cimgdimgbox) {};
    \node[draw=green!70!black, fill=green!12, rounded corners=6pt, line width=1pt, inner sep=2pt, fit=(ccap) (dcap)] (ccapdcapbox) {};
  \end{scope}
  \coordinate (simline) at ($(current bounding box.south) + (0,-1mm)$);
  \coordinate (acapbcapmid) at ($(acap.south)!0.5!(bcap.south)$);
  \node[anchor=north, font=\footnotesize] at (acapbcapmid |- simline)
    {images: $0.37$, captions: $-0.04$};
  \coordinate (ccapdcapmid) at ($(ccap.south)!0.5!(dcap.south)$);
  \node[anchor=north, font=\footnotesize] at (ccapdcapmid |- simline)
    {images: $0.05$, captions: $0.50$};
\end{tikzpicture}

\vspace{6pt}
\begin{tikzpicture}
  \node[inner sep=0pt, minimum width=0.23\textwidth, minimum height=0.17\textwidth] (aimg) at (0,0)
    {\includegraphics[width=0.23\textwidth,height=0.17\textwidth,keepaspectratio]{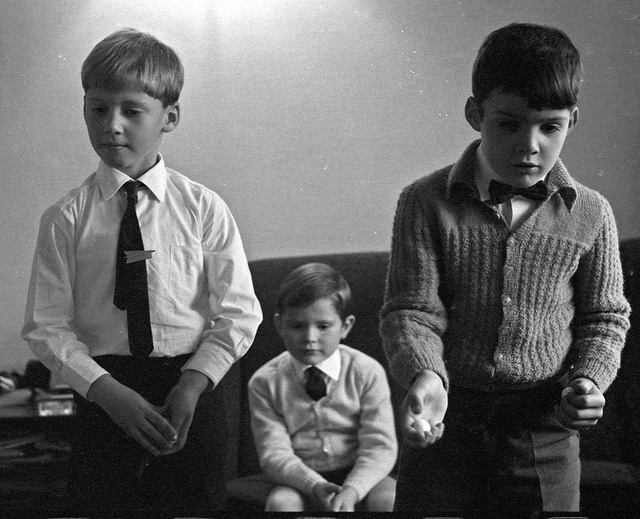}};
  \node[inner sep=0pt, minimum width=0.23\textwidth, minimum height=0.17\textwidth] (bimg) [right=4mm of aimg]
    {\includegraphics[width=0.23\textwidth,height=0.17\textwidth,keepaspectratio]{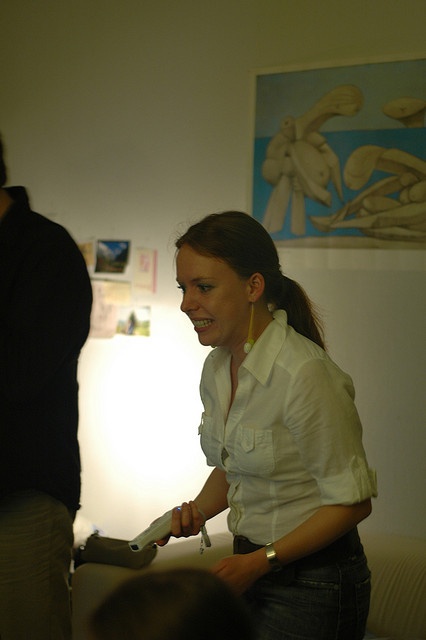}};
  \node[inner sep=0pt, minimum width=0.23\textwidth, minimum height=0.17\textwidth] (cimg) [right=4mm of bimg]
    {\includegraphics[width=0.23\textwidth,height=0.17\textwidth,keepaspectratio]{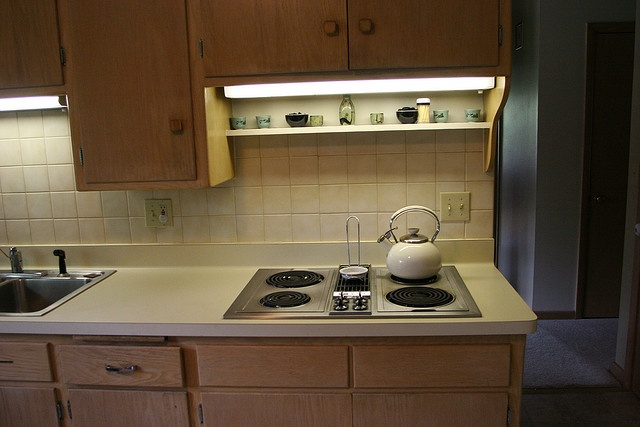}};
  \node[inner sep=0pt, minimum width=0.23\textwidth, minimum height=0.17\textwidth] (dimg) [right=4mm of cimg]
    {\includegraphics[width=0.23\textwidth,height=0.17\textwidth,keepaspectratio]{figs/semsim_appendix/txt_more_similar_4_800.jpg}};
  \coordinate (capline) at ($(aimg.south) + (0,-3mm)$);
  \node[align=center, text width=0.225\textwidth, anchor=north] (acap) at ($(aimg.south |- capline)$)
    {Three young boys dressed in formal clothes, two are standing one is sitting};
  \node[align=center, text width=0.225\textwidth, anchor=north] (bcap) at ($(bimg.south |- capline)$)
    {A woman standing in a room with a remote.};
  \node[align=center, text width=0.225\textwidth, anchor=north] (ccap) at ($(cimg.south |- capline)$)
    {A stove top with a tea kettle sitting on top of it.};
  \node[align=center, text width=0.225\textwidth, anchor=north] (dcap) at ($(dimg.south |- capline)$)
    {Living room with coffee table and a brown couch and chair.};
  \begin{scope}[on background layer]
    \node[draw=green!70!black, fill=green!12, rounded corners=6pt, line width=1pt, inner sep=2pt, fit=(aimg) (bimg)] (aimgbimgbox) {};
    \node[draw=red!70!black, fill=red!10, rounded corners=6pt, line width=1pt, inner sep=2pt, fit=(acap) (bcap)] (acapbcapbox) {};
    \node[draw=red!70!black, fill=red!10, rounded corners=6pt, line width=1pt, inner sep=2pt, fit=(cimg) (dimg)] (cimgdimgbox) {};
    \node[draw=green!70!black, fill=green!12, rounded corners=6pt, line width=1pt, inner sep=2pt, fit=(ccap) (dcap)] (ccapdcapbox) {};
  \end{scope}
  \coordinate (simline) at ($(current bounding box.south) + (0,-1mm)$);
  \coordinate (acapbcapmid) at ($(acap.south)!0.5!(bcap.south)$);
  \node[anchor=north, font=\footnotesize] at (acapbcapmid |- simline)
    {images: $0.27$, captions: $-0.15$};
  \coordinate (ccapdcapmid) at ($(ccap.south)!0.5!(dcap.south)$);
  \node[anchor=north, font=\footnotesize] at (ccapdcapmid |- simline)
    {images: $0.03$, captions: $0.43$};
\end{tikzpicture}

\vspace{6pt}
\begin{tikzpicture}
  \node[inner sep=0pt, minimum width=0.23\textwidth, minimum height=0.17\textwidth] (aimg) at (0,0)
    {\includegraphics[width=0.23\textwidth,height=0.17\textwidth,keepaspectratio]{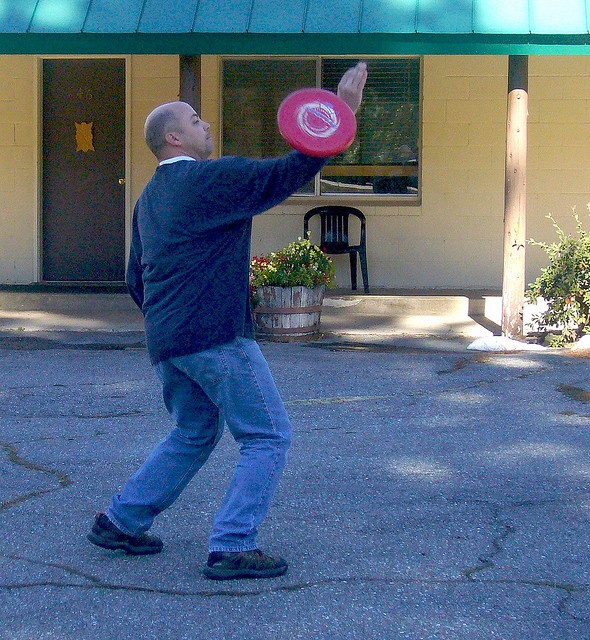}};
  \node[inner sep=0pt, minimum width=0.23\textwidth, minimum height=0.17\textwidth] (bimg) [right=4mm of aimg]
    {\includegraphics[width=0.23\textwidth,height=0.17\textwidth,keepaspectratio]{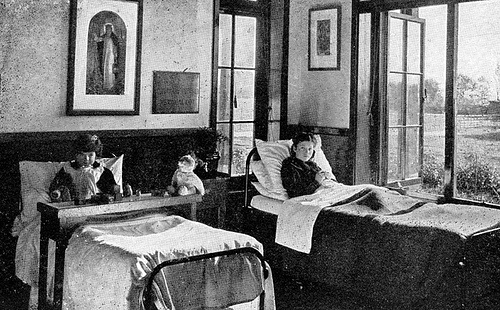}};
  \node[inner sep=0pt, minimum width=0.23\textwidth, minimum height=0.17\textwidth] (cimg) [right=4mm of bimg]
    {\includegraphics[width=0.23\textwidth,height=0.17\textwidth,keepaspectratio]{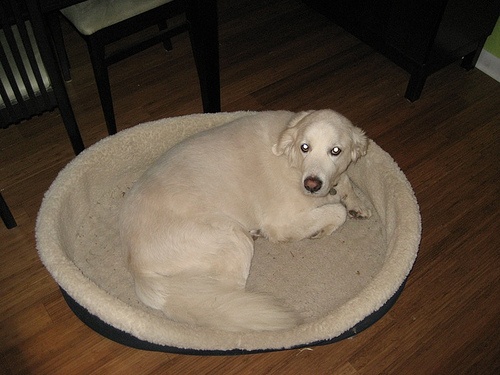}};
  \node[inner sep=0pt, minimum width=0.23\textwidth, minimum height=0.17\textwidth] (dimg) [right=4mm of cimg]
    {\includegraphics[width=0.23\textwidth,height=0.17\textwidth,keepaspectratio]{figs/semsim_appendix/txt_more_similar_5_555.jpg}};
  \coordinate (capline) at ($(aimg.south) + (0,-3mm)$);
  \node[align=center, text width=0.225\textwidth, anchor=north] (acap) at ($(aimg.south |- capline)$)
    {A man is trying to catch a red Frisbee.};
  \node[align=center, text width=0.225\textwidth, anchor=north] (bcap) at ($(bimg.south |- capline)$)
    {Two young people laying in two beds in a bedroom under blankets.};
  \node[align=center, text width=0.225\textwidth, anchor=north] (ccap) at ($(cimg.south |- capline)$)
    {A yellow lab curled up in its dog bed.};
  \node[align=center, text width=0.225\textwidth, anchor=north] (dcap) at ($(dimg.south |- capline)$)
    {A big and white horse in a dirt cage.};
  \begin{scope}[on background layer]
    \node[draw=green!70!black, fill=green!12, rounded corners=6pt, line width=1pt, inner sep=2pt, fit=(aimg) (bimg)] (aimgbimgbox) {};
    \node[draw=red!70!black, fill=red!10, rounded corners=6pt, line width=1pt, inner sep=2pt, fit=(acap) (bcap)] (acapbcapbox) {};
    \node[draw=red!70!black, fill=red!10, rounded corners=6pt, line width=1pt, inner sep=2pt, fit=(cimg) (dimg)] (cimgdimgbox) {};
    \node[draw=green!70!black, fill=green!12, rounded corners=6pt, line width=1pt, inner sep=2pt, fit=(ccap) (dcap)] (ccapdcapbox) {};
  \end{scope}
  \coordinate (simline) at ($(current bounding box.south) + (0,-1mm)$);
  \coordinate (acapbcapmid) at ($(acap.south)!0.5!(bcap.south)$);
  \node[anchor=north, font=\footnotesize] at (acapbcapmid |- simline)
    {images: $0.09$, captions: $-0.36$};
  \coordinate (ccapdcapmid) at ($(ccap.south)!0.5!(dcap.south)$);
  \node[anchor=north, font=\footnotesize] at (ccapdcapmid |- simline)
    {images: $0.03$, captions: $0.43$};
\end{tikzpicture}

\caption{Additional examples of cases of similar images with dissimilar captions or vice versa. Beneath each pair is the cosine similarity between the images and the average embedding of all of their possible captions, with a random one of the captions being shown. Pairs are sampled uniformly at random from all pairs whose image and caption cosine similarities differ by more than $0.4$.}
\label{fig:semsim_app}
\end{figure}

\section{Example Images from VisLingA}
\label{app:vislinga}
See Figure \ref{fig:vislinga}.

\begin{figure}
  \centering
  \setlength{\tabcolsep}{1pt}
  \renewcommand{\arraystretch}{0.5}
  \begin{tabular}{cccccccccc}
    \includegraphics[width=0.095\textwidth]{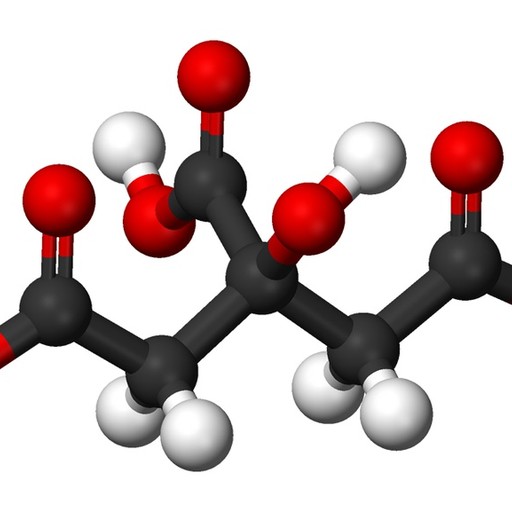} &
    \includegraphics[width=0.095\textwidth]{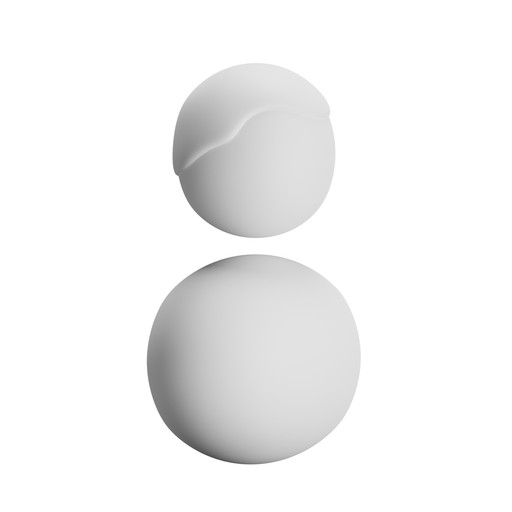} &
    \includegraphics[width=0.095\textwidth]{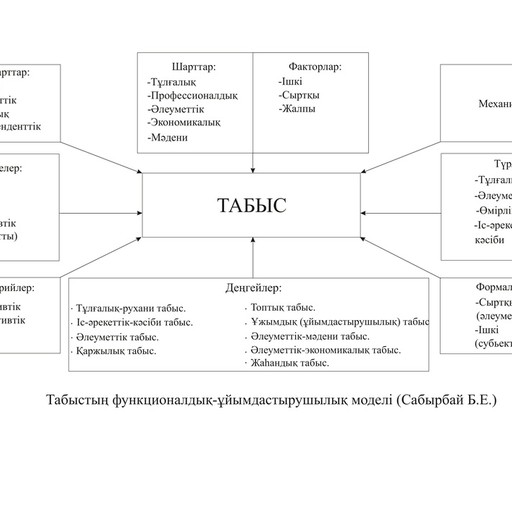} &
    \includegraphics[width=0.095\textwidth]{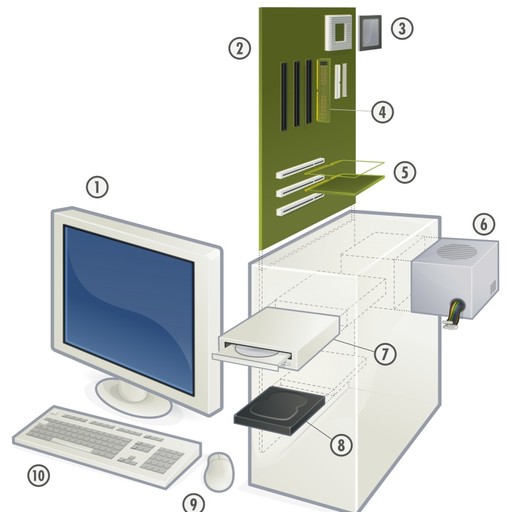} &
    \includegraphics[width=0.095\textwidth]{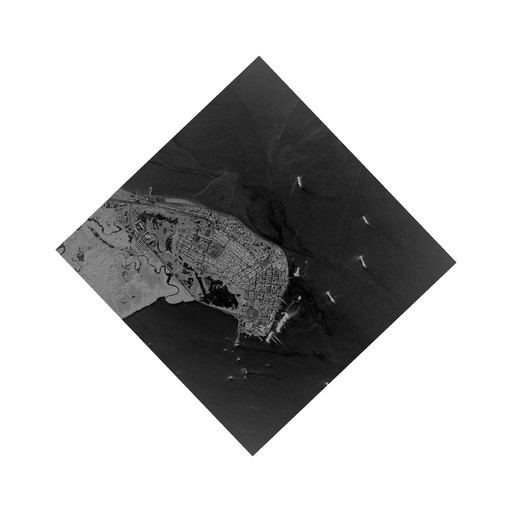} &
    \includegraphics[width=0.095\textwidth]{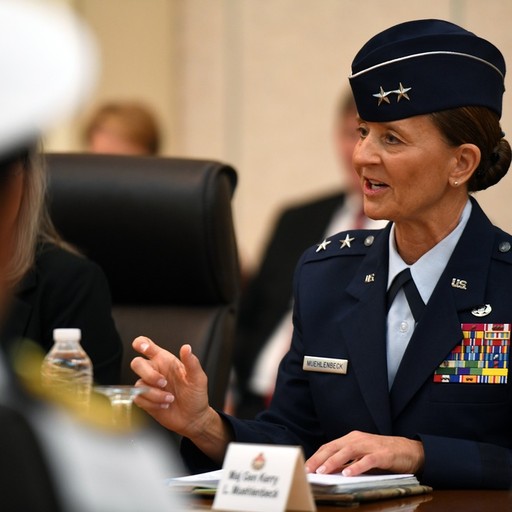} &
    \includegraphics[width=0.095\textwidth]{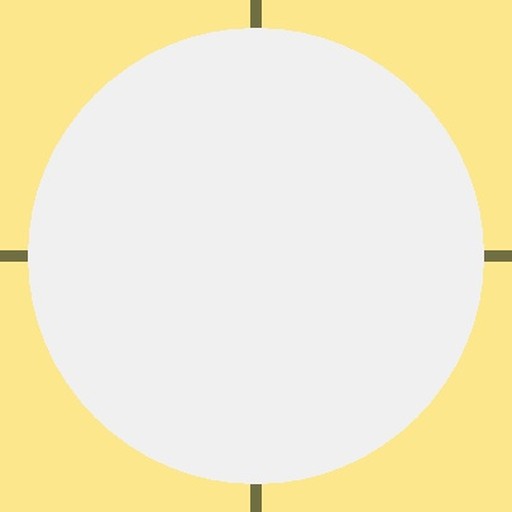} &
    \includegraphics[width=0.095\textwidth]{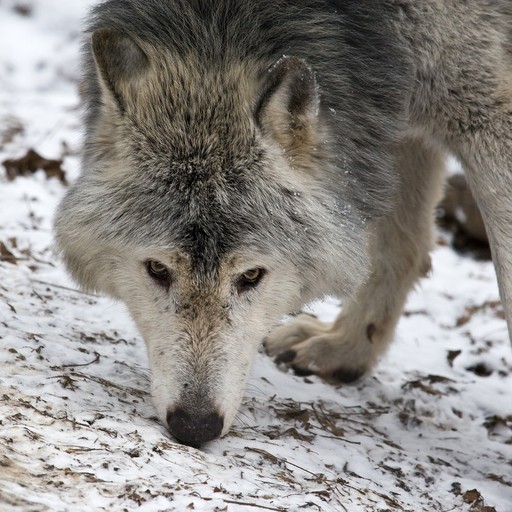} &
    \includegraphics[width=0.095\textwidth]{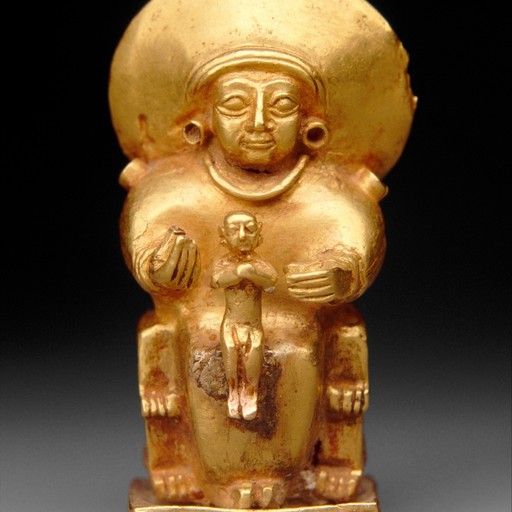} &
    \includegraphics[width=0.095\textwidth]{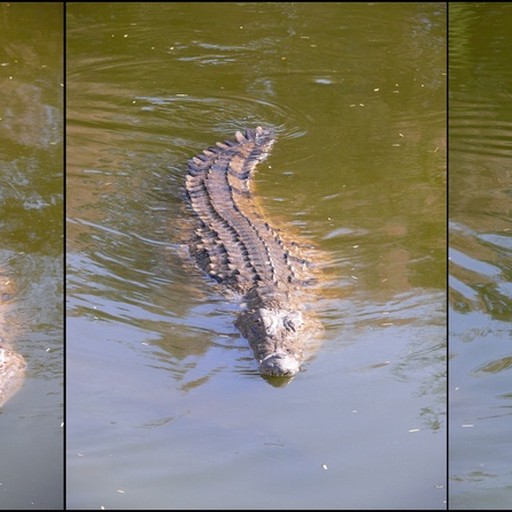} \\
    \includegraphics[width=0.095\textwidth]{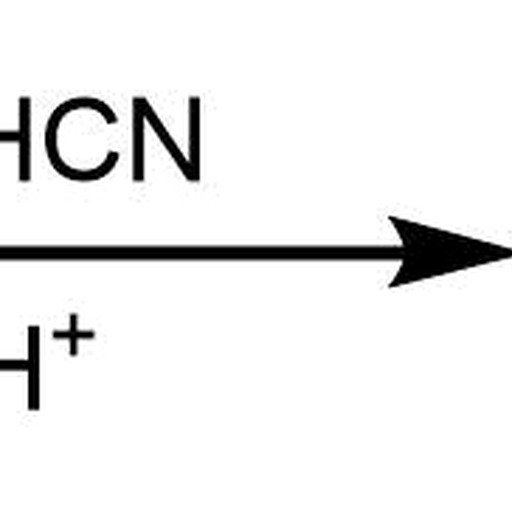} &
    \includegraphics[width=0.095\textwidth]{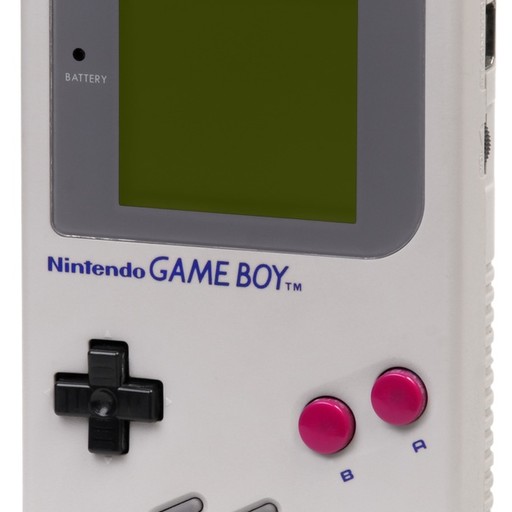} &
    \includegraphics[width=0.095\textwidth]{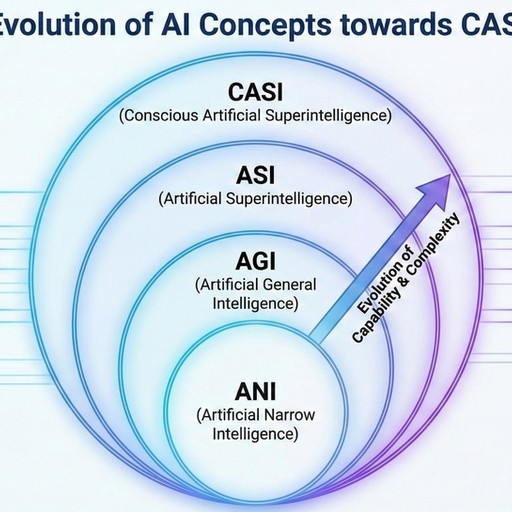} &
    \includegraphics[width=0.095\textwidth]{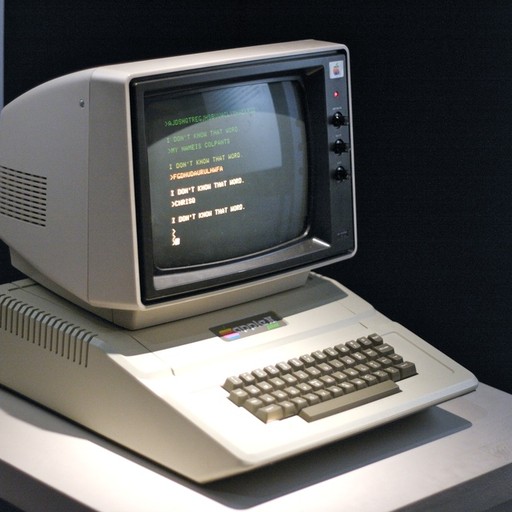} &
    \includegraphics[width=0.095\textwidth]{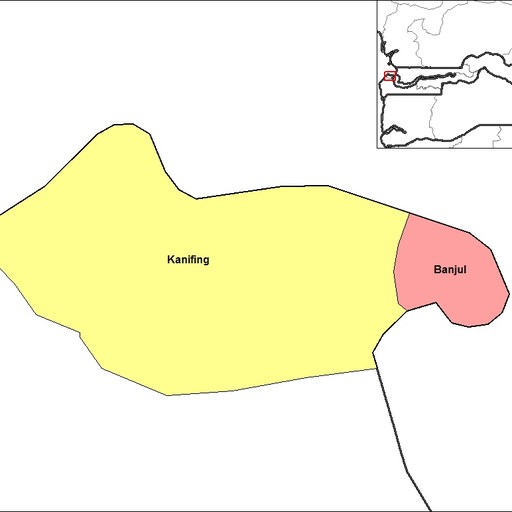} &
    \includegraphics[width=0.095\textwidth]{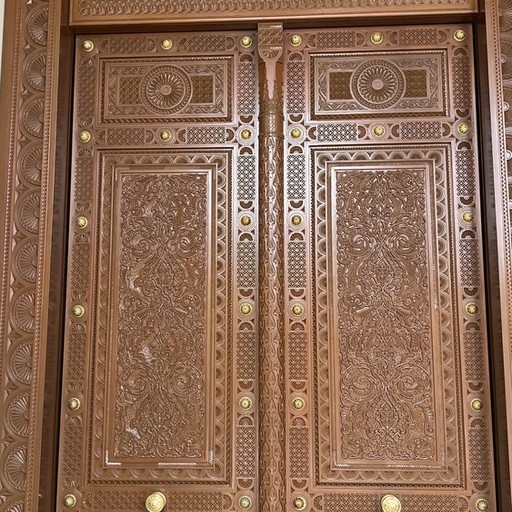} &
    \includegraphics[width=0.095\textwidth]{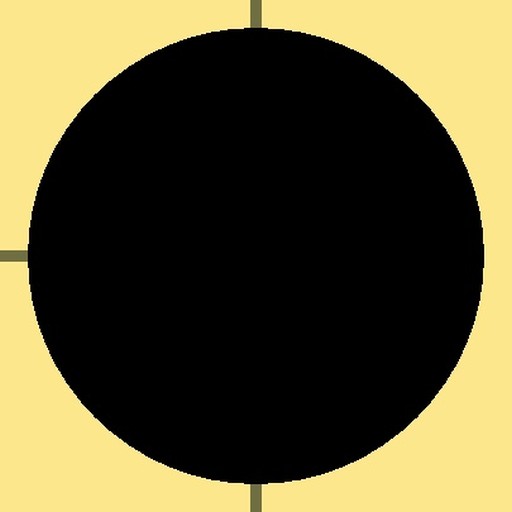} &
    \includegraphics[width=0.095\textwidth]{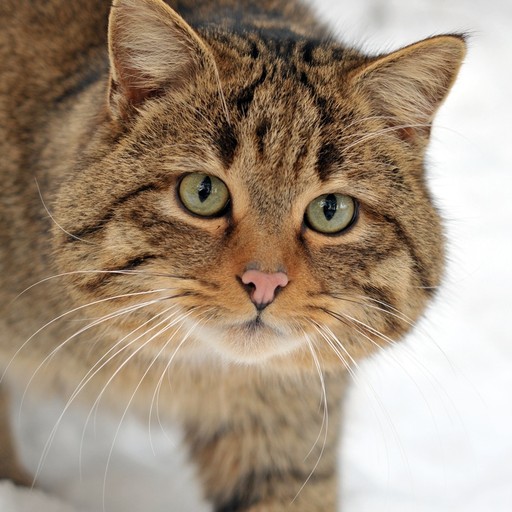} &
    \includegraphics[width=0.095\textwidth]{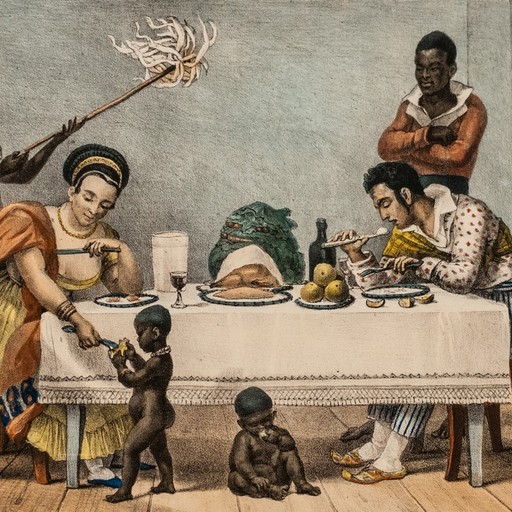} &
    \includegraphics[width=0.095\textwidth]{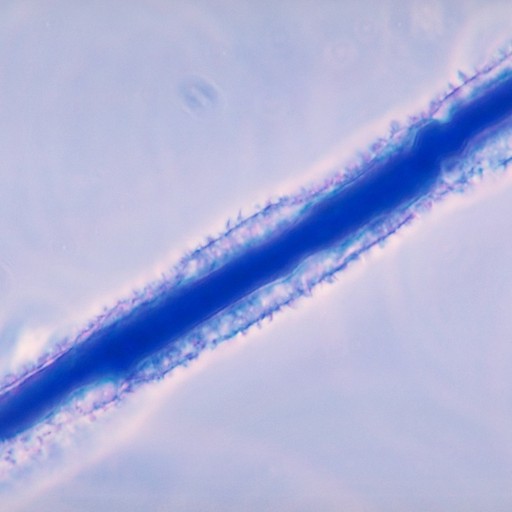} \\
    \includegraphics[width=0.095\textwidth]{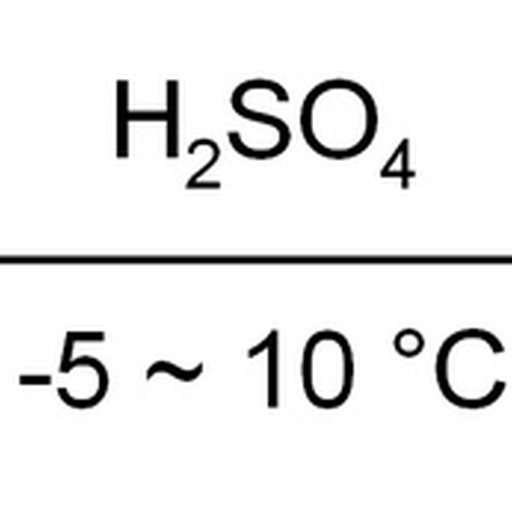} &
    \includegraphics[width=0.095\textwidth]{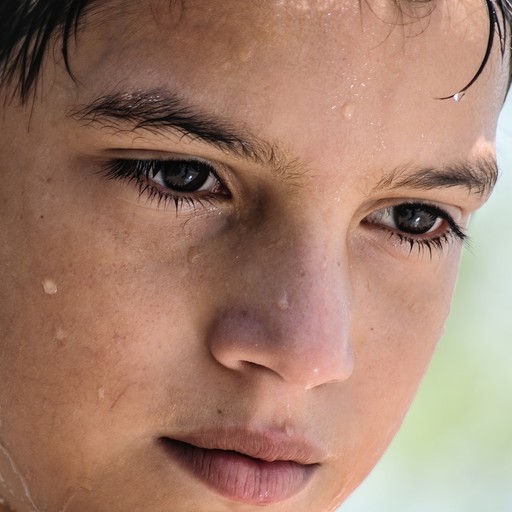} &
    \includegraphics[width=0.095\textwidth]{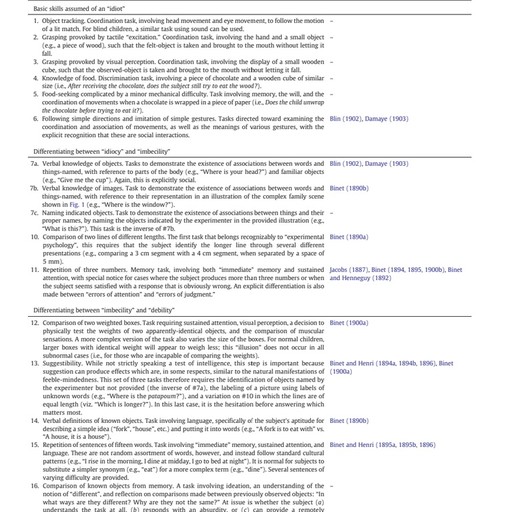} &
    \includegraphics[width=0.095\textwidth]{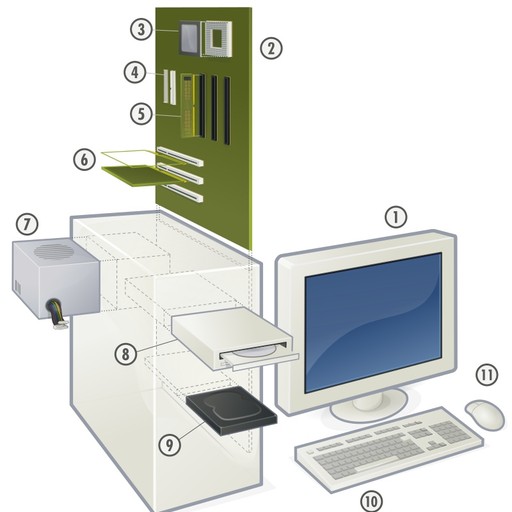} &
    \includegraphics[width=0.095\textwidth]{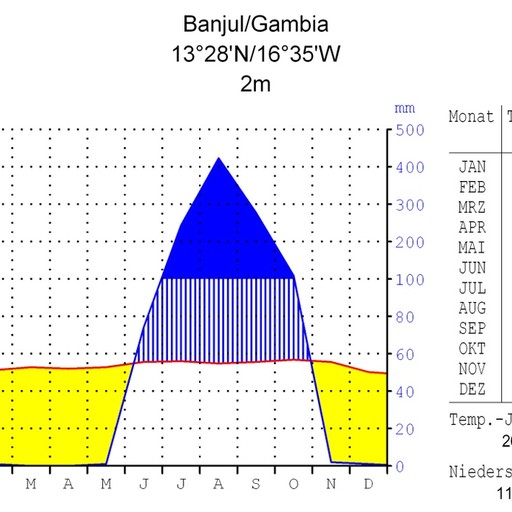} &
    \includegraphics[width=0.095\textwidth]{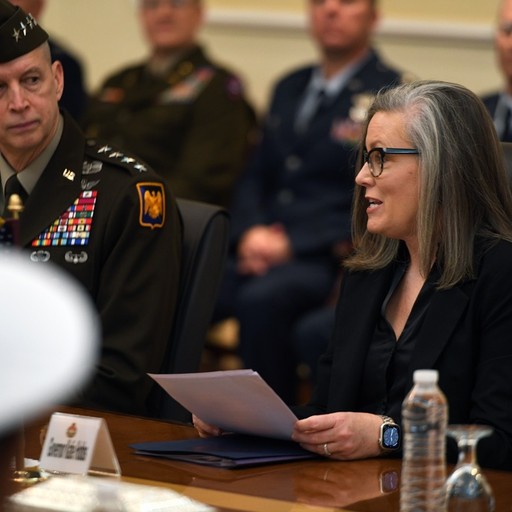} &
    \includegraphics[width=0.095\textwidth]{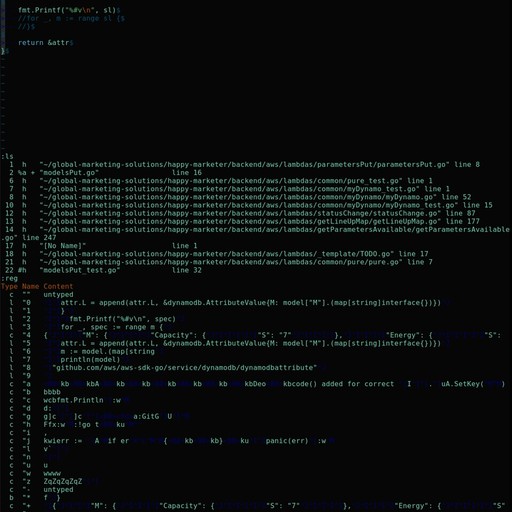} &
    \includegraphics[width=0.095\textwidth]{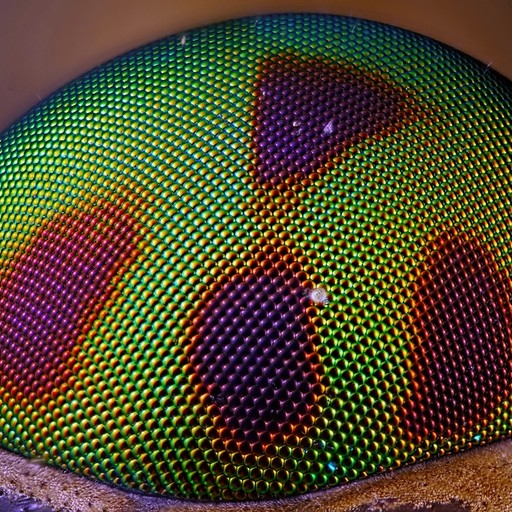} &
    \includegraphics[width=0.095\textwidth]{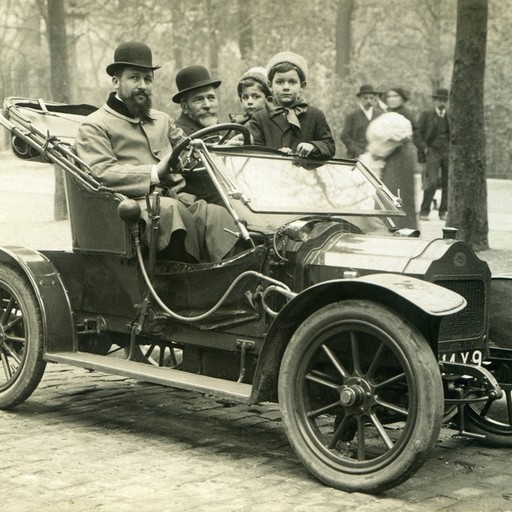} &
    \includegraphics[width=0.095\textwidth]{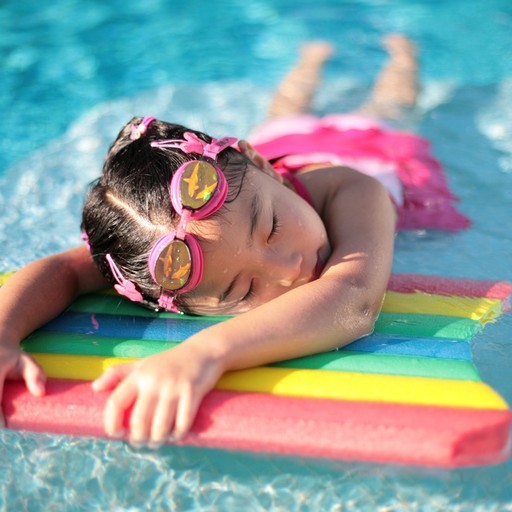} \\
    \includegraphics[width=0.095\textwidth]{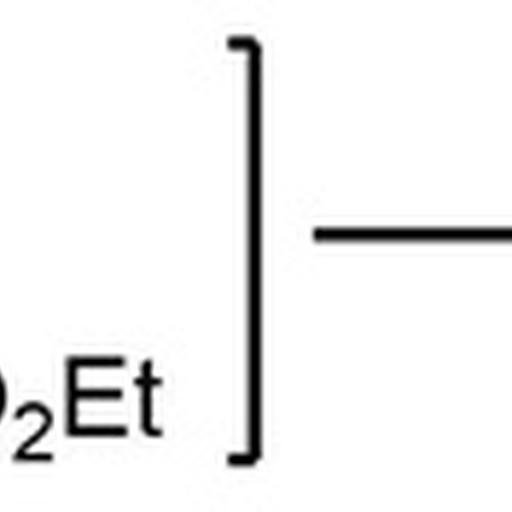} &
    \includegraphics[width=0.095\textwidth]{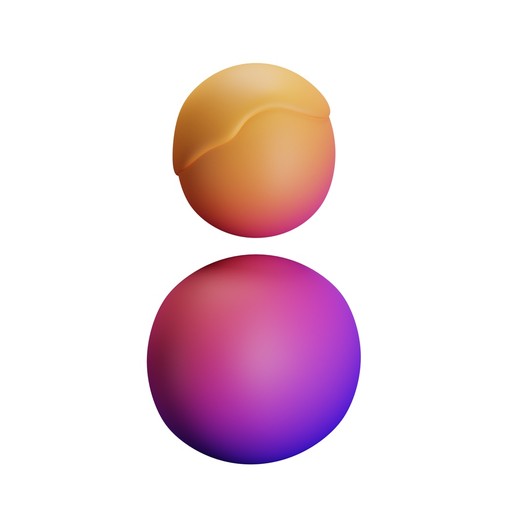} &
    \includegraphics[width=0.095\textwidth]{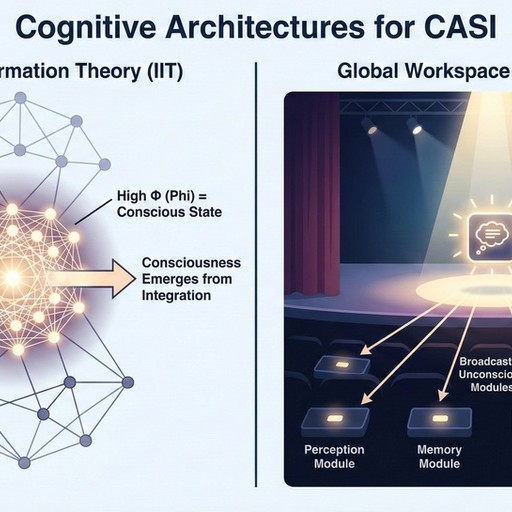} &
    \includegraphics[width=0.095\textwidth]{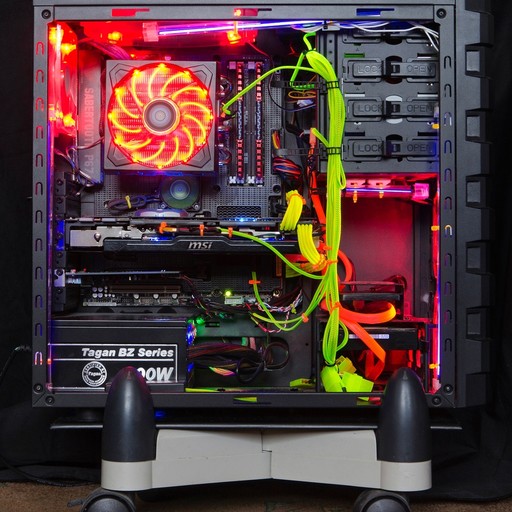} &
    \includegraphics[width=0.095\textwidth]{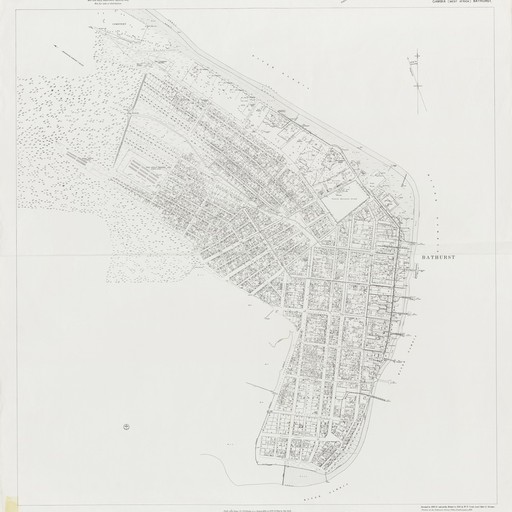} &
    \includegraphics[width=0.095\textwidth]{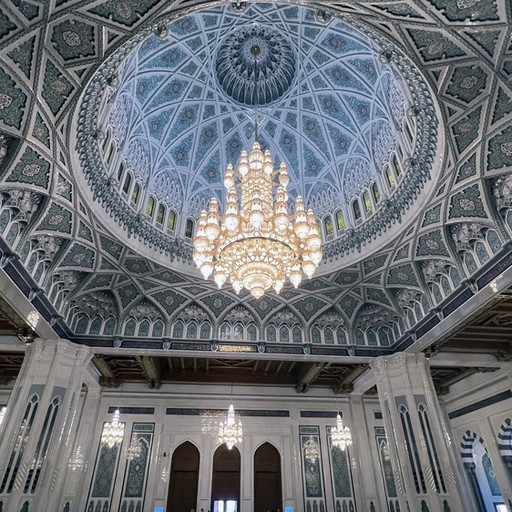} &
    \includegraphics[width=0.095\textwidth]{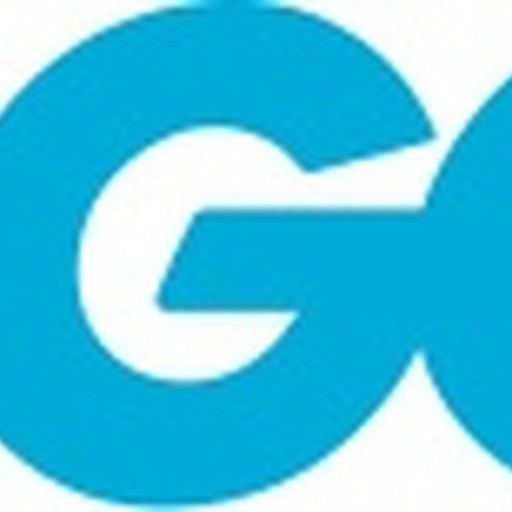} &
    \includegraphics[width=0.095\textwidth]{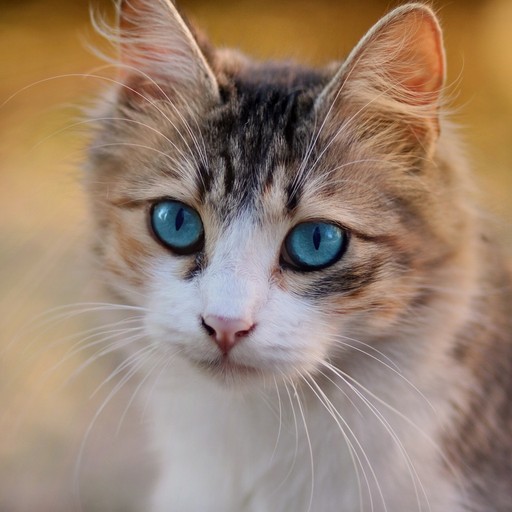} &
    \includegraphics[width=0.095\textwidth]{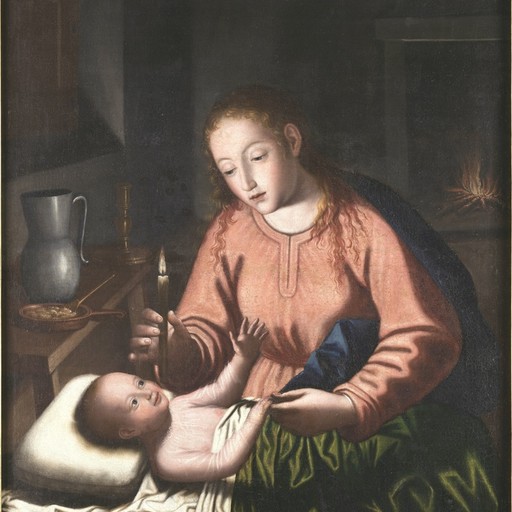} &
    \includegraphics[width=0.095\textwidth]{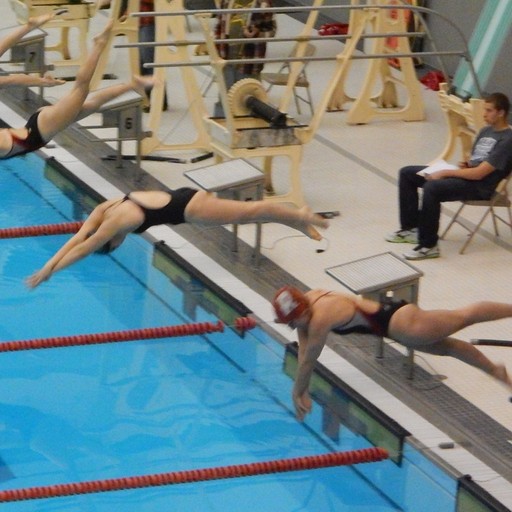} \\
    \includegraphics[width=0.095\textwidth]{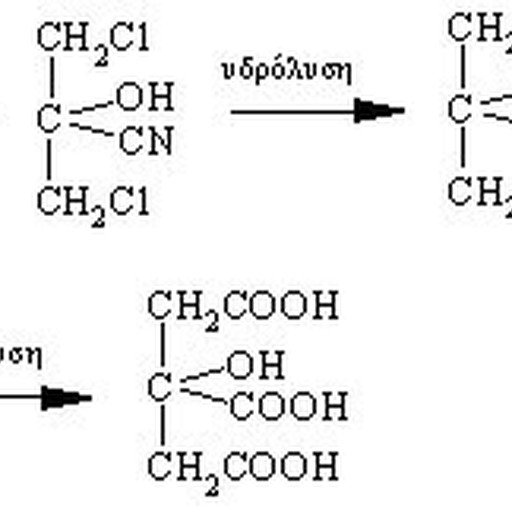} &
    \includegraphics[width=0.095\textwidth]{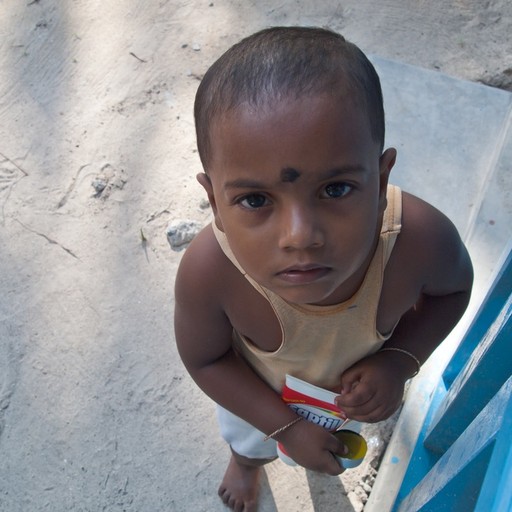} &
    \includegraphics[width=0.095\textwidth]{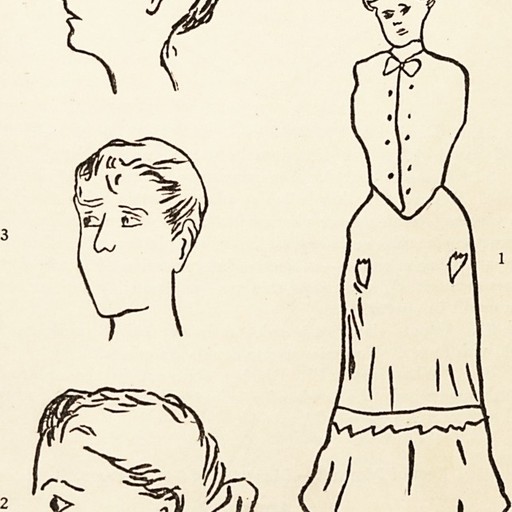} &
    \includegraphics[width=0.095\textwidth]{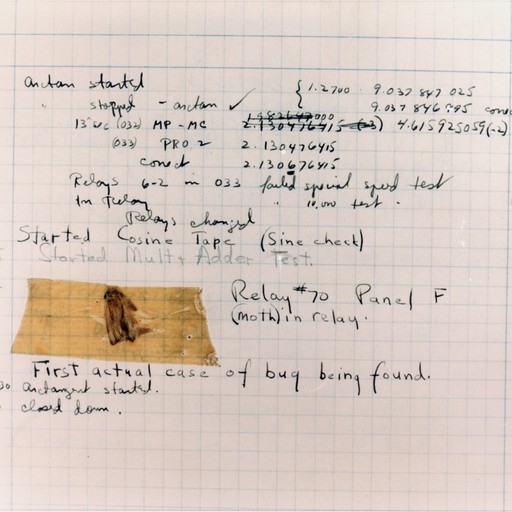} &
    \includegraphics[width=0.095\textwidth]{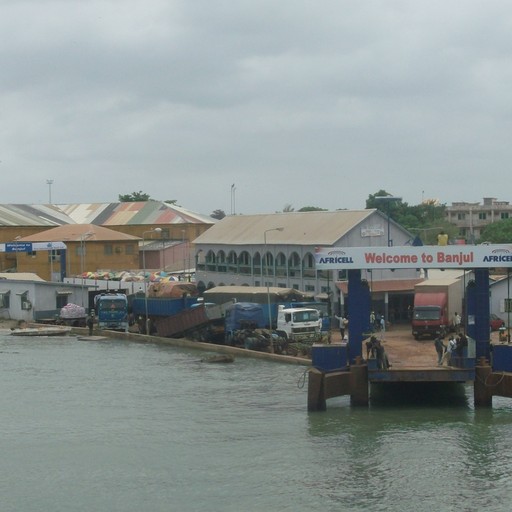} &
    \includegraphics[width=0.095\textwidth]{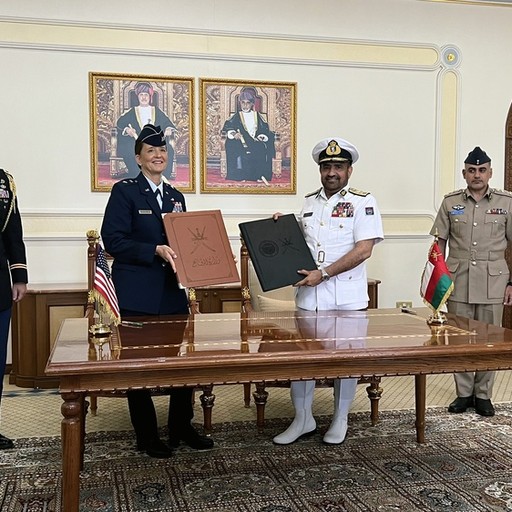} &
    \includegraphics[width=0.095\textwidth]{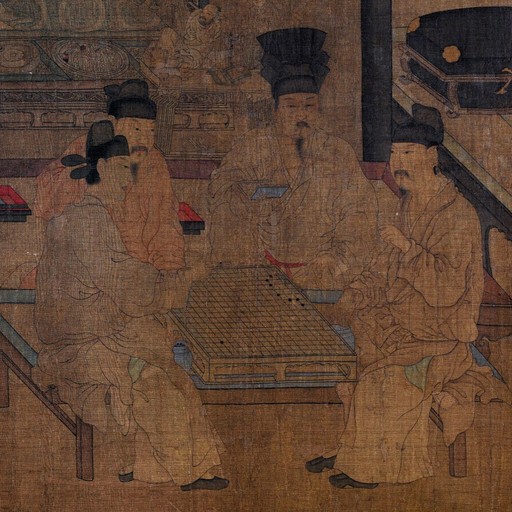} &
    \includegraphics[width=0.095\textwidth]{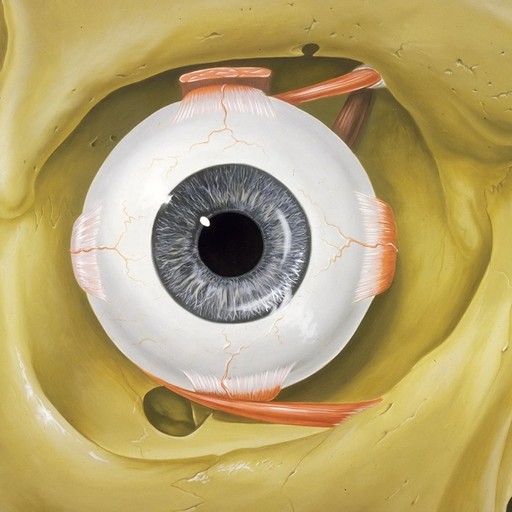} &
    \includegraphics[width=0.095\textwidth]{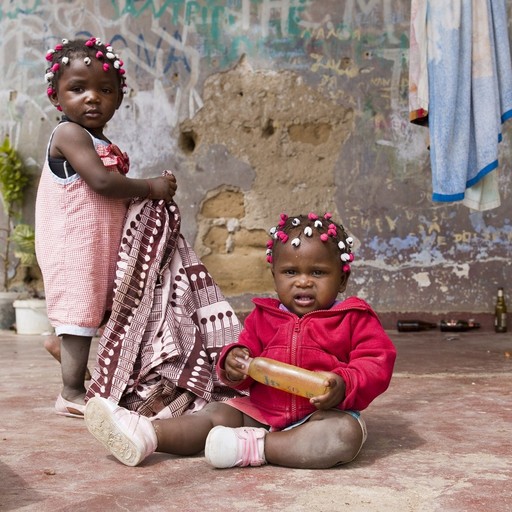} &
    \includegraphics[width=0.095\textwidth]{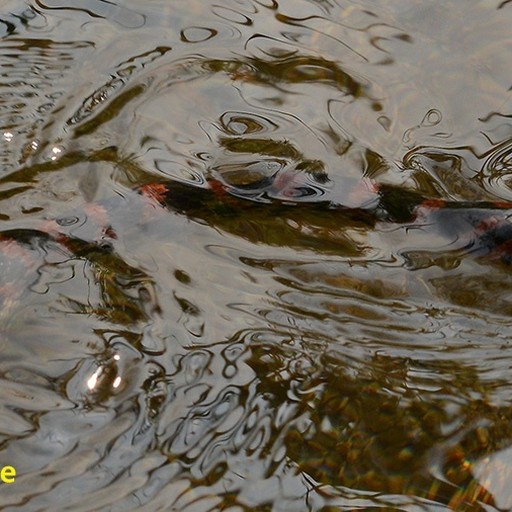} \\
    \scriptsize weaker & \scriptsize boy & \scriptsize smartest & \scriptsize computer & \scriptsize Banjul & \scriptsize Oman & \scriptsize go & \scriptsize eyes & \scriptsize child & \scriptsize swam \\
    \noalign{\vspace{6pt}}
    \includegraphics[width=0.095\textwidth]{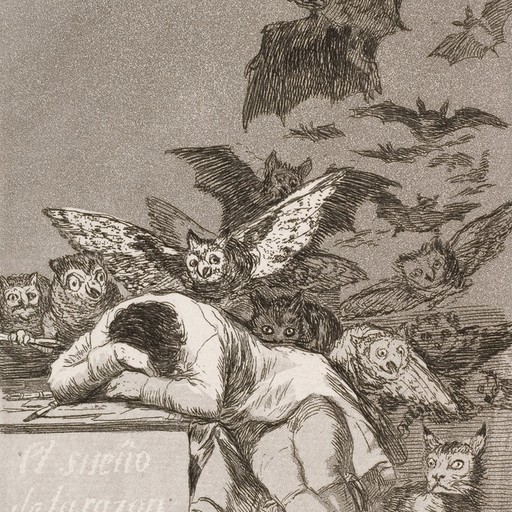} &
    \includegraphics[width=0.095\textwidth]{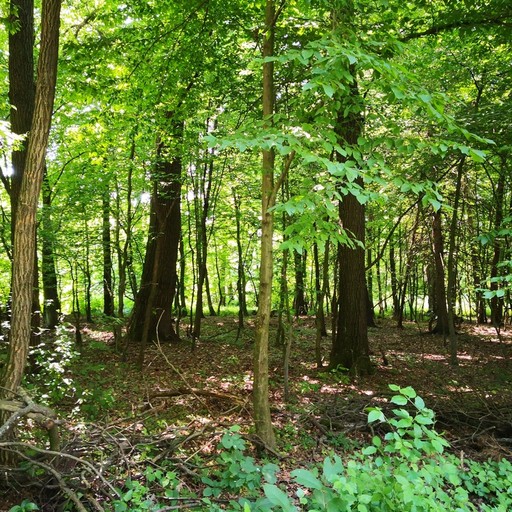} &
    \includegraphics[width=0.095\textwidth]{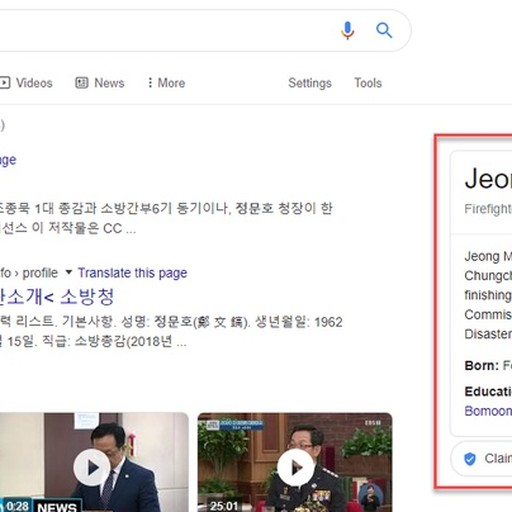} &
    \includegraphics[width=0.095\textwidth]{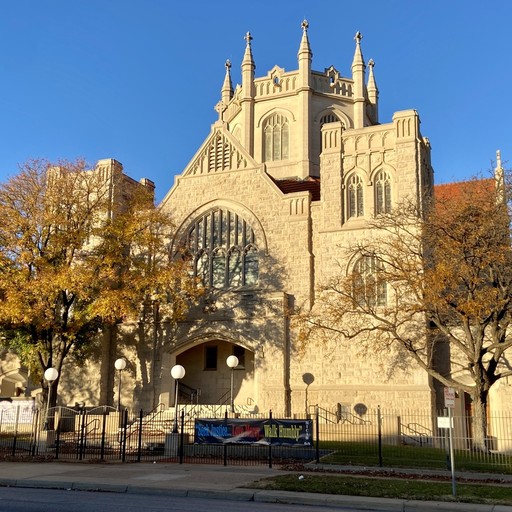} &
    \includegraphics[width=0.095\textwidth]{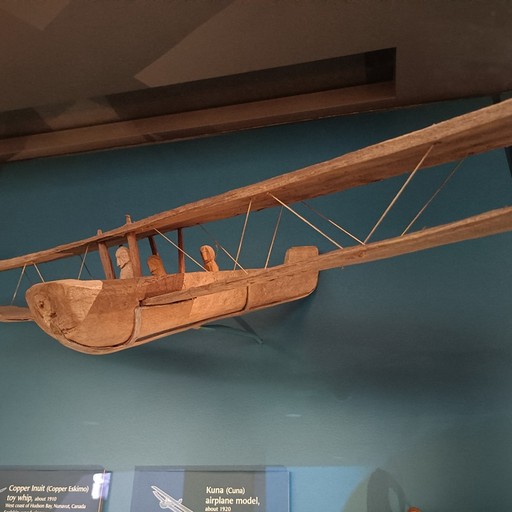} &
    \includegraphics[width=0.095\textwidth]{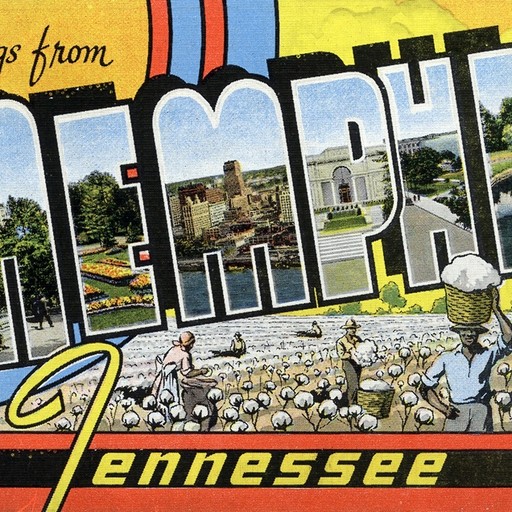} &
    \includegraphics[width=0.095\textwidth]{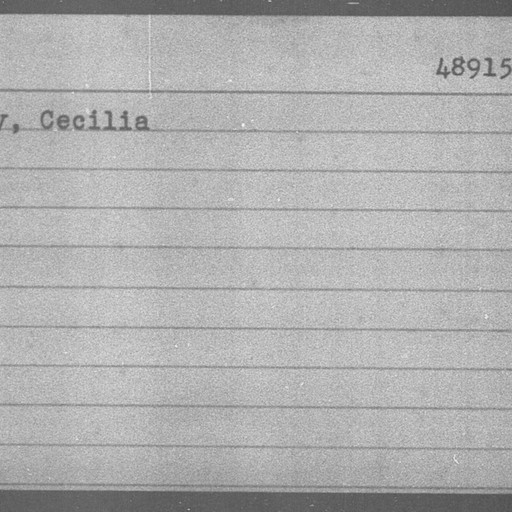} &
    \includegraphics[width=0.095\textwidth]{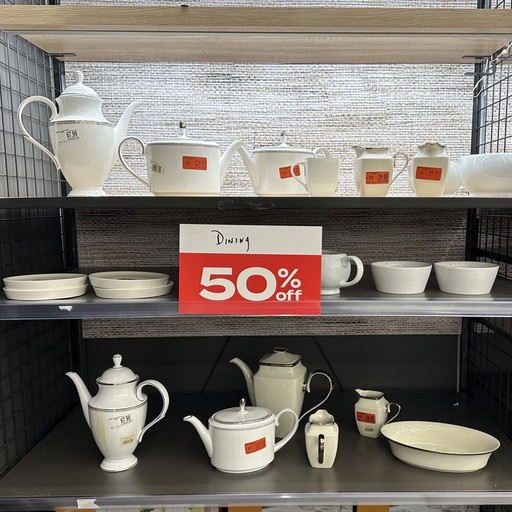} &
    \includegraphics[width=0.095\textwidth]{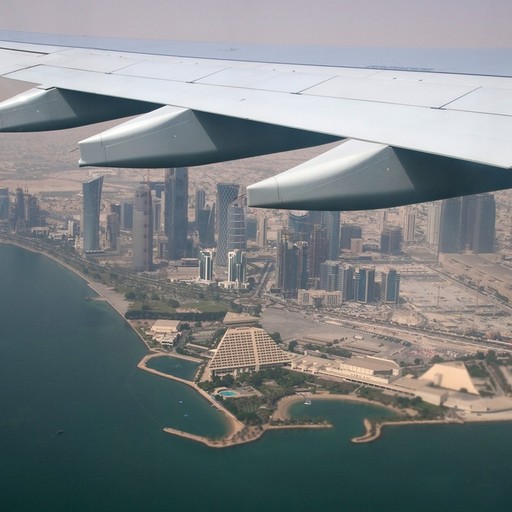} &
    \includegraphics[width=0.095\textwidth]{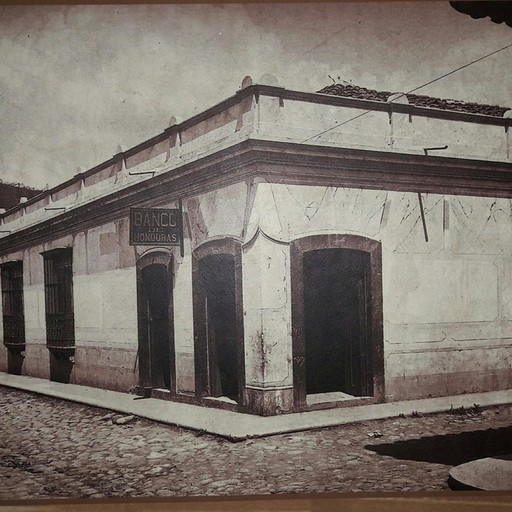} \\
    \includegraphics[width=0.095\textwidth]{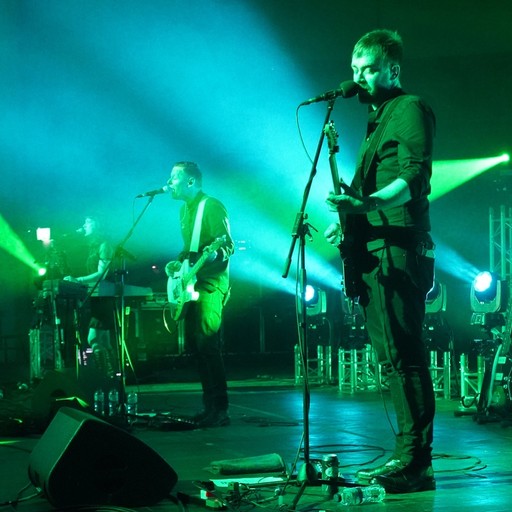} &
    \includegraphics[width=0.095\textwidth]{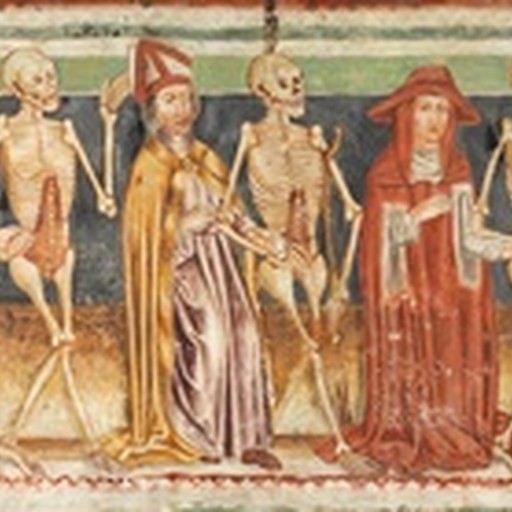} &
    \includegraphics[width=0.095\textwidth]{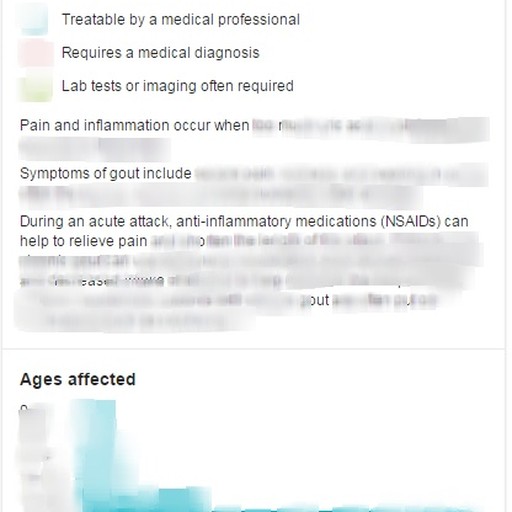} &
    \includegraphics[width=0.095\textwidth]{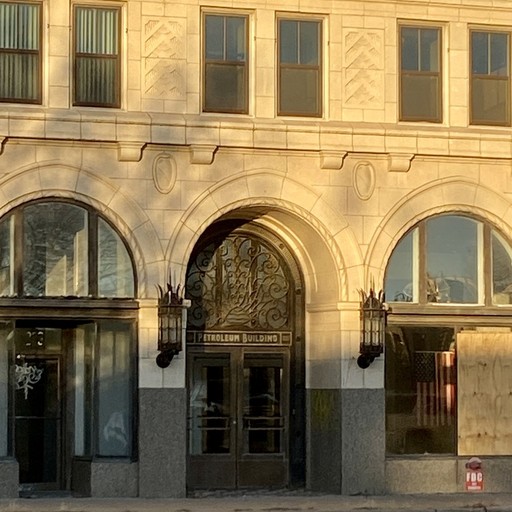} &
    \includegraphics[width=0.095\textwidth]{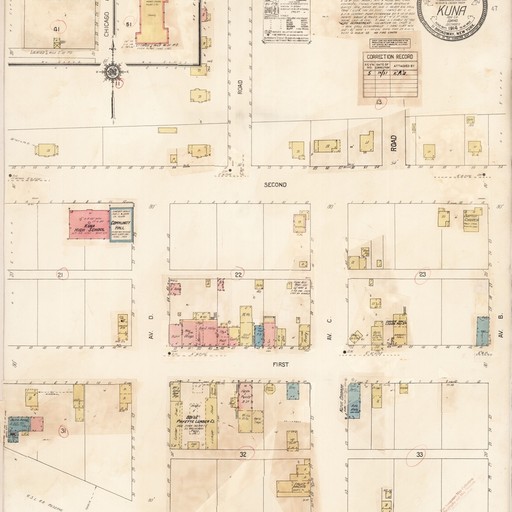} &
    \includegraphics[width=0.095\textwidth]{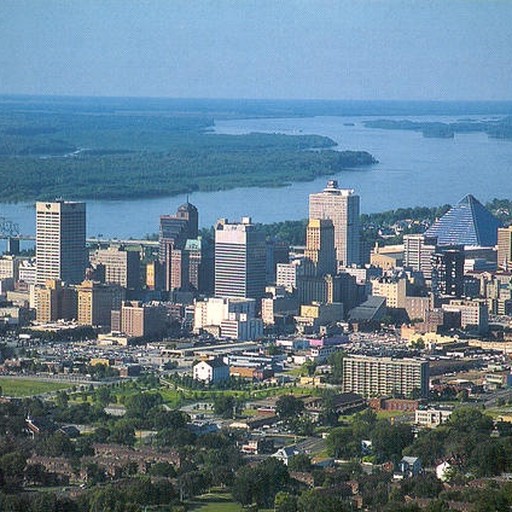} &
    \includegraphics[width=0.095\textwidth]{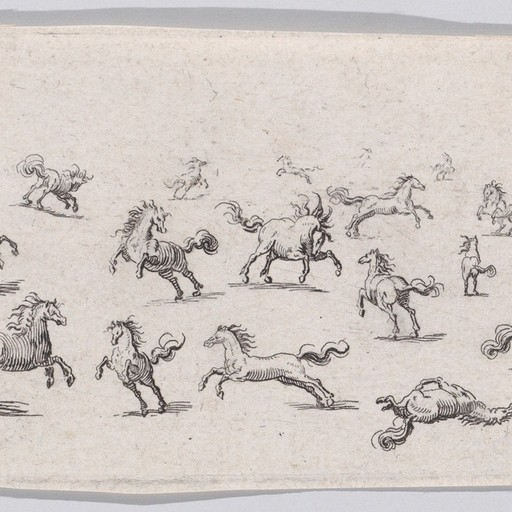} &
    \includegraphics[width=0.095\textwidth]{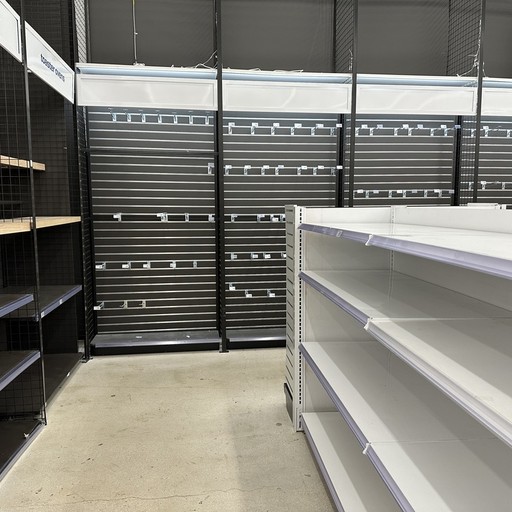} &
    \includegraphics[width=0.095\textwidth]{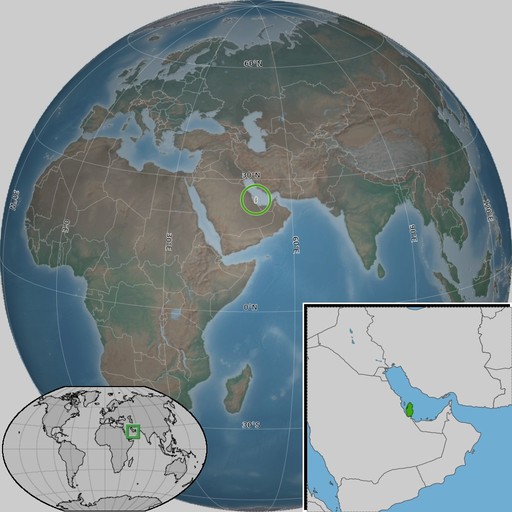} &
    \includegraphics[width=0.095\textwidth]{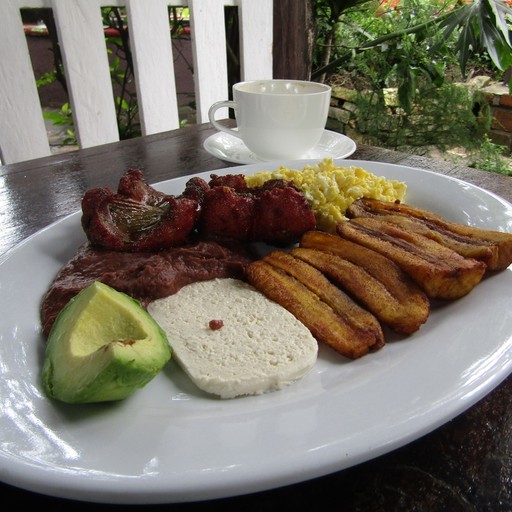} \\
    \includegraphics[width=0.095\textwidth]{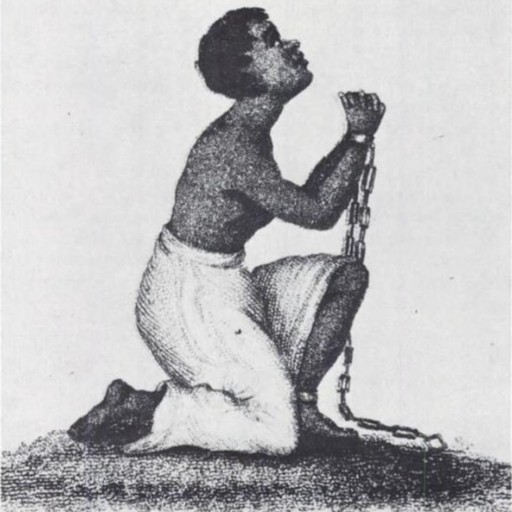} &
    \includegraphics[width=0.095\textwidth]{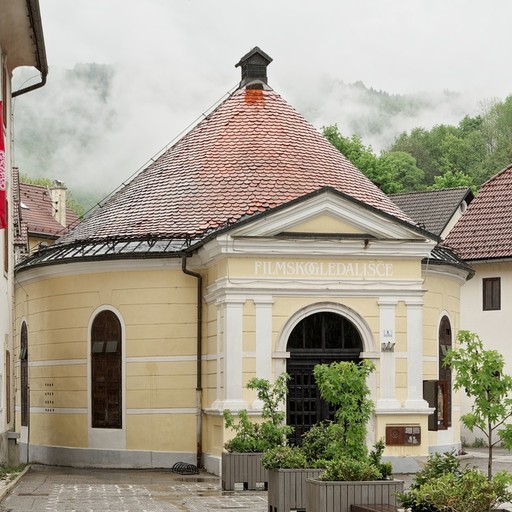} &
    \includegraphics[width=0.095\textwidth]{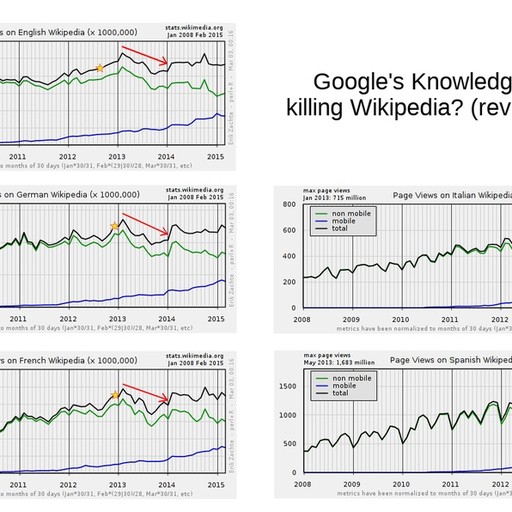} &
    \includegraphics[width=0.095\textwidth]{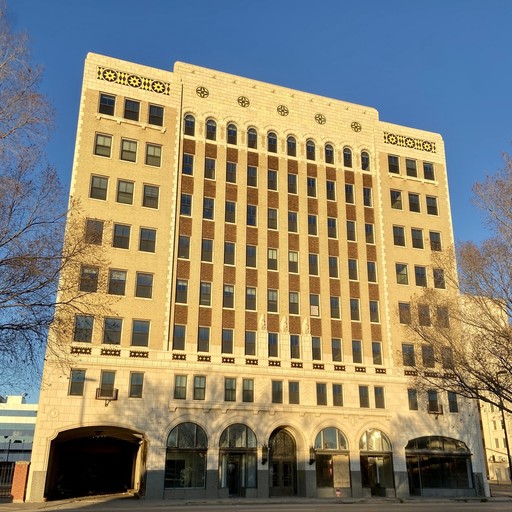} &
    \includegraphics[width=0.095\textwidth]{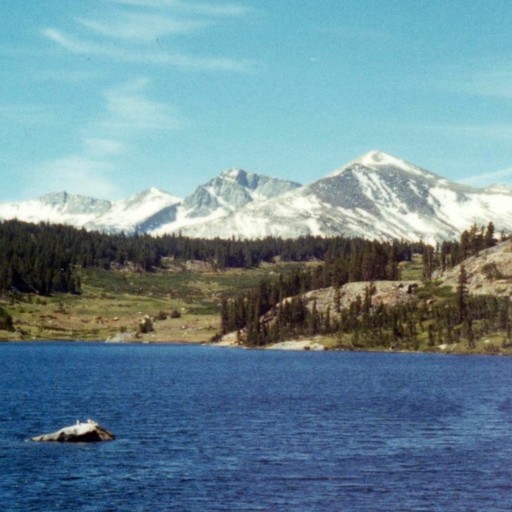} &
    \includegraphics[width=0.095\textwidth]{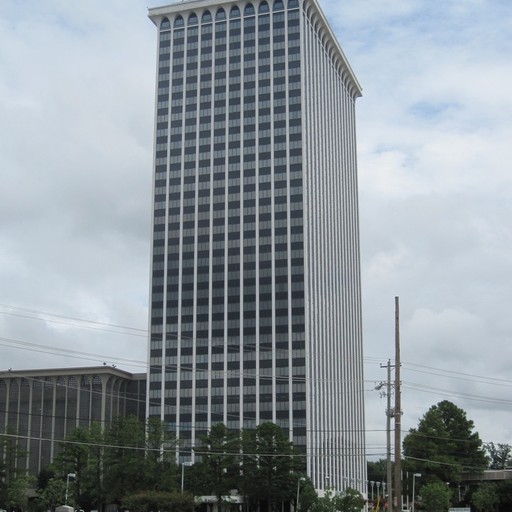} &
    \includegraphics[width=0.095\textwidth]{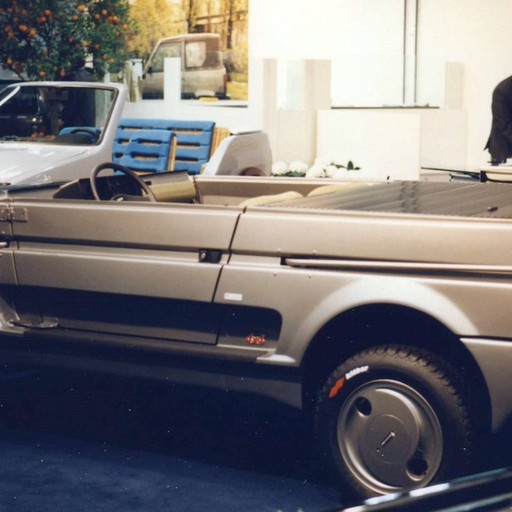} &
    \includegraphics[width=0.095\textwidth]{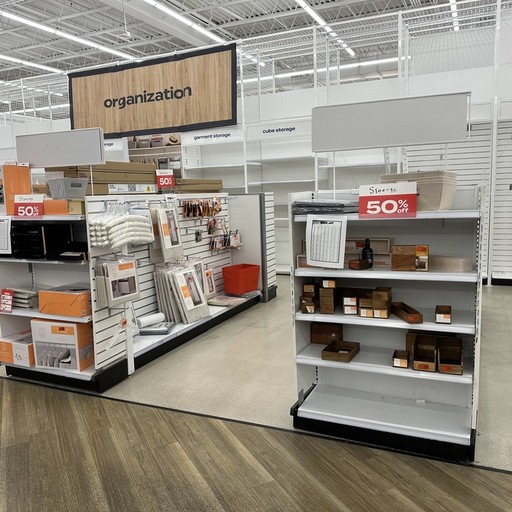} &
    \includegraphics[width=0.095\textwidth]{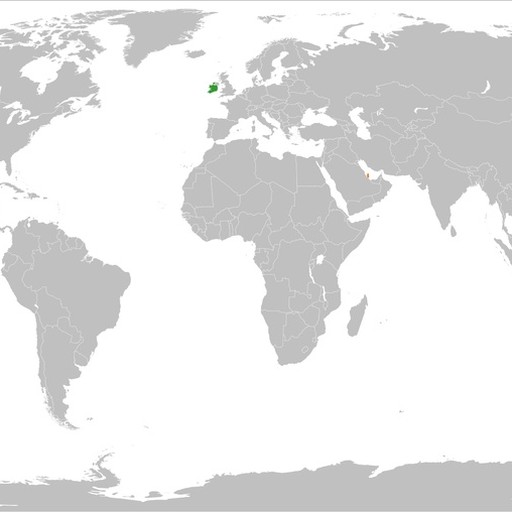} &
    \includegraphics[width=0.095\textwidth]{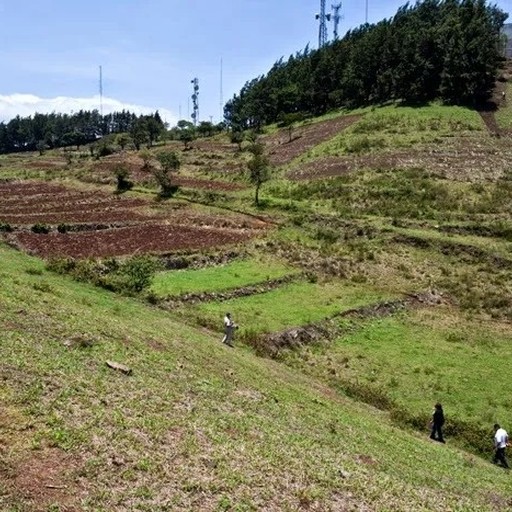} \\
    \includegraphics[width=0.095\textwidth]{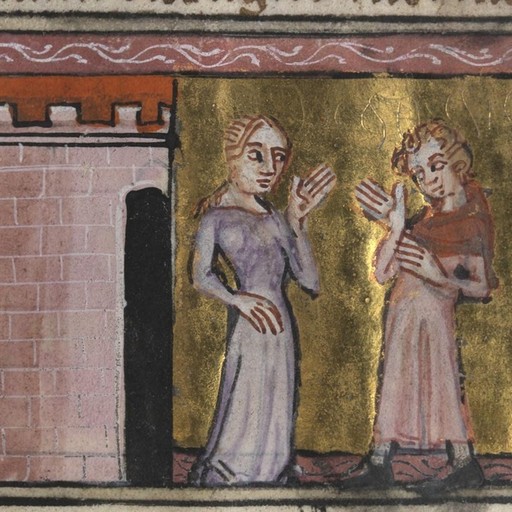} &
    \includegraphics[width=0.095\textwidth]{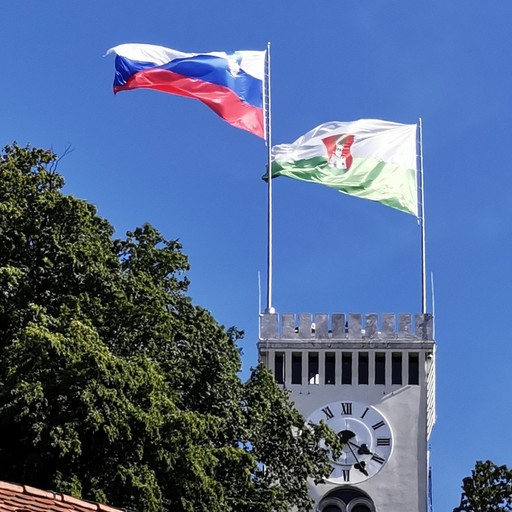} &
    \includegraphics[width=0.095\textwidth]{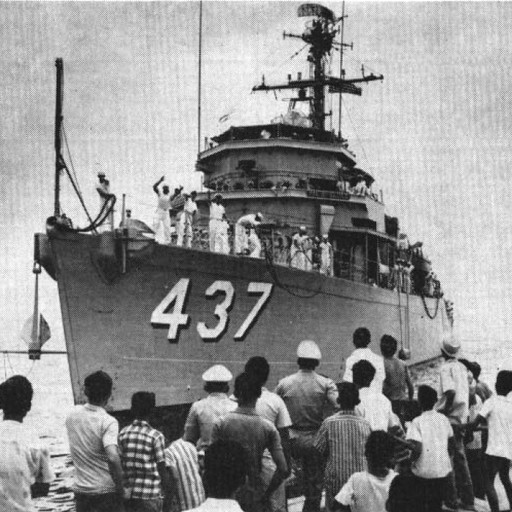} &
    \includegraphics[width=0.095\textwidth]{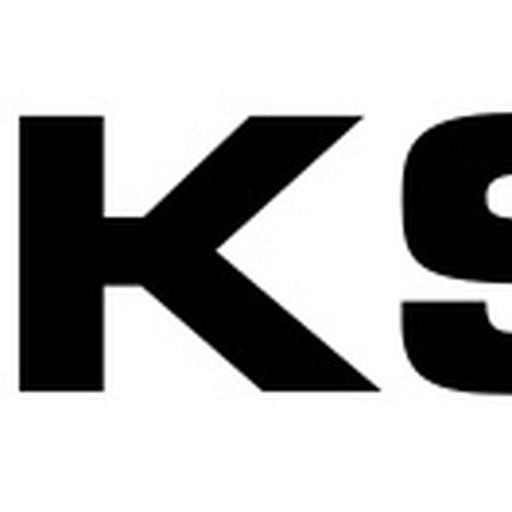} &
    \includegraphics[width=0.095\textwidth]{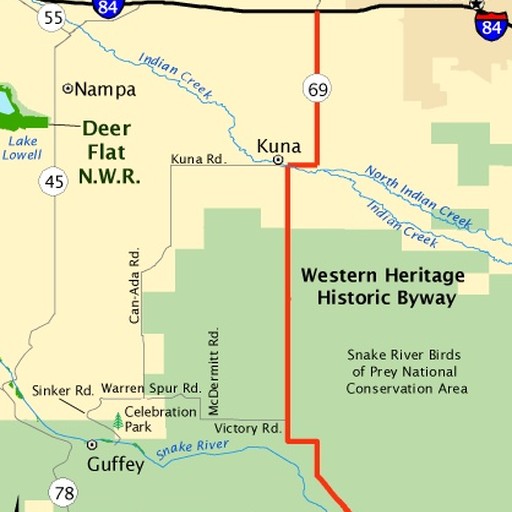} &
    \includegraphics[width=0.095\textwidth]{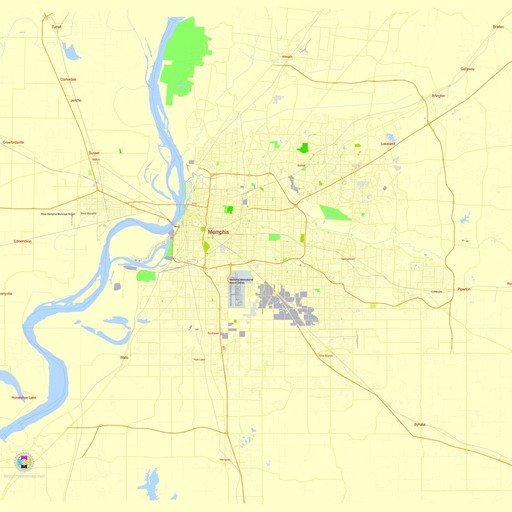} &
    \includegraphics[width=0.095\textwidth]{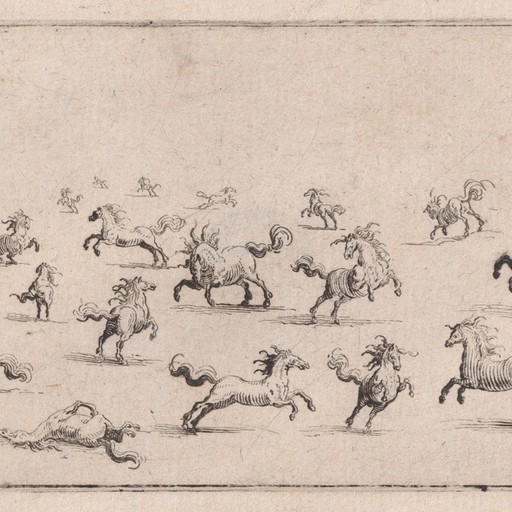} &
    \includegraphics[width=0.095\textwidth]{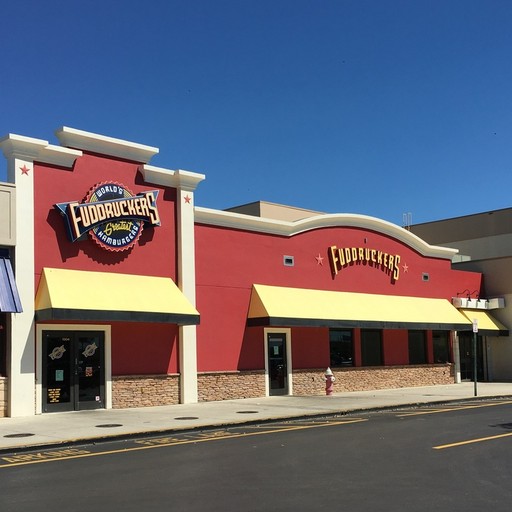} &
    \includegraphics[width=0.095\textwidth]{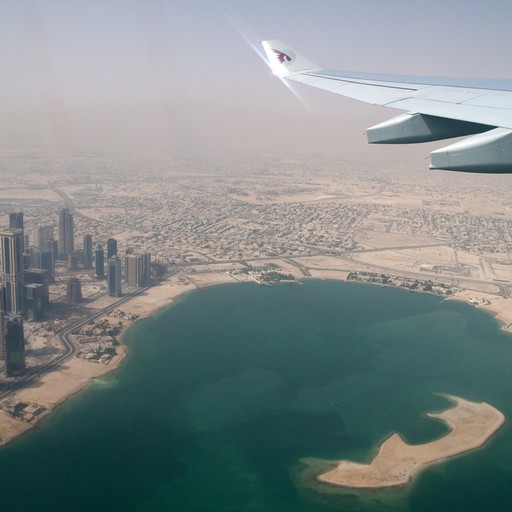} &
    \includegraphics[width=0.095\textwidth]{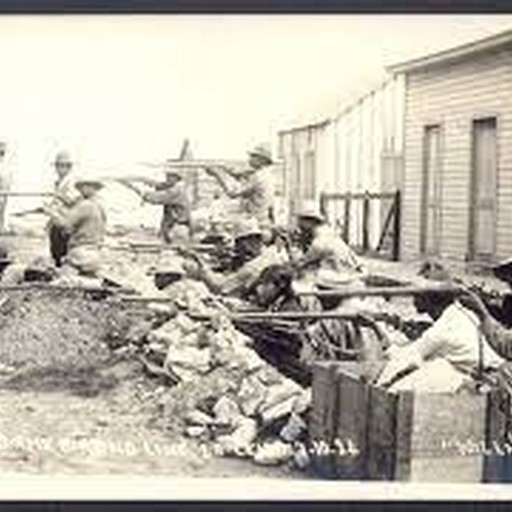} \\
    \includegraphics[width=0.095\textwidth]{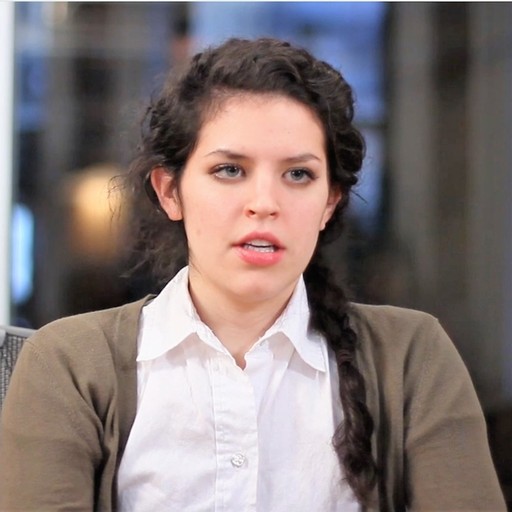} &
    \includegraphics[width=0.095\textwidth]{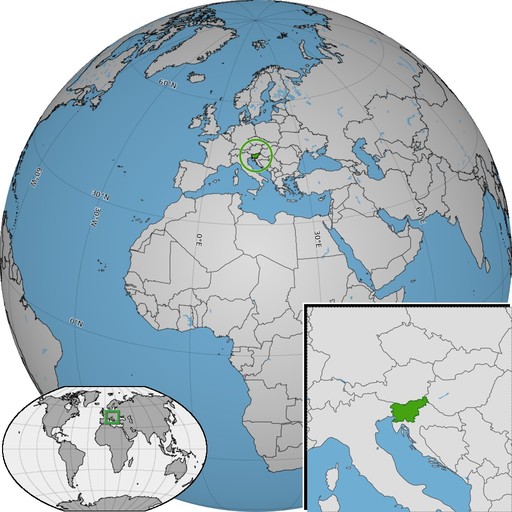} &
    \includegraphics[width=0.095\textwidth]{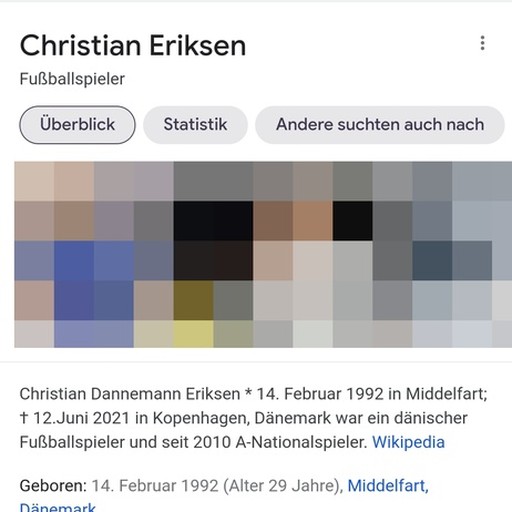} &
    \includegraphics[width=0.095\textwidth]{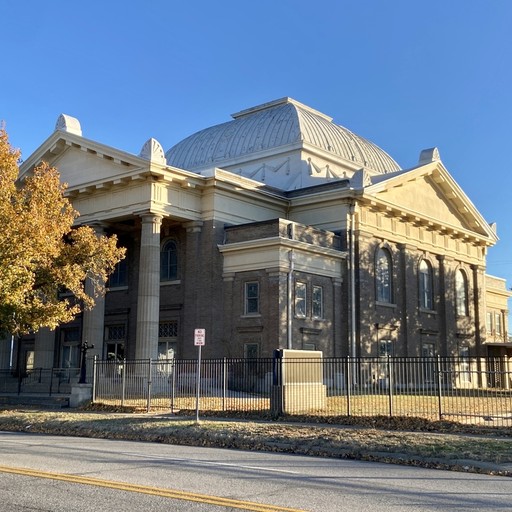} &
    \includegraphics[width=0.095\textwidth]{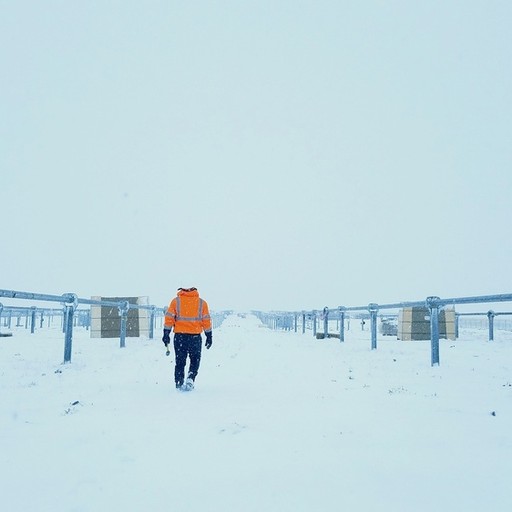} &
    \includegraphics[width=0.095\textwidth]{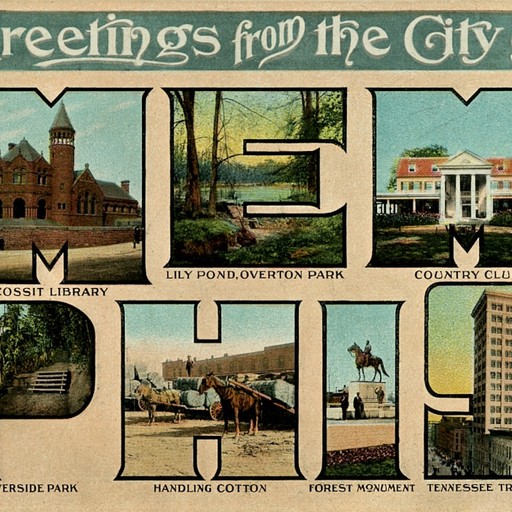} &
    \includegraphics[width=0.095\textwidth]{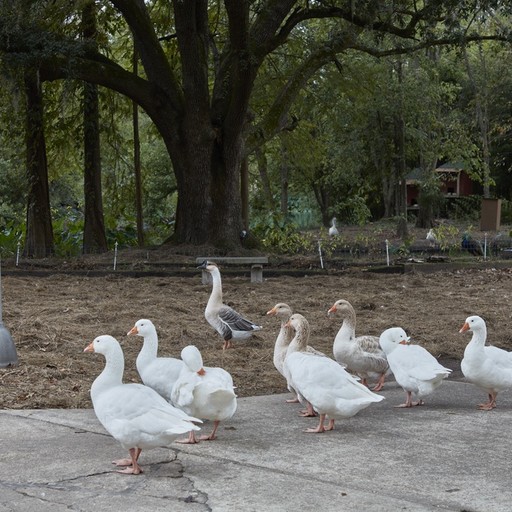} &
    \includegraphics[width=0.095\textwidth]{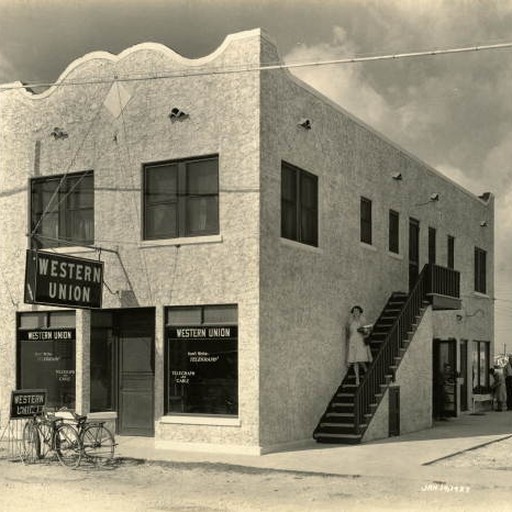} &
    \includegraphics[width=0.095\textwidth]{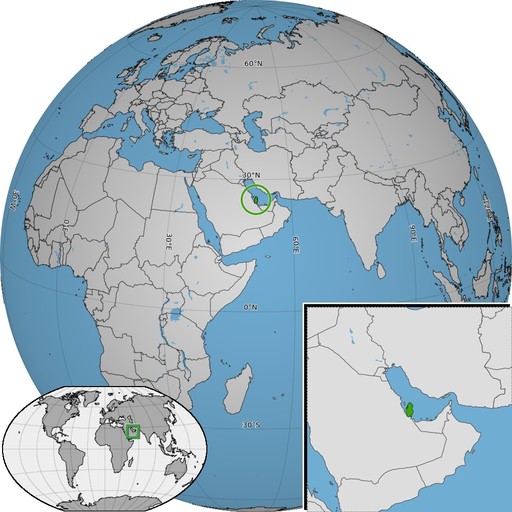} &
    \includegraphics[width=0.095\textwidth]{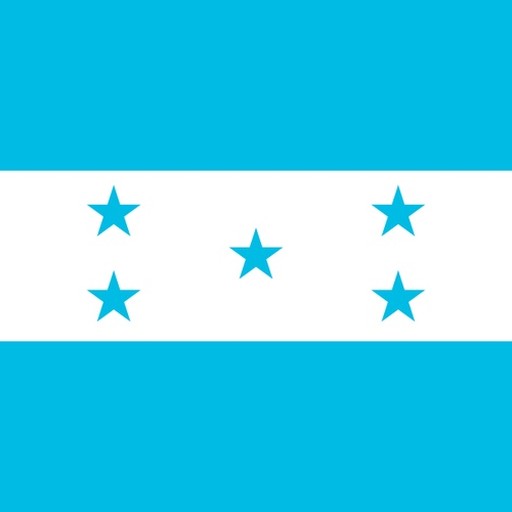} \\
    \scriptsize unreasonable & \scriptsize Slovenia & \scriptsize enhance & \scriptsize Wichita & \scriptsize kuna & \scriptsize Memphis & \scriptsize freely & \scriptsize Hialeah & \scriptsize Qatar & \scriptsize Honduras \\
  \end{tabular}
  \caption{Randomly sampled VisLingA concepts and some of their images. Each column shows 5 randomly selected images for a randomly selected concept, named below the column. Images are center-cropped for display only and in practice have varying ratios.}
  \label{fig:vislinga}
\end{figure}

\subsection{Attribution for CC BY Images in \Cref{fig:vislinga}}
\label{app:vislinga_mini_attribution}
All images in \Cref{fig:vislinga} are from Wikimedia Commons and are in the public domain or have CC0/various CC BY licenses. The 46 images whose licenses require attribution are credited below, identified by the concept named under their column and their row (counted from the top of each block). Each row number links to the image's Commons page and each license name to the license. The other 54 images are in the public domain or released under CC0.

{\scriptsize
\begin{itemize}\setlength{\itemsep}{0pt}\setlength{\parskip}{0pt}
  \item \textit{boy} (top): row \href{https://commons.wikimedia.org/wiki/File:Havelock_Island,_Boy,_Andamans.jpg}{5}: Vyacheslav Argenberg, \href{https://creativecommons.org/licenses/by/4.0}{CC BY 4.0}.
  \item \textit{smartest} (top): row \href{https://commons.wikimedia.org/wiki/File:Table_(first_part)_including_the_contents_and_origins_of_the_first_metric_scale_of_intelligence_Binet-Simon_1905.png}{3}: Serge Nicolas, Bernard Andrieu, Jean-Claude Croizet, Rasyid B. Sanitisio, Jeremy Trevelyan Burman, \href{https://creativecommons.org/licenses/by/3.0}{CC BY 3.0}; row \href{https://commons.wikimedia.org/wiki/File:Binet-Simon_Intelligence_Test_\%27What_is_missing_in_this_picture\%3F\%27.jpg}{5}: J. E. Wallace Wallin, \href{https://creativecommons.org/licenses/by/4.0}{CC BY 4.0}.
  \item \textit{computer} (top): rows \href{https://commons.wikimedia.org/wiki/File:Personal_computer,_exploded.svg}{1}, \href{https://commons.wikimedia.org/wiki/File:Personal_computer,_exploded_4.svg}{3}: Gustavb, \href{https://creativecommons.org/licenses/by/2.5}{CC BY 2.5}; row \href{https://commons.wikimedia.org/wiki/File:Apple_II_Plus,_Museum_of_the_Moving_Image.jpg}{2}: Marcin Wichary, \href{https://creativecommons.org/licenses/by/2.0}{CC BY 2.0}; row \href{https://commons.wikimedia.org/wiki/File:2014_Modding_komputera_01.jpg}{4}: Halicki, \href{https://creativecommons.org/licenses/by/3.0}{CC BY 3.0}.
  \item \textit{Banjul} (top): row \href{https://commons.wikimedia.org/wiki/File:Banjul,_The_Gambia_(CAPELLA_C09_SP_GEO_HH_20241121154435_20241121154459).png}{1}: Capella Space, \href{https://creativecommons.org/licenses/by/4.0}{CC BY 4.0}; row \href{https://commons.wikimedia.org/wiki/File:Klimadiagramm-deutsch-Banjul-Gambia.png}{3}: Hedwig in Washington, \href{https://creativecommons.org/licenses/by/2.5}{CC BY 2.5}; row \href{https://commons.wikimedia.org/wiki/File:Banjul_dock.jpg}{5}: Niels Broekzitter, \href{https://creativecommons.org/licenses/by/2.0}{CC BY 2.0}.
  \item \textit{Oman} (top): rows \href{https://commons.wikimedia.org/wiki/File:United_States_National_Guard_Bureau_visit_in_Muscat,_Sultanate_of_Oman,_March_12,_2023_-_10.jpg}{1}, \href{https://commons.wikimedia.org/wiki/File:United_States_National_Guard_Bureau_visit_in_Muscat,_Sultanate_of_Oman,_March_12,_2023_-_53.jpg}{2}, \href{https://commons.wikimedia.org/wiki/File:United_States_National_Guard_Bureau_visit_in_Muscat,_Sultanate_of_Oman,_March_12,_2023_-_7.jpg}{3}, \href{https://commons.wikimedia.org/wiki/File:United_States_National_Guard_Bureau_visit_in_Muscat,_Sultanate_of_Oman,_March_12,_2023_-_37.jpg}{4}, \href{https://commons.wikimedia.org/wiki/File:United_States_National_Guard_Bureau_visit_in_Muscat,_Sultanate_of_Oman,_March_12,_2023_-_19.jpg}{5}: Chief, National Guard Bureau, \href{https://creativecommons.org/licenses/by/2.0}{CC BY 2.0}.
  \item \textit{eyes} (top): row \href{https://commons.wikimedia.org/wiki/File:MC_Timberwolf.jpg}{1}: Christian Mehlführer, User:Chmehl, \href{https://creativecommons.org/licenses/by/2.5}{CC BY 2.5}; row \href{https://commons.wikimedia.org/wiki/File:Felis_silvestris_silvestris.jpg}{2}: Michael Gäbler, \href{https://creativecommons.org/licenses/by/3.0}{CC BY 3.0}; row \href{https://commons.wikimedia.org/wiki/File:\%D0\%93\%D0\%BB\%D0\%B0\%D0\%B7_\%D1\%84\%D0\%BB\%D0\%B5\%D1\%80\%D0\%BD\%D0\%B8\%D1\%86\%D1\%8B.jpg}{3}: Gribkov mikhail, \href{https://creativecommons.org/licenses/by/4.0}{CC BY 4.0}; row \href{https://commons.wikimedia.org/wiki/File:Eye_orbit_anatomy_anterior2.jpg}{5}: Patrick J. Lynch, medical illustrator, \href{https://creativecommons.org/licenses/by/2.5}{CC BY 2.5}.
  \item \textit{child} (top): row \href{https://commons.wikimedia.org/wiki/File:Kids,_Namibe,_Angola_cropped.jpg}{5}: Alfred Weidinger from Vienna, Austria; derivative work: Tomer T, \href{https://creativecommons.org/licenses/by/2.0}{CC BY 2.0}.
  \item \textit{swam} (top): row \href{https://commons.wikimedia.org/wiki/File:CrocoLoco-swimming-01.jpg}{1}: MathKnight and Zachi Evenor, \href{https://creativecommons.org/licenses/by/4.0}{CC BY 4.0}; row \href{https://commons.wikimedia.org/wiki/File:Girl_with_swimming_board.jpg}{3}: Tommy Wong, \href{https://creativecommons.org/licenses/by/2.0}{CC BY 2.0}; row \href{https://commons.wikimedia.org/wiki/File:PB141203_(15176335693).jpg}{4}: Dick Powell Photos Indy, \href{https://creativecommons.org/licenses/by/2.0}{CC BY 2.0}; row \href{https://commons.wikimedia.org/wiki/File:Lycodon_rufozonatus_in_Yesan_South_Korea_swimming_pic3.jpg}{5}: pintail, \href{https://creativecommons.org/licenses/by/4.0}{CC BY 4.0}.
  \item \textit{unreasonable} (bottom): row \href{https://commons.wikimedia.org/wiki/File:Pure_Reason_Revolution_live_in_London_2022_-_52017598276.jpg}{2}: kitmasterbloke, \href{https://creativecommons.org/licenses/by/2.0}{CC BY 2.0}; row \href{https://commons.wikimedia.org/wiki/File:Emily_Axford_in_A_Totally_Unreasonable_Will.jpg}{5}: Isaac Leicht, \href{https://creativecommons.org/licenses/by/3.0}{CC BY 3.0}.
  \item \textit{Slovenia} (bottom): row \href{https://commons.wikimedia.org/wiki/File:Slovenia_-_Grajevnik_-_gozd.jpg}{1}: Dibko, \href{https://creativecommons.org/licenses/by/4.0}{CC BY 4.0}; row \href{https://commons.wikimedia.org/wiki/File:Idrija_miners\%27_theatre,_Slovenia,_20240503_1313_8192.jpg}{3}: Jakub Hałun, \href{https://creativecommons.org/licenses/by/4.0}{CC BY 4.0}; row \href{https://commons.wikimedia.org/wiki/File:SLO_\%E2\%80\%94_Ljubljana_(Blick_auf_Burg)_2020.JPG}{4}: User:Mateus2019, \href{https://creativecommons.org/licenses/by/3.0/de/deed.en}{CC BY 3.0 de}.
  \item \textit{Wichita} (bottom): rows \href{https://commons.wikimedia.org/wiki/File:First_Presbyterian_Church,_Wichita,_Kansas.jpg}{1}, \href{https://commons.wikimedia.org/wiki/File:Petroleum_Building,_Wichita,_Kansas_01.jpg}{2}, \href{https://commons.wikimedia.org/wiki/File:Petroleum_Building,_Wichita,_Kansas_02.jpg}{3}, \href{https://commons.wikimedia.org/wiki/File:Former_First_Church_of_Christ,_Scientist,_Wichita,_Kansas.jpg}{5}: Myotus, \href{https://creativecommons.org/licenses/by/4.0}{CC BY 4.0}.
  \item \textit{kuna} (bottom): row \href{https://commons.wikimedia.org/wiki/File:Kuna_(Cuna)_aeroplane_model_circa_1920.jpg}{1}: Brigade Piron, \href{https://creativecommons.org/licenses/by/3.0}{CC BY 3.0}; row \href{https://commons.wikimedia.org/wiki/File:Mammoth_Peak_and_Kuna_Crest.jpg}{3}: David Ackerman, \href{https://creativecommons.org/licenses/by/2.0}{CC BY 2.0}.
  \item \textit{Memphis} (bottom): rows \href{https://commons.wikimedia.org/wiki/File:Greetings_from_Memphis,_Tennessee_-_Large_Letter_Postcard_(4601111792).jpg}{1}, \href{https://commons.wikimedia.org/wiki/File:Greetings_from_the_City_of_Memphis,_Tennessee_-_Large_Letter_Postcard_(53156469449).jpg}{5}: Steve Shook from Moscow, Idaho, USA, \href{https://creativecommons.org/licenses/by/2.0}{CC BY 2.0}; row \href{https://commons.wikimedia.org/wiki/File:Clark_Tower_Memphis_TN_01.jpg}{3}: Thomas R Machnitzki, \href{https://creativecommons.org/licenses/by/3.0}{CC BY 3.0}.
  \item \textit{Hialeah} (bottom): rows \href{https://commons.wikimedia.org/wiki/File:Bed_Bath_\%26_Beyond_Closing_Last_10_Days,_Hialeah,_Florida,_March_2023_-_6.jpg}{1}, \href{https://commons.wikimedia.org/wiki/File:Bed_Bath_\%26_Beyond_Closing_Last_10_Days,_Hialeah,_Florida,_March_2023_-_7.jpg}{2}, \href{https://commons.wikimedia.org/wiki/File:Bed_Bath_\%26_Beyond_Closing_Last_10_Days,_Hialeah,_Florida,_March_2023_-_8.jpg}{3}: Phillip Pessar, \href{https://creativecommons.org/licenses/by/2.0}{CC BY 2.0}; row \href{https://commons.wikimedia.org/wiki/File:Fuddruckers_Westland_Mall_Hialeah_Florida_March_2018.jpg}{4}: Phillip Pessar from Miami, USA, \href{https://creativecommons.org/licenses/by/2.0}{CC BY 2.0}.
  \item \textit{Qatar} (bottom): rows \href{https://commons.wikimedia.org/wiki/File:Qatar,_Doha,_Skyscrapers,_Qatar.jpg}{1}, \href{https://commons.wikimedia.org/wiki/File:Qatar,_Doha,_Qatar,_Persian_Gulf.jpg}{4}: Vyacheslav Argenberg, \href{https://creativecommons.org/licenses/by/4.0}{CC BY 4.0}.
  \item \textit{Honduras} (bottom): row \href{https://commons.wikimedia.org/wiki/File:Honduran_food_breakfast.jpg}{2}: Luis Alfredo Romero, \href{https://creativecommons.org/licenses/by/4.0}{CC BY 4.0}; row \href{https://commons.wikimedia.org/wiki/File:Las_terrazas_Ruins_Honduras.jpg}{3}: Diario El Tiempo, \href{https://creativecommons.org/licenses/by/4.0}{CC BY 4.0}.
\end{itemize}
}

\section{Small-Scale Experiments Details}
\label{app:smallscale_details}
For captioning, a linear layer is trained between the vision and language models as an adapter. For retrieval each has a linear layer trained on top of its last layer representations, by default using the class token for the vision head and averaging the post-adapter token representations for the language model. Different, more powerful adapters and representation aggregation schemes were not found to make a big difference. The training data is COCO-2014 \citep{lin2014microsoft}.

To ensure the models are well trained optimal learning rates and weight decays were chosen by doing a grid search over learning rates and weight decays, spaced log-linearly in $10^{-5},10^{-4.5},...10^{-2}$ and $10^{-4},10^{-3},...10^{-1}$ respectively. Optimal retrieval hyperparameters were chosen based on the model's top-1 ImageNet classification accuracy over ImageNet's validation set. Generation hyperparameters were optimized on the BLEU score over COCO's validation set. Models were trained for 10 epochs with an early stopping patience of 5, with validation metrics computed once every 0.25 epochs. 

Resulting optimal hyperparameters for all generative model configurations, including those referenced in the next section, are listed in \Cref{tab:gen_hyps}. For retrieval a learning rate of $10^{-4}$ and weight decay of $10^{-2}$ was found to work well for all settings.

\begin{table}[h]
    \centering
    \begin{tabular}{lcc}
         Adapter type & Learning rate & Weight decay  \\
         \midrule
         Linear & $10^{-2.5}$ & $10^{-2}$ \\
         MLP & $10^{-3}$ & $10^{-2}$ \\
         Attention & $10^{-3.5}$ & $10^{-3}$
    \end{tabular}
    \caption{Optimal learning rates and weight decays for different adapters for the small-scale generative experiments.}
    \label{tab:gen_hyps}
    \vspace{-15pt}
\end{table}

\subsection{Different Adapters Yield Similar Results}
\label{app:diff_adapters}
To see if better adapters help we repeated the two small-scale experiments with two more powerful adapters -- an MLP and an attention block between the tokens and learnable embeddings. The attention adapter has a single learnable token embedding for retrieval and 200 learnable embeddings for generation, where a ViT-B/16 has 197 vision tokens by default. For retrieval we also experiment with three different kinds of pooling -- mean pooling, max pooling, and, specifically for the vision transformer, taking the class token. When the vision transformer uses the class token the language model defaults to mean pooling. The attention adapter needs no pooling due to the cross attention already aggregating information.

All adapters and choices of pooling yielded similar held-out concept detection accuracies. At the end of training, the models trained without the held-out concepts have detection rates of 1.2\%, 2.0\% and 2.2\% of the time with the linear, MLP and attention adapters respectively. The retrieval models reach concept detection accuracies between 1.6\% and 8.9\%, with a linear adapter and max pooling doing best. When trained on the full dataset, all generative variants reach 84.0--85.9\% and all retrieval variants 79.0--81.5\%. Even the best held-out concept detection accuracy at any point during training, a much more lenient metric, remains low: 7.2\% for generation and 10.3\% for retrieval, both with a linear adapter. Held-out concept accuracies are averaged over concepts and measured every 0.1 epochs.

\section{What Does it Mean for a Model to ``Know'' of a Concept?}
\label{app:know_conc}
For the vision models this is simple -- all the concepts are ImageNet classes and the models are ImageNet classifiers that get high accuracies. For the language models, they can be prompted to complete texts such as ``a \texttt{<concept>} is'' or asked questions about it, thereby gauging some form of recognition.

To see whether the LLaVA backbone knows of the held-out concepts we ask it various questions about them, such as ``What is a zebra?'' and ``What is a horse-like animal with black and white stripes?''. Doing so with the LLaVA model's LLM backbone after it was trained showed that the language model by itself still recognizes the held-out concepts.

\section{Concept Filtering}
\label{app:concept_filt}
The held-out concepts are ``zebra'', ``banana'', ``pizza'', ``umbrella'', and ``toothbrush'' for the small-scale experiments, with additionally ``broccoli'', ``microwave'', ``strawberry'', ``castle'', ``envelope'', ``violin'', ``tractor'', ``flamingo'', ``printer'', ``volcano'' for the large-scale experiments. The small-scale experiments use a smaller concept set as even the small scale models trained on the entire dataset fail to detect most of the additional concepts. To ensure the datasets are not overly reduced, for the large-scale experiments we train three models without the held-out concepts, each having only five concepts held-out.

To filter the held-out concepts from the finetuning data we remove all image-text pairs where a) the text contains the object's name or b) the image is likely to contain it. Image-based filtering is important to prevent misclassified data from hampering the experiment, e.g. zebra images accidentally captioned as giraffes. This is done using a CLIP ViT-L/14 model and filtering images with a similarity above a chosen threshold to the text ``a photo of a \texttt{<concept>}'', with the threshold found through manual tuning. For the small-scale experiments the threshold was 0.2 while for the large-scale ones it was 0.25. This reduces the datasets by at most 9\%.

Specifically for the large-scale experiment we also filter the concepts ``tooth brush'' and ``volcanic'' as we find they are occasionally not detected by the image-based filter. There may be other, less direct forms of leakage we are unaware of that artificially boost the large-scale model's performance on detecting held-out concepts.

\section{More Qualitative Examples of LLaVA Generalization Failures}
\label{app:qual_llava_fail}
See \Cref{fig:more_fail_ex}. Note that all descriptions are generated at $T=0$, except for the successful banana example (o), as none of the $T=0$ descriptions of the banana images mention a banana. Example (o) was obtained by sampling with $T=0.2$ repeatedly.

\begin{figure}
  \centering

  \begin{subfigure}[t]{0.32\textwidth}
    \centering
    \qualimg{pizza_neg_n07873807_440.JPEG}
    \caption{The image features a plate of food, which includes \boldred{a sandwich and a salad}. The sandwich is placed on a pink plate...}
  \end{subfigure}\hfill
  \begin{subfigure}[t]{0.32\textwidth}
    \centering
    \qualimg{pizza_neg_n07873807_1181.JPEG}
    \caption{The image features a table with \boldred{a cake} on it, along with a birthday candle and a birthday card...}
  \end{subfigure}\hfill
  \begin{subfigure}[t]{0.32\textwidth}
    \centering
    \qualimg{pizza_pos_n07873807_207.JPEG}
    \caption{The image features a large, delicious-looking \boldgreen{pizza} with various toppings, sitting on a plate.}
  \end{subfigure}

  \begin{subfigure}[t]{0.32\textwidth}
    \centering
    \qualimg{umbrella_neg_rank090_24-860_train_n04507155_n04507155_6654.JPEG}
    \caption{The image features a young boy standing in the rain, holding \boldred{a kite}.}
  \end{subfigure}\hfill
  \begin{subfigure}[t]{0.32\textwidth}
    \centering
    \qualimg{umbrella_neg_rank025_25-910_train_n04507155_n04507155_8167.JPEG}
    \caption{...a man standing outside, holding a cell phone in his hand. He is also holding \boldred{a blue bag or purse}.}
  \end{subfigure}\hfill
  \begin{subfigure}[t]{0.32\textwidth}
    \centering
    \qualimg{umbrella_pos_rank100_24-666_train_n04507155_n04507155_14898.JPEG}
    \caption{The image features \boldgreen{a blue umbrella} sitting on the grass in a park.}
  \end{subfigure}

  \begin{subfigure}[t]{0.32\textwidth}
    \centering
    \qualimg{toothbrush_neg_rank294_21-758_train_n03742115_n03742115_6826.JPEG}
    \caption{The image features a bathroom with a sink, mirror, and \boldred{a large collection of toiletries and personal care items}.}
  \end{subfigure}\hfill
  \begin{subfigure}[t]{0.32\textwidth}
    \centering
    \qualimg{toothbrush_neg_rank257_21-887_train_n04553703_n04553703_38705.JPEG}
    \caption{The image features a woman \boldred{washing her hands} in a bathroom sink.}
  \end{subfigure}\hfill
  \begin{subfigure}[t]{0.32\textwidth}
    \centering
    \qualimg{toothbrush_pos_rank307_21-722_train_n03742115_n03742115_10980.JPEG}
    \caption{...bathroom with a mirrored cabinet, a sink, and a shelf. The shelf is filled with various items, including cups, \boldgreen{a toothbrush}, and a bottle.}
  \end{subfigure}
  
  \begin{subfigure}[t]{0.32\textwidth}
    \centering
    \qualimg{zebra_neg_rank049_26-575_train_n02391049_n02391049_2141.JPEG}
    \caption{The image features \boldred{a giraffe} standing in a grassy field.}
  \end{subfigure}\hfill
  \begin{subfigure}[t]{0.32\textwidth}
    \centering
    \qualimg{zebra_neg_rank073_26-356_train_n02391049_n02391049_5633.JPEG}
    \caption{The image features \boldred{a black and white striped giraffe}.}
  \end{subfigure}\hfill
  \begin{subfigure}[t]{0.32\textwidth}
    \centering
    \qualimg{zebra_pos_rank095_26-181_train_n02391049_n02391049_147.JPEG}
    \caption{The image features a close-up of \boldgreen{a zebra's face}, with its eyes and mouth visible.}
  \end{subfigure}

  \begin{subfigure}[t]{0.32\textwidth}
    \centering
    \qualimg{banana_neg_rank049_25-610_train_n07753592_n07753592_11210.JPEG}
    \caption{...a man riding a bicycle with a cart attached to it, filled with \boldred{a large number of apples}...}
  \end{subfigure}\hfill
  \begin{subfigure}[t]{0.32\textwidth}
    \centering
    \qualimg{banana_neg_rank097_25-192_train_n07753592_n07753592_12069.JPEG}
    \caption{The image features a close-up of \boldred{two peeled oranges} sitting on a wooden table.}
  \end{subfigure}\hfill
  \begin{subfigure}[t]{0.32\textwidth}
    \centering
    \qualimg{banana_pos_rank064_25-480_train_n07753592_n07753592_3872.JPEG}
    \caption{The image features a large pile of yellow \boldgreen{bananas}, which are stacked on top of each other.}
  \end{subfigure}
\vspace{-5pt}
  \caption{Additional examples of generalization failures from the LLaVA-1.5-7B model from \Cref{fig:llava_examples}, trained without the held-out concepts, when prompted ``What is in this image?''. Rows correspond to the concepts pizza, umbrella, toothbrush, zebra, and banana respectively. For each concept we randomly sample two images whose description does not mention the concept (left, center) and one image where the concept was correctly detected (right). Note that this 2:1 ratio is purely illustrative and is not representative of the real detection rate per concept.}
  \label{fig:more_fail_ex}
\end{figure}

\section{Discussion on Large-Scale Results}
\label{app:large_scale_disc}
The large-scale model getting better held-out concept detection accuracies than the smaller models could mean that the failure described in \Cref{sec:vlms_failure} is partially alleviated by scale. However, there are other possible reasons a larger model could do better, such as LLaVA generally outputting longer completions or it potentially relying on its internal knowledge more so than what it sees, akin to the failure described by \citet{vo2025vision}. For example, the completion of ``a bathroom with a toilet, mirror, toothpaste, ...'' would likely mention a toothbrush regardless of whether it actually exists in the image.

To test this, we benchmark the large-scale models when prompting them to only name the main object they see in the image, limiting them to output answers with at most 8 tokens. In this setting the models trained without the held-out concepts performed comparably well, getting an average detection accuracy of 19.6\%. Thus, the large-scale model's success is not due to relying on the language model's knowledge while ignoring the vision model.

Another option is that there is some data leakage. While running these experiments, we initially filtered for ``toothbrush'' but not for ``tooth brush'', thereby having some examples slip by. As larger models are more sample efficient, these $\sim$20 extra image-text pairs resulted in toothbrushes being detected much more often. In the current setup, the large-scale models detect ``pizza'' much more than any other concept, although it is hard to say whether there is some data leakage we have not managed to find, the better detection performance is due to scale, or some mixture thereof. When training the model with the dataset visual similarity filtering threshold set to 0.2 instead of 0.25 models generally perform much worse, getting accuracies around half of those reported in \Cref{tab:concept_detection}, which lends credence to there possibly being some leakage.

\section{Details on \Cref{sec:repr_align_failures} LLaVA Models}
\label{app:align_failures}
For the vision backbones we use 11 different standard pretrained ViT-B models and 26 finetunes thereof. ViT-B are used and not ViT-L, which is more typical for LLaVA, due to not all models having pretrained ViT-L variants. For the original set of 11 models, we use BEiT \citep{bao2021beit}, CLIP models trained by OpenAI and on the Datacomp dataset \citep{radford2021learning, datacomp}, DEiT \citep{touvron2021training}, DINO \citep{caron2021emerging}, DINOv2 \citep{oquab2023dinov2}, MAE \citep{he2022masked}, SIGLip and SIGLip2 \citep{zhai2023sigmoid,tschannen2025siglip}, and ViTs trained on ImageNet 1k and 21k \citep{Dosovitskiy2020AnII}. For the 26 finetunes we use 6 models from FusionBench \citep{tang2025fusionbench} and various other models uploaded to HuggingFace, all with at least 400 downloads over the last month as of writing \citep{nurul2023watermarkdetector,iakubovskyi2023aivsreal,iakubovskyi2023closedeyes,iakubovskyi2023deepfakevsreal,iakubovskyi2023facialemotions,iakubovskyi2023manwomanface,enright2023rorshark,ng2023blurvsclean,vishwakarma2023humanaction,sharma2024ulbase,sharma2024ulbedroom,sharma2024ulinterior,iakubovskyi2024fairfaceage,iakubovskyi2024fairfacegender,codewithdark2025vitchestxray,sakthi2025genderclassifiermini,sakthi2025deepfakedetectorv2,sakthi2025agesiglip2,sakthi2025watermarksiglip2,sahoo2025faceemotion}.

\section{Limitations}
\label{app:limits}
Representational similarity is generally a vague notion without a universally agreed upon definition, making it difficult to making decisive claims on representational similarity or alignability \citep{sucholutsky2023getting}. While we attempt to make broad claims by showing connections between different measures of similarity (see section \ref{sec:align_equiv}) and grounding results in theory where possible (section \ref{sec:intermod_hard}), these are naturally not exhaustive over the many notions of similarity. Additionally, results are believed to be general due to their underlying cause, namely inherent differences in the information content in different modalities. The difference in information content has been discussed in several prior works \citep{gal2021robustness,schrodi2025two}, with the results herein providing further evidence.

Most of the theoretical results are on linear alignability while in practice many VLMs are built using nonlinear adapters. Some results show that linear alignability bounds some notions of nonlinear similarity between representations (Appendix \ref{app:repr_sim_lin_align}), while others show that failures appear also given nonlinear alignment, namely the LLaVA experiments in sections \ref{sec:vlms_failure} and \ref{sec:repr_align_failures}, and the ablations in Appendix \ref{app:diff_adapters}. 

Most results are circumstantial in supporting that many VLM failures likely arise due to representational misalignment. In this sense the results are consistent with the hypothesis but do not causally show that representational misalignment is what caused a failure. Results in section \ref{sec:vlms_failure} are an important exception, as based on the theory they make a causal prediction which turns out to be correct -- that if vision-language representations are misaligned then the generalization failure would occur.

\end{document}